\documentclass{article} % For LaTeX2e
\usepackage{iclr2025_conference,times}

\usepackage{amsmath,amsfonts,bm}

\def\eqref#1{equation~\ref{#1}}
\def\1{\bm{1}}

\DeclareMathAlphabet{\mathsfit}{\encodingdefault}{\sfdefault}{m}{sl}
\SetMathAlphabet{\mathsfit}{bold}{\encodingdefault}{\sfdefault}{bx}{n}

\usepackage{amsthm, amssymb}
\usepackage{amsmath}
\usepackage{algorithm, algorithmic}
\newtheorem{assumption}{Assumption}

\newtheorem{lemma}{Lemma}
\newtheorem{theorem}{Theorem}
\newtheorem{definition}{Definition}

\newtheorem{corollary}{Corollary}
\usepackage[utf8]{inputenc} % allow utf-8 input
\usepackage[T1]{fontenc}    % use 8-bit T1 fonts
\usepackage{hyperref}       % hyperlinks
\usepackage{cleveref}
\usepackage{url}            % simple URL typesetting
\usepackage{booktabs}       % professional-quality tables
\usepackage{amsfonts}       % blackboard math symbols
\usepackage{nicefrac}       % compact symbols for 1/2, etc.
\usepackage{microtype}      % microtypography
\usepackage{xcolor}         % colors
\usepackage{bm}

\usepackage{makecell}
\usepackage{multirow}
\usepackage{adjustbox}
\usepackage{wrapfig}
\usepackage{caption} % for figure and table captions

\title{MGDA Converges under Generalized Smoothness, Provably}

\author{Qi Zhang\thanks{Equal contribution.}\space\thinspace\textsuperscript{1}, Peiyao Xiao\footnotemark[1]\thinspace\space\textsuperscript{2}, Shaofeng Zou\textsuperscript{1} \& Kaiyi Ji\textsuperscript{2}\\
School of Electrical, Computer and Energy Engineering, Arizona State University\textsuperscript{1}\\
Department of Computer Science and Engineering, University at Buffalo\textsuperscript{2}\\
\texttt{\{qzhan261,zou\}@asu.edu, \{peiyaoxi,kaiyiji\}@buffalo.edu} \\
}

\iclrfinalcopy % Uncomment for camera-ready version, but NOT for submission.
\begin{document}

\maketitle

\begin{abstract}
Multi-objective optimization (MOO) is receiving more attention in various fields such as multi-task learning. Recent works provide some effective algorithms with theoretical analysis but they are limited by the standard $L$-smooth or bounded-gradient assumptions, which typically do not hold for neural networks, such as Long short-term memory (LSTM) models and Transformers. In this paper, we study a more general and realistic class of generalized $\ell$-smooth loss functions, where $\ell$ is a general non-decreasing function of gradient norm. We revisit and analyze the fundamental multiple
gradient descent algorithm (MGDA) and its stochastic version with double sampling for solving the generalized $\ell$-smooth MOO problems,  which approximate the conflict-avoidant (CA) direction that maximizes the minimum improvement among objectives. We provide a comprehensive convergence analysis of these algorithms and show that they converge to an $\epsilon$-accurate Pareto stationary point with a guaranteed $\epsilon$-level average CA distance (i.e., the gap between the updating direction and the CA direction) over all iterations,  
where totally $\mathcal{O}(\epsilon^{-2})$ and $\mathcal{O}(\epsilon^{-4})$ samples are needed for deterministic and stochastic settings, respectively. 
We prove that they can also guarantee a tighter $\epsilon$-level CA distance in each iteration using more samples. Moreover, we analyze an efficient variant of MGDA named MGDA-FA using only $\mathcal{O}(1)$ time and space, while achieving the same performance guarantee as MGDA. 
\end{abstract}

\section{Introduction}
There have been a variety of emerging applications of multi-objective optimization (MOO), such as online advertising \citep{ma2018modeling}, autonomous driving \citep{huang2019apolloscape}, and reinforcement learning \citep{thomas2021multi}. Mathematically, the MOO problem takes the following formulation.
\begin{align}\label{eq:MOO}
F^*=\min_{x\in\mathbb{R}^m}{F}(x) := (f_1(x), f_2(x),...,f_K(x)),
\end{align}
where $K$ is the total number of objectives and $f_k(x)$ is the $k$-objective function given model parameters $x$. Under the stochastic setting, $f_k(x)=\mathbb{E}_s[f_k(x;s)]$, where $s$ denotes data sample. {In MOO setting, we are interested in optimizing all the objectives simultaneously. However,} this problem is challenging due to the gradient conflict that some objectives with larger gradients dominate the update direction at the sacrifice of
significant performance degeneration on the less-fortune objectives with smaller gradients. {Thus, a  widely-adopted target is to find the Pareto stationary point $x$ that the performance of all objectives cannot be further improved without compromising some objectives. } A variety of MOO-based methods have been proposed to mitigate this conflict and find a more balanced solution among all objectives. In particular, the multiple gradient descent algorithm (MGDA) \citep{desideri2012multiple} aims to find a conflict-avoidant (CA) update direction that maximizes the minimal improvement among all objectives and converges to a Pareto stationary point at which there is no common descent direction for all objective functions. 
This idea then inspired numerous follow-up methods including but not limited to CAGrad \citep{liu2021conflict}, PCGrad \citep{yu2020gradient}, GradDrop~\citep{chen2020just}, FAMO \citep{liu2024famo} and  FairGrad \citep{ban2024fair} with a convergence guarantee in the deterministic setting with full-gradient computations.  
% Besides, another issue of MOO lies in the fact that it is hard to find a point that minimizes all objective functions simultaneously. Towards this end, the target becomes finding a Pareto stationary point such that there is no common descent direction for all objective functions. 
% In this context, many methods have been proposed including multiple gradient descent algorithm (MGDA) \citep{desideri2012multiple}, CAGrad \citep{liu2021conflict}, PCGrad \citep{yu2020gradient}, MoCo \citep{fernando2022mitigating}, SDMGrad \citep{xiao2024direction} with extensive empirical results and some theoretical analysis. 
The theoretical understanding of the convergence and complexity of stochastic MOO is not well-developed until very recently.  \cite{liu2021stochastic} proposed stochastic multi-gradient (SMG) as a stochastic version of MGDA, and established its convergence guarantee. %but with an increasing batch size.
\cite{zhou2022convergence} analyzed the non-convergence issues of MGDA,
CAGrad and PCGrad in the stochastic setting, and further proposed a convergent approach named CR-MOGM. More recently, \cite{fernando2022mitigating} and \cite{chen2024three} proposed single-loop stochastic MOO methods named MoCo and MoDo, and proved their convergence to an $\epsilon$-accurate Pareto stationary point while guaranteeing an $\epsilon$-level {\bf average CA distance}\footnote{CA distance means the distance between the updating direction and the CA direction. Its formal definition can be found in \Cref{sec:CAs}} over all iterations. \cite{xiao2024direction} proposed a double-loop algorithm named SDMGrad that enables to obtain an unbiased stochastic multi-gradient via a double-sampling strategy. They established the convergence of SDMGrad with a guaranteed $\epsilon$-level CA distance in every iteration, which we call as {\bf iteration-wise CA distance}.

% Recent works of MOO tend to approximate the conflict-avoidant (CA) direction to mitigate the gradient conflict \citep{fernando2022mitigating,xiao2024direction,chen2024three}. In this case, measuring the distance from the update direction to the CA direction, the so-called CA distance will be important. The larger the CA distance is, the further the update direction will be away from the CA direction and the more conflict it will have. Theoretically, there are two kinds of CA distance analysis in literature. The single-loop algorithms MoCo and MoDo guarantee a $\epsilon$-level of average CA distance over iterations. The double-loop algorithm SDMGrad can guarantee small-level CA distances in every iteration, which we call iteration-wise CA distance.
% \begin{figure}[t]
% \centering
% \begin{minipage}[t]{0.45\textwidth}
%     \centering
%     \includegraphics[width=\textwidth]{scatter_plot_GS_task1.png}
%     % \caption{}
%     \label{fig:first-figure}
% \end{minipage}
% % \vspace{-0.5cm}
% % \hfill
% \begin{minipage}[t]{0.45\textwidth}
%     \centering
%     \includegraphics[width=\textwidth]{scatter_plot_GS_task2.png}
%     % \caption{Caption for the second figure}
% \end{minipage}
% \caption{Local smoothness constant vs. gradient norm on training SegNet on CityScapes dataset of each task. Task 1 on the left and Task 2 on the right.} \label{fig:task2}
% \end{figure}
\begin{wrapfigure}{r}{0.4\textwidth}
    \vspace{-0.7cm}
     \centering
    \includegraphics[width=0.5\textwidth]{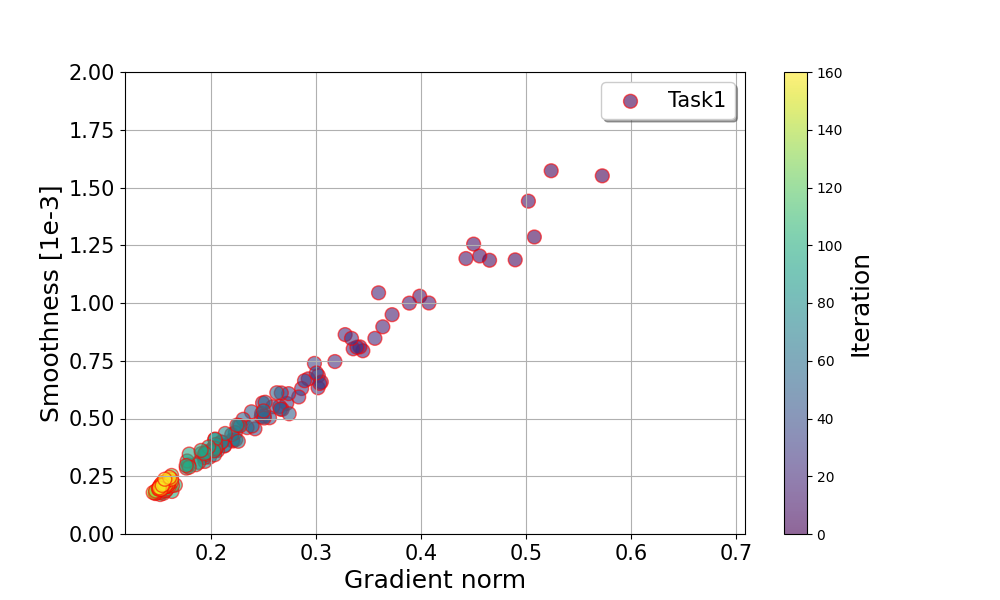}
    % \caption{}
    % \label{fig:first}
    \caption{Local smoothness constant vs. gradient norm for  training SegNet on CityScapes dataset of Task 1.} \label{fig:firsttaskgs}
    \vspace{-0.8cm}
\end{wrapfigure}
However, all existing works are limited by the standard $L$-smooth and bounded-gradient assumptions.
Nevertheless, a recent study \citep{zhang2019gradient} indicates that such assumptions may not necessarily be true for the training of neural networks and an alternative $(L_0,L_1)$-smoothness condition was observed and studied, which assumes the Lipschitz constant to be linear in the gradient norm and the gradient norm to be potentially infinite. Furthermore, this phenomenon has been consistently observed in our experiments (e.g., see \Cref{fig:firsttaskgs}). Interestingly, it has been widely observed that 
MGDA algorithms always converge even under such generalized smoothness conditions \citep{sener2018multi,liu2021conflict,xiao2024direction}. This naturally raises a thought-provoking question:
\begin{list}{$\bullet$}{\topsep=0.3ex \leftmargin=0.25in \rightmargin=0.in \itemsep =-0.022in}
   \item[Q.]\textit{Can the fundamental MGDA algorithm provably converge under the generalized smoothness condition, while achieving a sufficiently small CA distance?} 
\end{list}
This question remains open due to the {\bf following challenges}. 
% Previous studies have been unable to address this question because of the following limitations. 
The analysis of existing MOO methods cannot be generalized to this $(L_0,L_1)$-smoothness directly due to the possible unbounded smoothness or gradient norm. In addition, all existing works \citep{reisizadeh2023variance,zhang2019gradient,li2024convergence2, li2024convex,jin2021non,crawshaw2022robustness,chen2023generalized,zhang2024convergence} in generalized smoothness are limited to the single task problems, which are fundamentally different from the MOO problems since even though each single task is generalized smooth, the linear combination of these tasks is not necessarily generalized smooth. In this paper, we provide an affirmative answer to this question. Our main contributions are summarized below. 
\subsection{Our Contributions}
% We propose two single-loop MOO methods, Generalized Smooth Multi-objective Gradient descent (GSMGrad) and its stochastic variant SGSMGrad, and provide them with a comprehensive convergence analysis under the generalized smoothness condition in different settings. Our detailed contributions are listed below. 

 % In this paper, we address this question by conducting a convergence analysis of the fundamental MGDA algorithm under the generalized smoothness condition. Additionally, we ensure a low average CA distance in the vanilla MGDA and achieve a small iteration-wise  CA distance, facilitated by a warm start process.  
 
We establish a comprehensive convergence analysis for MGDA under the generalized smoothness condition, in both the deterministic and stochastic settings. 
Moreover, the analysis covers both the average CA distance and the iteration-wise CA distance. 
% , where the latter distance  

% Besides, with the help of a warm-start, we are able to guarantee small iteration-wise CA distance for our single-loop algorithm. Our detailed contributions are listed below.

% To provide a fundamental understanding of the MOO under more realistic assumptions and theoretically characterize the CA distance during the training process, we propose a novel single-loop algorithm, Generalized Smooth Multi-objective Gradient descent (GSMGrad) and its stochastic variant SGSMGrad, plus convergence rates under different settings.

%\vspace{0.1cm}
\noindent\textbf{Weakest assumptions in MOO.} In this paper, we investigate the generalized $\ell$-smooth assumption, where $\ell$ is a general non-decreasing function of gradient norm, and includes both the standard $L$-smooth and $(L_0,L_1)$-smooth assumptions as special cases. 
% In this paper, we study the case that $\ell$ is any function such that $\frac{a^2}{\ell(a)}$ is monotonically increasing for any $a>0$,{\color{red} (solved) KY: explain $\ell$ here in more details}
% which includes both the standard $L$-smooth and $(L_0,L_1)$-smooth assumptions as special cases.
This assumption finds many applications, such as LSTM models \citep{zhang2019gradient}, transformers \citep{crawshaw2022robustness}, distributionally robust optimization \citep{jin2021non} and higher-order polynomial functions \citep{chen2023generalized}. In addition, we {\bf do not make any bounded-gradient assumption}, which is required in previous analysis to ensure the bounded multi-gradient approximation.   
To the best of our knowledge, this is the first work to investigate generalized smoothness in MOO problems. Table \ref{table:reference} presents a detailed comparison of assumptions with existing analyses.
% and our $\ell$-smooth assumption is the weakest assumption in MOO and even the weakest assumption in single-task optimizations.

%\vspace{0.1cm}
\noindent\textbf{Convergence analysis.}
% Generalized smooth MGDA algorithm and its stochatsic variant are easy to implement by updating the weights $w$ of objectives and model parameters $x$ simultaneously via 
% a single-loop structure. 
% A warm-start initialization sub-procedure is also introduced at the beginning of both algorithms to ensure a sufficiently small iteration-wise CA distance under the single-loop structure. We also propose a computation- and memory-efficient variant named MGDA-FA by updating the objective weights $w$ using only forward passes of $F(\cdot)$ rather than gradient $\nabla F$, which effectively reduces $O(K)$ time and space to $O(1)$ without hurting the performance guarantee.        
% We provide a comprehensive analysis of MGDA
% under the generalized $\ell$-smooth condition in both deterministic and stochastic settings. 
% {\color{red} KY: perhaps we  wan to mention that 'Stochastic MGDA uses double-sampling strategy from \cite{xiao2024direction} to ensure unbiased multi-gradient approximation.'}
We first show that the vanilla MGDA method can converge to an $\epsilon$-accurate Pareto stationary point, while guaranteeing a small $\epsilon$-level average CA distance. A warm start process is introduced before the main loop of MGDA to achieve a more aggressive $\epsilon$-level iteration-wise CA distance. In the stochastic setting, we provide the convergence guarantee for MGDA with a double sampling, which was introduced by \cite{xiao2024direction,chen2024three} to obtain unbiased multi-gradient approximations. Furthermore, 
% because it contributes to have a more precise control on it. In addition, the stochastic MGDA employs the double-sampling strategy from \citep{xiao2024direction} to ensure the unbiased multi-gradient approximation. 
we analyze a computation- and memory-efficient variant of MGDA, named MGDA with fast approximation (MGDA-FA), which updates the objective weights $w$ using only forward passes of $F(\cdot)$ rather than gradient $\nabla F$, effectively reducing $O(K)$ time and space to $O(1)$ without hurting the performance guarantee.

\vspace{0.1cm}
\noindent\textbf{Sample complexity comparison.} To achieve an $\epsilon$-accurate Pareto stationary point and an $\epsilon$-level average CA distance, we show that MGDA require $\mathcal O(\epsilon^{-2})$ and $\mathcal O(\epsilon^{-4})$ samples in the deterministic and stochastic settings, respectively. {\bf Both of these complexities match with the existing best results.}  
Furthermore, to achieve an $\epsilon$-level iteration-wise CA distance, MGDA with warm start requires an increased number of samples, at the order of $\mathcal{O}(\epsilon^{-11})$ and $\mathcal{O}(\epsilon^{-17})$ respectively, in both deterministic and stochastic scenarios, due to smaller step sizes and mini-batch data sampling, shown in \Cref{table:reference}. Typically, achieving an {\bf $\epsilon$-level iteration-wise CA} distance results in much higher sample complexity, such as $\mathcal O(\epsilon^{-24})$ in \cite{fernando2022mitigating}, $\mathcal O(\epsilon^{-12})$\setcounter{footnote}{0} \footnote{The order differs from \citep{chen2024three} due to different definitions of the $\epsilon$-accurate Pareto stationarity and is taken when both $\epsilon$-accurate Pareto stationarity and $\epsilon$-level iteration-wise CA distance can be achieved.} in \cite{chen2024three} and $\mathcal O(\epsilon^{-12})$ in \cite{xiao2024direction} for non-convex stochastic setting. Moreover, we show that MGDA-FA achieves the same performance guarantee as vanilla MGDA.

 % GSMGrad and SGSMGrad require an increased number $\mathcal O(\epsilon^{-11})$ and $\mathcal O(\epsilon^{-17})$ of samples in the deterministic and stochastic settings, due to smaller step sizes and data sampling noise, respectively.   

 % the total sample complexities required are $\mathcal O(\epsilon^{-2})$ for deterministic setting, and $\mathcal O(\epsilon^{-4})$ for stochastic setting. These complexities match the optimal level for results of single-task optimizations. To guarantee the iteration-wise CA distance in the order of $\mathcal O(\epsilon)$, a warm start process is required to bound the initial CA distance. The sample complexity increases to $\mathcal O(\epsilon^{-11})$  for the deterministic setting due to smaller step sizes and $\mathcal O(\epsilon^{-17})$ due to the leveraging of mini-batches for the stochastic setting. Moreover, GSMGrad-FA has the same convergence rate as GSMGrad.

% \vspace{0.2cm}
% \noindent\textbf{Supportive experiments.} Our experiments on the MTL benchmark  Cityscapes~\citep{cordts2016cityscapes} and NYU-v2 \citep{silberman2012indoor} validate our theory and demonstrate the effectiveness of our proposed algorithms. 

% conduct experiments on the MTL benchmark  Cityscapes \citep{cordts2016cityscapes} and demonstrate that the proposed algorithm GSMGrad can achieve improved performance and a better balance on different tasks than existing state-of-the-art methods. Besides, we show that GSMGrad-FA can achieve a comparable result and will be more useful when there is a memory limitation.
\vspace{-0.3cm}
\subsection{Related Works}
\vspace{-0.2cm}

\textbf{Gradient-based multi-objective optimization.}  
A variety of gradient manipulation techniques have emerged for simultaneous learning of multiple tasks. One prevalent category of methods adjusts the weights of various objectives according to factors such as uncertainty~\citep{kendall2018multi}, gradient norm~\citep{chen2018gradnorm}, and training complexity~\citep{guo2018dynamic}. Methods based on MOO have garnered increased attention due to their systematic designs, enhanced training stability, and model-agnostic nature. For instance, \cite{sener2018multi} framed Multi-Task Learning (MTL) as a MOO problem and introduced an optimization method akin to MGDA~\citep{desideri2012multiple}.
Afterward, many MGDA-based methods have been proposed to mitigate gradient conflict with promising empirical performance. Among them, PCGrad  \citep{yu2020gradient} avoids conflict by projecting the gradient of each task on the norm plane of other tasks. GradDrop~\citep{chen2020just} randomly drops out conflicted gradients. CAGrad~\citep{liu2021conflict} adds a constraint on the update direction to be close to the average gradient. NashMTL~\citep{navon2022multi} and 
FairGrad~\citep{ban2024fair} formulated MTL as a bargaining game and a resource allocation problem, respectively. Theoretically, \cite{fernando2022mitigating} proposed a provably convergent stochastic MOO method named MoCo
% the understanding of the convergence and complexity of MOO is not well-developed until recently and few of them focus on the stochastic settings.
based on an auxiliary tracking variable for gradient approximation.  \cite{chen2024three} characterized the trade-off among optimization, generalization, and conflict avoidance in MOO. 
\cite{xiao2024direction} proposed and analyzed a stochastic MOO method named SDMGrad with a preference-oriented regularizer. However, all these works rely on the $L$-smoothness and bounded-gradient assumptions. In contrast, this paper focuses on the MOO problems with generalized $\ell$-smooth objectives.

\begin{table}[t]
\centering\small
% \vspace{-0.3cm}
\setlength{\tabcolsep}{0.8mm}{
% \begin{tabular}{|c c c c |} 
%  \hline
%  Method& Smoothness \footnotemark[1] & Assumption\footnotemark[2]&Sample Complexity \\ [0.3ex] 
%  \hline\hline
%  SMG \citep{liu2021stochastic} & (LS) &(BG)&N/A\footnotemark[3]\\
%  CR-MOGM\citep{zhou2022convergence} & (LS)  &(BF), (BG)&$\mathcal O(\epsilon^{-4})$ \\ 
%   MoCo\citep{fernando2022mitigating} & (LS)   &(BF), (BG)&$\mathcal O(\epsilon^{-4})$ \\ 
% MoDo\citep{chen2024three} & (LS)  &(BG)&$\mathcal O(\epsilon^{-4})$\\
% SDMGrad\citep{xiao2024direction}  & (LS)  & (BG)&$\mathcal O(\epsilon^{-4})$\\
% {\bf Stochastic MGDA (this paper)} &(GS)  & N/A&$\mathcal O(\epsilon^{-4})$\\
%  \hline
\begin{tabular}{|c c c c c|} 
 \hline
 Method& Smoothness \footnotemark[1] & Assumption\footnotemark[2]& Setting & Complexity \footnotemark[3]\\ [0.3ex] 
 \hline\hline
 CAGrad \citep{liu2021conflict} & (LS) & (OD) & Deterministic & N/A\\
 PCGrad \citep{yu2020gradient}  &(LS)& (C) or (BC) & Deterministic & N/A\\
 % FAMO \citep{liu2024famo} &(LS)&& Deterministic\\
 % FairGrad \cite{ban2024fair} &(LS)&& Deterministic\\
 SMG \citep{liu2021stochastic} & (LS) &(C) or (BC)& Stochastic & N/A\\
 CR-MOGM \citep{zhou2022convergence} & (LS)  &(BF) and  (BG)& Stochastic & $\mathcal{O}(\epsilon^{-4})$\\ 
  MoCo \citep{fernando2022mitigating} & (LS)   &(BF) and (BG)& Stochastic & $\mathcal{O}(\epsilon^{-4})$\\ 
MoDo \citep{chen2024three} & (LS)  &(BG)& Stochastic & $\mathcal{O}(\epsilon^{-4})$\\
SDMGrad \citep{xiao2024direction}  & (LS)  & (BG)& Stochastic & $\mathcal{O}(\epsilon^{-4})$\\
{\bf MGDA (Them. \ref{thm:gddeterministic} and \ref{thm:consistsmallcadistance} )} &{\bf (GS) } & {\bf N/A}&{\bf  Deterministic} & $\mathcal{O}(\epsilon^{-2})$\\
{\bf Stochastic MGDA (Them. \ref{thm:gdstochastic} and \ref{thm:sgdeachca})} &{\bf (GS) } & {\bf N/A}& {\bf Stochastic} & $\mathcal{O}(\epsilon^{-4})$\\
 \hline
 
\end{tabular}}
\caption{{Comparison to assumptions in existing analyses. MGDA \citep{desideri2012multiple} assumes the access of optimal update direction and step size for each iteration and gets an asymptotic result thus is omitted in the table.
Explanation on the upper footmarks: $1:$ (LS) indicates that the objectives are standard $L$-smooth while (GS) the objectives are generalized $\ell$-smooth as defined in Definition \ref{define:rlsmooth2}; $2:$ (OD) denotes the assumption of optimal direction in each iteration, (C) denotes the convex loss function assumption, (BC) denotes a lower bound on multi-task curvature, which is defined as $H(f,x,x')=\int_0^1\nabla f(x)^\top\nabla^2f(x+a(x-x'))\nabla f(x)da$ for a function $f$,
(BF) denotes the bounded function value assumption and (BG) denotes the bounded gradient assumption; $3:$ Sample complexity achieving $\epsilon$-accurate Pareto stationary point. 
N/A denotes no exploration on non-convex settings.
% $3:$ the analysis in \citep{liu2021stochastic} focuses on convex objective functions.
}
}
\label{table:reference}
\vspace{-0.5cm}
\end{table}

\noindent
\textbf{Generalized smoothness.} 
% The standard $L$-smoothness condition \citep{nemirovskij1983problem, ghadimi2013stochastic} has been widely studied. 
The generalized $(L_0,L_1)$-smoothness was firstly proposed by \cite{zhang2019gradient}, which was observed from extensive empirical experiments in training neural networks. A clipping algorithm was developed by \cite{zhang2019gradient} and the convergence rate was provided. Later, \cite{jin2021non} analyzed the convergence of a normalized momentum method. The SPIDER algorithm was also applied to solve generalized smooth problems in \cite{reisizadeh2023variance,chen2023generalized}, where \cite{chen2023generalized} studied a new notion of $\alpha$-symmetric generalized smoothness, which includes $(L_0,L_1)$-smoothness as a special case. Very recently, a new generalized $\ell$-smoothness condition was studied in \cite{li2024convex,li2024convergence2}, which is the weakest smoothness condition and includes all the smoothness conditions discussed above. However, all the existing works on generalized smoothness are limited to single-task optimizations and the understanding of MOO is insufficient. This paper provides the first study of MOO under the generalized $\ell$-smoothness condition.

\vspace{-0.2cm}
\section{Preliminaries}
\vspace{-0.2cm}
\subsection{Generalized smoothness}
The standard $L$-smoothness condition is widely investigated in existing optimization studies \citep{ghadimi2013stochastic,ghadimi2016mini}, which assumes a function $f: \mathcal X \to \mathbb R$ to be $L$-smooth if there exists a bounded constant $L$ such that for any $x,y\in \mathcal X$,
% \begin{flalign}
    $\|\nabla f(x)-\nabla f(y)\|\le L\|x-y\|.$
% \end{flalign}
Nevertheless, recent studies show that in the training of neural networks such as LSTM models \citep{zhang2019gradient}, transformers \citep{crawshaw2022robustness}, distributionally robust optimization \citep{jin2021non} and high-order polynomials functions \citep{chen2023generalized}, the standard $L$-smoothness assumption does not hold. Instead, a generalized $(L_0,L_1)$-smoothness assumption was observed and studied in the training of LSTM models in \cite{zhang2019gradient}, which assumes that for any $x\in \mathcal X$,
% \begin{flalign}
    $\|\nabla^2 f(x) \|\le L_0+L_1\|\nabla f(x)\|. $
% \end{flalign}
This assumption implies the Lipschitz constant is potentially unbounded and reduces to the $L$-smoothness if $L_1=0$. Later, a more generalized assumption was proposed and studied in \cite{li2024convex}:
\begin{definition}\label{define:rlsmooth2}
(Generalized $\ell$-smoothness, Definition 1 in \cite{li2024convex}). A real-valued differentiable function $f:\mathcal{X}\rightarrow \mathbb{R}$ is  generalized $\ell$-smooth if  $\|\nabla^2 f(x)\|\le \ell(\|\nabla f(x)\|)$  almost everywhere in $\mathcal X$, {\bf where $\ell : [0, +\infty) \rightarrow (0, +\infty)$ is a continuous non-decreasing function.}
\end{definition}
The $(L_0,L_1)$-smoothness is a special case of generalized $\ell$-smoothness, where $\ell(a)=L_0+L_1a$. 
%In this paper, we study the case where the $\ell$ function is sub-quadratic, which implies function $\varphi$ is monotonically increasing.
Another definition of generalized smooth is widely used and equivalent to the $\ell$-smoothness:
\begin{definition}\label{define:rlsmooth}
(($r, \ell$)-smoothness, Definition 2 in \cite{li2024convex}). A real-valued differentiable function $f:\mathcal{X}\rightarrow \mathbb{R}$ is $(r, \ell)$-smooth 
if 1) for any $x \in \mathcal{X}, B(x,r(\|\nabla f(x)\|)) \in \mathcal{X}$, and 2) for any $x_1,x_2 \in B(x,r(\|\nabla f(x)\|)), \|\nabla f(x_1)- \nabla f(x_2)\| \leq \ell(\|\nabla f(x)\|) \|x_1 - x_2\|$, where for continuous functions $r, \ell : [0, +\infty) \rightarrow (0, +\infty)$, $r$ is non-increasing, $\ell$ is non-decreasing and $B(x, R)$ is the Euclidean ball centered at $x$ with radius $R$.
\end{definition}
In $B(x, r(\|\nabla  f(x)\|))$, $f$ is also $L$-smooth where $L=\ell(\|\nabla f(x)\|)$. Definitions \ref{define:rlsmooth2} and \ref{define:rlsmooth} are equivalent \citep{li2024convex}:
An $(r, \ell)$-smooth function is $\ell$-smooth; and a $\ell$-smooth function satisfying Assumption \ref{ass:diffandlowerbound}
is $(r, m)$-smooth with $m(u) := \ell(u + a)$ and $r(u) := a/m(u)$ for any $a > 0$.
\subsection{Pareto concepts in multi-objective optimization (MOO)}
% Since MOO can hardly find a point that minimizes all individual objective functions simultaneously, 
MOO aims to find points at which there is no common descent direction for all objectives. Considering points $x_1, x_2\in\mathbb{R}^m$, we claim that $x_1$ dominates $x_2$ if $f_i(x_1)\geq f_i(x_2)$ for all $i\in[K]$ and $F(x_1)\neq F(x_2)$. We say a point is \textit{Pareto optimal} if it is not dominated by any other point. In other words, we cannot improve one objective without compromising another when we reach a \textit{Pareto optimal point}. In non-convex settings, MOO aims to find a \textit{Pareto stationary point} defined as follows.
\begin{definition}\label{def:pspoints}
$x\in\mathbb{R}^m$ is a Pareto stationary point if $\min_{w\in\mathcal{W}}\|\nabla F(x)w\|^2=0$, {where  $\mathcal{W}$ is the probability simplex over $[K]$}. $x$ an $\epsilon$-accurate Pareto stationary point if $\min_{w\in\mathcal{W}}\|\nabla F(x)w\|^2\leq\epsilon^2$.
%under deterministic setting and $\mathbb{E}[\min_{w\in\mathcal{W}}\|\nabla F(x)w\|^2]\leq\epsilon^2$ under stochastic setting, where $\mathcal{W}$ is the probability simplex defined on $[K]$.
\end{definition}
 
\subsection{Existing MOO Algorithms}
\textbf{Deterministic MOO.} One of the big challenges of MOO is the gradient conflict, i.e., the gradients of different objectives may vary heavily in scale such that the largest gradient dominates the update direction. As a result, the performance of those objectives with smaller gradients \citep{yu2020gradient} may be significantly compromised. As the most fundamental MOO algorithm, MGDA tends to find a balanced update direction for all objectives by considering the minimum improvement across all objectives and maximizes it by solving the following problem
\begin{align}\label{eq:maxmin}
\max_{d\in\mathbb{R}^m}\min_{i\in[K]}\Big\{\frac{1}{\alpha}(f_i(x)-f_i(x-\alpha d))\Big\}\approx\max_{d\in\mathbb{R}^m}\min_{i\in[K]}\langle\nabla f_i(x),d\rangle,
\end{align}
where $\alpha$ is the step size, $d$ is the update direction, and the first-order Taylor approximation is applied at $x$. To efficiently solve the above problem in \cref{eq:maxmin}, we substitute the following relation
\begin{align}\label{eq:maxmind}
\max_{d\in\mathbb{R}^m}\min_{i\in[K]}\langle \nabla f_i(x),d\rangle-\frac{1}{2}\|d\|^2=\max_{d\in\mathbb{R}^m}\min_{w\in\mathcal{W}}\Big\langle\sum_{k=1}^K\nabla f_i(x)w_i,d\Big\rangle-\frac{1}{2}\|d\|^2,
\end{align}
where the regularization term $-\frac{1}{2}\|d\|^2$ is to regulate the magnitude of our update direction. The solution to the problem in \cref{eq:maxmind} can be obtained by solving the following problem \citep{desideri2012multiple}: 
\begin{align}\label{eq:cadirection}
d^*=\nabla F(x)w^*;\;\; s.t.\;w^*\in\arg\min_{w\in\mathcal{W}}\frac{1}{2}\|\nabla F(x)w\|^2.
\end{align}
The above approach has been widely used in various variants of MGDA such as SDMGrad, CAGrad, and PCGrad \citep{xiao2024direction,yu2020gradient,liu2021conflict}.

%\vspace{0.2cm}
\noindent\textbf{Stochastic MOO.} SMG  \citep{liu2021stochastic} is the first stochastic MGDA. It directly replaces the gradients with stochastic gradients and the update rule becomes
\begin{align*}
    d_s^*=\nabla F(x;s)w_s^*;\;\;s.t.\;w_s^*\in\arg\min_{w\in\mathcal{W}}\frac{1}{2}\|\nabla F(x;s)w\|^2,
\end{align*}
where $\nabla F(x;s)$ is the estimate of $\nabla F(x)$ based on the sample $s$.
However, this leads to a \textbf{biased} gradient estimation of the update direction $d_s^*$, and thus it requires an increasing batch size. To solve this issue, another work MoCo \citep{fernando2022mitigating} introduces a tracking variable $Y$ as a stochastic estimation of the true gradient. Afterward, a double-sampling strategy is proposed by \cite{chen2024three,xiao2024direction} to generate a near-unbiased update direction. 

{\bf Strong assumptions in analysis.} All the works mentioned above require bounded gradients such as \cite{chen2024three,fernando2022mitigating,xiao2024direction} or $L$-smoothness such as \cite{ban2024fair,liu2021conflict,navon2022multi,yang2024federated}. Their analyses do not apply to generalized $\ell$-smoothness objectives, since the Lipschitz constant is potentially infinity. %Thus the step size we choose becomes $0$, making the algorithms diverge. 

% {\color{red} (solved) KY: Mention about their assumptions like bounded gradients and also mention none of these analyses can be applied to generalized smoothness case}

% \begin{align}
% w_\rho^*=\arg\min_{w\in\mathcal{W}}\frac{1}{2}\|\nabla F(x)w\|^2+\frac{\rho}{2}\|w\|^2
% \end{align}
% \begin{remark}[On the Lipschitz continuity of $w_\rho^*(x)$]
    
% \end{remark}
\subsection{Conflict-avoidant (CA) direction and CA distance}\label{sec:CAs}
We call the update direction $d^*$ in \cref{eq:cadirection} the \textit{conflict-avoidant} (CA) direction since it mitigates gradient conflict. Though it may not be feasible to calculate the exact CA direction, we aim to find an update direction to be close to the CA direction. Therefore, measuring the gap between the CA direction and the estimated update direction is important, which we define as the CA distance.
\begin{definition}\label{def:cadistance}
$\|d-d^*\|$ is the CA distance between estimated update direction $d$ and CA direction $d^*$.
%under deterministic setting while $\|\mathbb{E}_\xi[d_\xi]-d^*\|$ \zhang{should it be $\mathbb{E}_\xi\|[d_\xi-d^*\|]$?} denotes CA distance between stochastic update direction $d_\xi$ and CA direction $d^*$ under stochastic setting.
\end{definition}
The larger the CA distance is, the further the estimated update direction will be away from the CA direction, and the more conflict there will be. In single-loop algorithms, MoCo \citep{fernando2022mitigating} ensures the average CA distance over iteration is of the order of $\epsilon$, while MoDo \citep{chen2024three} guarantees the $\epsilon$-order iteration-wise CA weight distance (as stated in their Theorem 3.4). Meanwhile, the double-loop algorithm SDMGrad \citep{xiao2024direction} guarantees an $\epsilon$-order CA distance in every iteration. In this work, we analyze the CA distance in both cases and provide convergence results.

\vspace{-0.2cm}
\section{MGDA Algorithms under Generalized Smoothness}
\vspace{-0.2cm}
% In this section, we present MGDA and its stochastic version both are easy to implement with a simple single-loop structure. The warm start option depends on the requirement of CA distance. We also show the efficient variant of it, called MGDA-FA with constant-level computational and memory costs.
% a practical implementation method in our algorithm which is faster in speed and more memory-saving.
\subsection{MGDA with and without Warm Start}
It has been shown in \cref{eq:cadirection} that MGDA needs to approximate the optimal weight $w^*$ and the optimal updating direction $d^*$.  
% We start to adopt MGDA in our method by computing an approximated weight $w_t$ and an update direction $d_t$  according to \cref{eq:cadirection} where $t$ is the iteration number.
However, since the optimal weight $w^*$ of the convex function is not unique, we deal with this issue by adding an $\ell_2$ regularization term and the problem becomes
\begin{align}\label{eq:weightrho}
w_\rho^*=\arg\min_{w\in\mathcal{W}}\frac{1}{2}\|\nabla F(x)w\|^2+\frac{\rho}{2}\|w\|^2.
\end{align}
Besides the benefit of a unique solution, adding an $\ell_2$ regularization term also makes $w_\rho^*(x)$ Lipschitz continuous \citep{fernando2022mitigating}. Note that $w^*(x)$ may not be Lipschitz continuous because  $\nabla F(x)^\top\nabla F(x)$ may not be positive definite. Nevertheless, the analysis of CA distance is difficult because $w^*$ may not be Lipschitz continuous. Thus, we will characterize the gap between $w^*$ and $w_\rho^*$ plus the change of $w_\rho^*$ after adding this $\ell_2$ regularization term. As a result, the update rules become Lines 5-6 in \Cref{alg:1}. We first update $w_t$ by a projected gradient descent process and compute the update direction $d_t=\nabla F(x_t)w_t$ to update model parameters.

For our single-loop algorithm, CA distance is proportional to the term $\|w_t-w_{t,\rho}^*\|$, which decreases as the algorithm iterates with some error terms controlled by appropriately chosen small step sizes. If we initialize $w_0$ randomly, $\|w_0-w_{0,\rho}^*\|$ will be a constant order, and so will the first CA distance. Meanwhile, we can only get an $\epsilon$-order CA distance after a certain iteration number $t^\prime>1$ when $\|w_{t^\prime}-w_{t^\prime,\rho}^*\|$ takes an $\epsilon$ order. Thus, we introduce an extra warm start process using \Cref{alg:warmstart} to guarantee the new $w_0$ is close enough to $w_{0,\rho}^*$ and a small level CA distance in every iteration. {\bf However, this warm start process is not needed if we only require a small averaged CA distance.}

\begin{algorithm}[htbp]
 \small
   \caption{Single loop MGDA with and without warm start}
   \label{alg:gs}
\begin{algorithmic}[1]\label{alg:1}
    \STATE {\bfseries Initialize:}  model parameters $x_0$, weights $w_0$ and a constant $\rho$
    % \STATE{$x_1=x_0-\alpha_0\nabla F(x_0)w_0$}
    \STATE{\textit{Option I: }$w_0$=\textbf{warm-start}($w_0$, $x_0$, $\rho$) \textcolor{blue}{\# for analyzing iteration-wise CA distance}}
    \STATE{\textit{Option II: }$w_0\leftarrow w_0$ \hspace{2.25cm} \textcolor{blue}{\# for analyzing averaged CA distance}}
    \FOR{$t=0, 1,...,T-1$}
%     \STATE{$w_{t+1}=\proj_\mathcal{W} \big (w_t-\beta_t
% [G(x_t)^\top G(x_t)w_t+\rho w_t]\big)$, \textcolor{red}{FAMO}}
    \STATE{$w_{t+1}=\Pi_\mathcal{W}\big(w_t-\beta[\nabla F(x_t)^\top\nabla F(x_t)w_t+\rho w_t]\big)$}
    \STATE{$x_{t+1}=x_t-\alpha \nabla F(x_t)w_t$}
    
\ENDFOR
\end{algorithmic}
\end{algorithm}

\vspace{-0.3cm}
\begin{algorithm}[htbp]
 \small
   \caption{\textbf{warm-start}($w_0$, $x_0$, $\rho$)}
   % \label{alg:warmstart}
\begin{algorithmic}[1]\label{alg:warmstart}
    % \STATE {\bfseries Input:}  model parameters $x_0$ and weights $w_0$
    % \STATE{$x_1=x_0-\alpha_0\nabla F(x_0)w_0$}
    \FOR{$n=0, 1, ..., N-1$}
    \STATE{$w_{n+1}=\Pi_\mathcal{W}\big(w_n-\beta'[\nabla F(x_0)^\top\nabla F(x_0)w_n+\rho w_n]\big)$}
    \ENDFOR
    \STATE{\textbf{Output} $w_N$} 
\end{algorithmic}
\end{algorithm}

% \subsection{Clipped Multi-objective Gradient descent}
% \begin{algorithm}[htbp]
%    \caption{Clipped Multi-objective Gradient descent (CMGrad)}
%    \label{alg:clip}
% \begin{algorithmic}
%     \STATE {\bfseries Initialize:}  model parameters $x_0$ and weights $w_0$
%     \FOR{$t=0, 1,...,T-1$}
%     \STATE{$x_{t+1}=x_t-h_t \nabla F(x_t)w_t$, where $h_t:=\min\{\textcolor{red}{confirm it later}\}$}
%     \STATE{update $w_t$ according to \cref{eq:updaterule}}
% \ENDFOR
% \end{algorithmic}
% \end{algorithm}
\vspace{-0.3cm}
\subsection{Stochastic MGDA with Double Sampling}
In the stochastic setting, our algorithm keeps the same structure, having a warm start process if we aim to control the CA distance in every iteration. In \Cref{alg:warmstart}, we do the same projected gradient descent without using stochastic gradients. This is because we only need to
compute $\nabla F(x_0)^\top\nabla F(x_0)$ once and reuse it in the whole loop, which does not bring a computational burden. Then in the update loop, we update the weight and model parameters accordingly. We use a double-sampling strategy here to make the weight gradient estimator unbiased \citep{xiao2024direction} such that $d_t$ is a near-unbiased multi-gradient
$\mathbb{E}[\nabla G_2(x_t)^\top\nabla G_3(x_t)w_t+\rho w_t]=\nabla F(x_t)^\top\nabla F(x_t)w_t+\rho w_t,$
where $\nabla G_2(x_t)$ and $\nabla G_3(x_t)$ are independent and unbiased estimates of $\nabla F(x_t)$. 
Similarly, we do not involve a warm start process if we require the average CA distance to be small.

\vspace{-0.2cm}
\subsection{MGDA with Fast Approximation}
It can be seen from \Cref{alg:1} and \Cref{alg:2} that MGDA requires $\mathcal{O}(K)$ space and time to compute and store all task gradients at each iteration for updating the weight $w_t$. This becomes a drawback when the number of tasks or the model size is large. Motivated by \cite{liu2024famo}, one solution is to use the Taylor Theorem to approximate the gradient for updating the weight $w_t$ as 
\begin{align*}
{F}(x_t)-{F}(x_{t+1})=\nabla F(x_t)^\top(x_t-x_{t+1})-R(x_t)=\alpha\nabla F(x_t)^\top\nabla F(x_t)w_t-R(x_t), 
\end{align*}
where $R(x_t)$ is the remainder term and it takes the order $R(x_t)=o(\|x_t-x_{t+1}\|^2)$, which can be made sufficiently small by adjusting the step size. By incorporating this fast approximation (FA) in \Cref{alg:1}, we then  present MGDA-FA in \Cref{alg:gs-famo} (shown in Appendix \ref{sec:algorithm4}), where $x_t$ is updated  along the update direction $d_t=\nabla F(x_t)w_t$ to get $x_{t+1}$ following by the update rule of $w_t$ 
\begin{align}\label{eq:updaterule}
w_{t+1}=\Pi_\mathcal{W}\Big(w_t-\beta\Big[\frac{F(x_t)-F(x_{t+1})}{\alpha}+\rho w_t\Big]\Big).
\end{align}
As a result, in the model parameters update process, MGDA-FA only requires one backward process by calculating the gradient of $F(x_t)w_t$ w.r.t. $x_t$ without storing it, and additional forward processes to compute $F(x_{t+1})$ in the weight update process. This approach saves computational and memory costs in the practical implementation significantly. {\bf Most importantly, we provide a theoretical guarantee for this efficient method (in \cref{eq:updaterule})}. %\zhang{is it ok to move the famo algo to the appendix?}
% theoretical analysis when the CA distance is guaranteed to be small in every iteration under the deterministic setting.
% {\color{red} KY: write out the algorithm for this one}
\begin{algorithm}[h]
   % \vspace{-0.1cm}
   \small
   \caption{Stochastic MGDA with Double Sampling}
\begin{algorithmic}[1]\label{alg:2}
    \STATE {\bfseries Initialize:}  model parameters $x_0$, weights $w_{0}$ and a constant $\rho$
    % \FOR{$n=0, 1, ..., N-1$}
    % \STATE{$w_{0,n+1}=\Pi_\mathcal{W}\big(w_{0,n}-\beta_n[\nabla F(x_0)^\top\nabla F(x_0)w_{0,n}+\rho w_{0,n}]\big)$}
    % \ENDFOR
    \STATE{\textit{Option I: }$w_0$=\textbf{warm-start}($w_0$, $x_0$, $\rho$) \textcolor{blue}{\# for analyzing iteration-wise CA distance}}
    \STATE{\textit{Option II: }$w_0\leftarrow w_0$ \hspace{2.35cm}\textcolor{blue}{\# for analyzing averaged CA distance}}
    \FOR{$t=0, 1,...,T-1$}
%     \STATE{$w_{t+1}=\proj_\mathcal{W} \big (w_t-\beta_t
% [G(x_t)^\top G(x_t)w_t+\rho w_t]\big)$, \textcolor{red}{FAMO}}
        \STATE{$x_{t+1}=x_t-\alpha 
 \nabla G_1(x_t)w_t$}
        \STATE{$w_{t+1}=\Pi_\mathcal{W}\big(w_t-\beta[\nabla G_2(x_t)^\top\nabla G_3(x_t)w_t+\rho w_t]\big)$ \textcolor{blue}{\# double sampling}}

\ENDFOR
\end{algorithmic}
\end{algorithm}
\vspace{-0.4cm}

\section{Convergence Analysis under Average CA distance}\label{sec:averageca}
In this section, we provide the theoretical results for Algorithms \ref{alg:1} and \ref{alg:2} {\bf without warm starts} to obtain an $\epsilon$-accurate Pareto stationary point, with the average CA distance over iterations in $\mathcal O(\epsilon)$. 

\subsection{Deterministic setting}
% We provide the analysis for Algorithm \ref{alg:1} for our deterministic setting with the following assumptions:
\begin{assumption}\label{ass:diffandlowerbound}
Each objective function $f_i$ $\forall i\in[K]$ is twice differentiable and lower bounded by $f_i^* := \inf_{x\in\mathbb{R}^m} f_i(x) > -\infty$.
\end{assumption}
\begin{assumption}\label{ass:lsmooth}
Each objective function $f_i$ $\forall i\in[K]$ is generalized $\ell$-smooth defined in Definition \ref{define:rlsmooth2}, where $\ell : [0, +\infty) \rightarrow (0, +\infty)$ is a continuous non-decreasing function such that $\varphi(a)=\frac{a^2}{2\ell(2a)}$ is monotonically increasing for any $a\ge 0$.
\end{assumption}
These assumptions are the most relaxed ones in existing MOO works since they directly assume objective smoothness or gradient/function value boundness \citep{liu2021conflict,fernando2022mitigating,navon2022multi,xiao2024direction,chen2024three,yang2024federated,ban2024fair}. It also includes the widely studied standard $L$-smoothness \citep{nemirovskij1983problem,ghadimi2013stochastic,ghadimi2016mini}, $(L_0,L_1)$-smoothness \citep{zhang2019gradient} as special cases. Moreover, for any $0\le\gamma\le 2$ and $x\in \mathcal X$, our assumption even holds for function $f$ such that $\|\nabla^2 f(x)\|\le L_0+L_1\|\nabla f(x)\|^\gamma$, where $\gamma$ are limited to $[0,1]$ in \cite{chen2023generalized}.

%We then provide our theoretical results. 
Let $c> 0$ and $F>0$ be some constants such that 
$
    \Delta+c\le F,
$
where $\Delta=\max_{i\in[K]} \{f_i(x_0)-{f_i^*}\}$.
Define $M=\sup\{z\ge0|\varphi(z)\le F\}$. We then have the following convergence rate for Algorithm \ref{alg:1}:
\begin{theorem}\label{thm:gddeterministic}
 Let Assumptions \ref{ass:diffandlowerbound} and \ref{ass:lsmooth} hold. Set {\small$\beta=
\mathcal O(\frac{1}{M^2}), \alpha=
\mathcal O(\frac{1}{M^2}+\frac{1}{M\ell(M+1)}), 
T= \max\Big(\Theta \big(\frac{1}{\alpha \epsilon^2}\big), \Theta\left(\frac{1}{\beta \epsilon^2}\right)\Big)$} and $ \rho= \mathcal O(\epsilon^2)$. We then have that 
    $\frac{1}{T}\sum_{t=0}^{T-1}\|\nabla F(x_t)w_t\|^2\le \epsilon^2.$
\end{theorem}
The full version with detailed constants and detailed proof can be found in Appendix \ref{proof:gddeterministic-avg}. \Cref{thm:gddeterministic} provides the first convergence rate to obtain an $\epsilon$-accurate Pareto stationary point for MOO problems with generalized $\ell$-smooth objectives. Moreover, it achieves the optimal sample complexity in the order of $\mathcal O(\epsilon^{-2})$for GD with a single standard $L$-smooth objective \citep{carmon2020lower}. The MOO problems with generalized $\ell$-smooth objectives are challenging due to {\bf two reasons}: 1) $\|\nabla F(x)\|$ is potentially unbounded in our generalized $\ell$-smoothness setting, making all existing analysis in MOO \citep{liu2021conflict,fernando2022mitigating,navon2022multi,xiao2024direction,chen2024three,yang2024federated,ban2024fair} not applicable. 2) the update of $x$ includes all gradient information from each task, making the existing adaptive methods for single generalized smooth functions invalid. 

To solve the challenges in \Cref{thm:gddeterministic}, we find that a bounded function value implies a bounded gradient norm. Thus in our proof, we use induction to show that with parameters selected in \Cref{thm:gddeterministic}, for any $w\in \mathcal W$ and $t\le T$, we have that $F(x_t)w$ is upper bounded by $F$. Consequently, for any $i\in[K]$, we have that $\|\nabla f_i(x)\|\le M$, which solves the unbounded gradient norm problem in our generalized smoothness setting. Then we can show that $\|\nabla F(x_t)w_t\|$ converges.

\begin{corollary}\label{coro:gdaverageca}
Under the same setting in \Cref{thm:gddeterministic}, $\frac{1}{T}\sum_{t=0}^{T-1}\|\nabla F(x_t)w_t-\nabla F(x_t)w_t^*\|^2=\mathcal{O}(\epsilon^2)$.
\end{corollary}
The proof is available in Appendix \ref{proof:gdaverageca}.  Corollary \ref{coro:gdaverageca} shows that the average CA distance converges. 
\subsection{Stochastic setting}
In the stochastic setting, we assume that we have access to an unbiased stochastic gradient ${\nabla}f_i(x;s)$ instead of the true gradient $\nabla f_i(x)$, where $s$ is the collected samples. To prove convergence, we have the following assumption. 
\begin{assumption}\label{ass:boundedvariance}
There exists some $\sigma\geq 0$ such that $\mathbb{E}[\|{\nabla}f_i(x;s)-\nabla f_i(x)\|^2]\leq\sigma^2 $ for any $ i\in[K]$.
\end{assumption}
Assumption \ref{ass:boundedvariance} indicates bounded gradient variances, which is widely studied \citep{xiao2024direction,li2024convex, fernando2022mitigating}. {At time $t$ in the stochastic model, let $s_{t,i}=(s_{t,i,1},s_{t,i,2},...,s_{t,i,k})$ be the $i$-th collection of samples  and $F(x_t;s_{t,i})=(f_1(x_t;s_{t,i,1}),f_2(x_t;s_{t,i,2}),...,f_k(x_t;s_{t,i,k})).$ Note that $i\in [3]$ because we have 3 samples in \Cref{alg:2}. We choose 
$G_i(x_t)=F(x_t;s_{t,i})$. Define $\varepsilon_{t,i}=(\varepsilon_{t,i,1},\varepsilon_{t,i,2},...,\varepsilon_{t,i,k})=\nabla F(x_t)-\nabla G_i(x_t)$. }

Let $F, c>0$ and $0<\delta\le \frac{1}{2}$ be some constants such that  $ F\ge \frac{4(\Delta+c)}{\delta}$ and $M=\sup\{z\ge0|\varphi(z)\le F\}$. Define the following random variables $\tau_1=\min\{t | \exists i \in [K], f_i(x_{t+1})-f_i^*>F\} \wedge T, \tau_2=\min\{t | \exists i \in [K], j\in [3] ,\|\varepsilon_{t,j,i}\|>\frac{L_0}{\sqrt{\alpha\rho} }\}\wedge T, \tau_3=\min\{t | \exists i,j \in [K],\|\varepsilon_{t,2,i}\|\|\varepsilon_{t,3,j}\|>\frac{L_1}{\sqrt{\alpha\rho} }\}\wedge T$ and $\tau=\min\{\tau_1, \tau_2, \tau_3\}$,
%\begin{flalign}\label{eq:introstopping}
%    &\tau=\min\{\tau_1, \tau_2, \tau_3\},
%\end{flalign} 
where $L_0,L_1>0$ are some constants and $a\wedge b$ denotes $\min(a,b)$.
We have the following theorem for Algorithm \ref{alg:2}::
\begin{theorem}\label{thm:gdstochastic}
 Let Assumptions \ref{ass:diffandlowerbound}, \ref{ass:lsmooth}, \ref{ass:boundedvariance} hold. Set {\small $\rho= \mathcal O(\delta^2 \epsilon^2), \beta= \mathcal O(\min\{\frac{\delta\epsilon^2}{M^4},\frac{\rho}{M^2}\}), \alpha= \mathcal O(\min\{\beta, \frac{\rho}{\ell(M+1)^2},  \frac{\rho}{M^2}, \frac{\delta }{\rho T},\frac{1}{\beta TM^2}\})$ and $T=\Theta (\max\{\frac{1}{\delta \alpha \epsilon^2},\frac{M^2}{\delta^2 \epsilon^4}\})$.} We then have that with the probability at least $1-\delta$,
% \begin{flalign}
    {\small $\frac{1}{T}\sum_{t=0}^{T-1}\|\nabla F(x_t)w_t\|^2=\mathcal{O}(\epsilon^2)$}
% \end{flalign}
.
\end{theorem}
The full version and detailed proof can be found in Appendix \ref{proof:stochasticaverage}. When we set $\alpha, \beta, \rho = \mathcal O(\epsilon^2)$ and $T= \Theta(\epsilon^{-4})$, we can find an $\epsilon$-stationary point with the optimal sample complexity in the order of $\mathcal O(\epsilon^{-4})$ for SGD with a single $L$-smooth objective \citep{arjevani2023lower}. Note that in the proof of \Cref{thm:gddeterministic}, we show for each $t\le T$ and $w\in W$, we have that $ F(x_t)w$ is bounded by applying a small constant step size $\alpha$ and $\beta$. However, this condition does not necessarily hold for our stochastic setting due to the unbounded gradient noise. To solve this problem, we introduce stopping time $\tau$. The advantages are as follows: 1) for any $t\le\tau$, $w\in W$, we have that $F(x_t)w$ is bounded; 2) for any $t<\tau$, the norm of gradient noise is bounded; 3) due to the {optional} stopping theorem, for any $w\in W$ and $i \in[3]$, we have that $\mathbb E[\sum_{t=0}^\tau\varepsilon_{t,i}w]=0$. Thus, we can further get the following lemma:
\begin{lemma}\label{lemma:sgdlemmaresult}
    Using the parameters selected in \Cref{thm:gdstochastic}, we have that 
       $\mathbb E[F(x_{\tau})w]-F^*w\le \frac{\delta F}{8}- \frac{\alpha}{2}\mathbb E\left[\sum_{t=0}^{\tau-1}\|\nabla F(x_t)w_t\|^2\right].$
\end{lemma}
The proof of Lemma \ref{lemma:sgdlemmaresult} is available in \ref{proof:lemmasgd}. Lemma \ref{lemma:sgdlemmaresult} indicates that $\frac{\alpha}{2}\mathbb E\left[\sum_{t=0}^{\tau-1}\|\nabla F(x_t)w_t\|^2\right]$ is bounded by some constant and if $\tau=T$ with high probability, we have that $\frac{1}{T}\mathbb E\left[\sum_{t=0}^{T-1}\|\nabla F(x_t)w_t\|^2\Big|\tau=T\right]= \mathcal O\left(\frac{1}{\alpha T}\right)$. Note that 
$    \{\tau<T\}= \{\tau_2<T\}\cup\{\tau_3<T\}\cup  \{\tau_1<T, \tau_2=T,\tau_3=T\}. $
The first two events are related to the gradient noise, where the probabilities can be bounded by Assumption \ref{ass:boundedvariance} and Chebyshev’s inequality. The last event indicates that for some $i\in [K]$, we have  
% \begin{flalign}
    $f_i(x_{\tau})-f_i^*\le \frac{F}{2}.$
% \end{flalign}
Based on Lemma \ref{lemma:sgdlemmaresult} and Markov inequality, we can show that $\mathbb P(\{\tau_1<T, \tau_2=T,\tau_3=T\})\le \frac{\delta}{4}$ and we can further show that $\mathbb P(\tau=T)\ge1 -\frac{\delta}{2}$. We then have that $\frac{1}{T}\sum_{t=0}^{T-1}\|\nabla F(x_t)w_t\|^2$ converges with high probability. Similar to Corollary \ref{coro:gdaverageca}, \Cref{thm:gdstochastic} also implies the average of CA distances converges with time with high probability. 

%In this paper, we study the case where the $\ell$ function is sub-quadratic, which implies function $\varphi$ is monotonically increasing.
\section{Convergence Analysis under Iteration-wise CA distance}
In \Cref{sec:averageca} we show that the average CA distance is bounded under generalized smooth conditions. The average CA distance is also studied in MoCo \citep{fernando2022mitigating} and MoDo \citep{chen2024three} with the bounded gradient assumption and these works only focus on guarantees of the average CA distance over iterations. However, a $\epsilon$-level average CA distance only implies the smallest CA distance to be $\epsilon$-level. Since we want to keep the update direction close enough to the CA direction, it is better to have a tighter bound of CA distances. In this section, we show the iterative CA distance is $\mathcal{O}(\epsilon)$   {\bf with a warm-start process} and convergence results for Algorithms \ref{alg:1}, \ref{alg:2}, and \ref{alg:4}.
\vspace{-0.2cm}

\subsection{Deterministic setting}
\vspace{-0.1cm}
{\bf Deterministic setting without fast approximation.} 
We first provide results about bounded iteration-wise CA distance for \Cref{alg:1} with a warm start. 
% {\color{red}KY: remove "informal"; change 'formal version' to 'full version with detailed constants'}
\begin{theorem}\label{thm:consistsmallcadistance}Let Assumptions \ref{ass:diffandlowerbound} and \ref{ass:lsmooth} hold.  Set $\beta^\prime\leq\frac{1}{M^2},\rho= \mathcal O(\epsilon^2), \beta= \mathcal O(\epsilon^4), \alpha= \mathcal O(\epsilon^9)$, $N= \Omega(\epsilon^{-2})$ as constants, and  $T = \Theta (\epsilon^{-11})$. All the parameters satisfy the requirements in the formal version  of \Cref{thm:gddeterministic} and we have
% \begin{align}
    $\|\nabla F(x_t)w_t-\nabla F(x_t)w_t^*\|=\mathcal{O}(\epsilon).$
% \end{align}
% The CA distance in every iteration takes the order of $\mathcal{O}(\epsilon)$.
\end{theorem}
The finite time error bound and the full proof can be found in Appendix \ref{proof:gddeterministic}. Since our parameters satisfy the requirements in the formal version of \Cref{thm:gddeterministic}, we can find an $\epsilon$-accurate Pareto stationary point with $\mathcal{O}(\epsilon^{-11})$ samples. In the analysis of CA distance, we show that the CA distance can be bounded by the term $\|w_t-w_{t,\rho}^*\|$ plus the strongly-convex constant $\rho$. Meanwhile, there is a decay relation between $\|w_{t+1}-w_{t+1,\rho}^*\|$ and $\|w_t-w_{t,\rho}^*\|$ with some error terms controlled by step sizes. Nevertheless, the error terms will accumulate since we do telescoping on this decay relation, which will be the dominating term. Thus, step sizes have to be much smaller than the choices in \Cref{thm:gddeterministic} to guarantee iteration-wise small CA distance.

% We then provide the convergence analysis. Let $c_1, c_2^\prime, c_3^\prime>0$ and $F$ be some constants such that
% \begin{align}
%     \Delta+c_1+c_2^\prime+c_3^\prime\leq F
% \end{align}
% \begin{theorem}\label{thm:gsmgrad} Suppose Assumptions \ref{ass:diffandlowerbound}-\ref{ass:lsmooth} are satisfied, and follow the same step-size selection that $\rho= \mathcal O(\epsilon^2), \beta= \mathcal O(\epsilon^4), \alpha= \mathcal O(\epsilon^9)$, and $T=\mathcal{O}(\epsilon^{-11})$. We have that
% \begin{align}
% \frac{1}{T}\sum_{t=0}^{T-1}\|\nabla F(x_t)w_t\|^2\leq\epsilon^2.
% \end{align}
% \end{theorem}
% The formal version and proof can be found in Appendix \ref{proof:gsmgrad}. Since our parameters satisfy all requirements in Theorem \ref{thm:consistsmallcadistance}, we can ensure CA distance keeps $\epsilon$-level in every iteration.

%\vspace{0.2cm}
\noindent{\bf Deterministic setting with fast approximation.}
 In this section, we show the convergence rate of \Cref{alg:gs-famo} and bounded iteration-wise CA distance.
\begin{theorem}\label{thm:gsmgrad-famoinformal}
Let Assumptions \ref{ass:diffandlowerbound} and \ref{ass:lsmooth} hold.  Set $N= \Omega(\epsilon^{-2})$, $\beta^\prime\leq\frac{1}{M^2},\rho= \mathcal O\big(\min\big\{ \epsilon^2, \frac{1}{\alpha T}\big\}\big), \beta= \mathcal O(\epsilon^2), \alpha= \mathcal O\big(\min\{ \beta, \epsilon^2, \frac{1}{\beta T}\})$ as constants, $T= \Theta\big(\max\{ \frac{1}{\alpha \epsilon^2}, \frac{1}{\beta \epsilon^2}\}\big)$. We have
$   \frac{1}{T}\sum_{t=0}^{T-1}\|\nabla F(x_t)w_t\|^2=\mathcal{O}(\epsilon^2).$
\end{theorem}
The full version and proof can be found in \Cref{proof:gsmgrad-famoformal}. We can easily extend the analysis in \Cref{proof:gddeterministic-avg} to the convergence analysis of \Cref{alg:gs-famo} under the average CA distance, because the only extra effort is dealing with the remainder term, which can be bounded by the smallest step size. As a result, the sample complexity remains the same $\mathcal{O}(\epsilon^{-11})$ to achieve a Pareto stationary point.
\begin{theorem}\label{thm:consistsmallcadistance-famo} Let Assumptions \ref{ass:diffandlowerbound} and \ref{ass:lsmooth} hold.  We choose $\beta^\prime\leq\frac{1}{M^2},\rho= \mathcal O(\epsilon^2), \beta= \mathcal O(\epsilon^4), \alpha= \mathcal O(\epsilon^9)$, $N= \Omega(\epsilon^{-2})$ as constants, and  $T = \Theta (\epsilon^{-11})$. We have that
% \begin{align}
    $\|\nabla F(x_t)w_t-\nabla F(x_t)w_t^*\|=\mathcal{O}(\epsilon).$
% \end{align}
% The CA distance in every iteration takes the order of $\mathcal{O}(\epsilon)$.
\end{theorem}
% The proof can be found in \Cref{proof:gddeterministic-famo}.
\vspace{-0.2cm}
\subsection{Stochastic setting}
\vspace{-0.1cm}
In this section, we show that  Algorithm \ref{alg:2} with a warm start and mini-batches achieves a bounded iteration-wise CA distance with high probability. In this section, we choose 
$
    G_i(x_t)=\frac{1}{n_s}\sum_{i=n_si-n_s+1}^{n_si}F(x_t;s_{t,i}),
$
where $n_s$ is the size of the mini-batch. 
\begin{theorem}\label{thm:sgdeachca}
Let Assumptions \ref{ass:diffandlowerbound}, \ref{ass:lsmooth} and \ref{ass:boundedvariance} hold. Set $\beta^\prime\leq\frac{1}{M^2}$, $\alpha = \mathcal O(\epsilon^9)$,  $\beta = \mathcal O(\epsilon^4)$, $\rho = \mathcal O(\epsilon^2)$, 
 $n_s = \Omega(\epsilon^{-6}), N=\Omega(\epsilon^{-2}),$ and  $T = \Theta(\epsilon^{-11})$, and all the parameters satisfy the requirements in \Cref{thm:gdstochastic}.  We then have 
% \begin{flalign}
    $\|\nabla F(x_t)w_t-\nabla F(x_t)w_t^*\|= \mathcal O(\epsilon),$
% \end{flalign}
with the probability at least $1-\delta$.
\end{theorem}
The full version with detailed constants and proof can be found in Appendix \ref{proof:sgdeachca}. Since our parameters satisfy all requirements in Theorem \ref{thm:gdstochastic}, we can find an $\epsilon$-accurate Pareto stationary point with high probability.
Compared with \Cref{thm:gdstochastic}, to guarantee an iteration-wise CA distance, despite our warm start process, a mini-batch method is required in our analysis. This is because given $\tau=T$, the gradient is not unbiased. In \Cref{thm:gdstochastic}, the {optional} stopping theorem is applied which indicates that the expectation of the cumulative gradient is zero. However, for each iteration, this {optional} stopping theorem does not hold and the estimated error is controlled by the size of the mini-batch. Then, the sample complexity to get a Pareto stationary point becomes $\mathcal{O}(\epsilon^{-17})$ due to necessary mini-batch $n_s$.

\vspace{-0.3cm}
\section{Experiments}
\vspace{-0.2cm}

In this experiment, we evaluate the performance of the Cityscapes\citep{cordts2016cityscapes} and NYU-v2 \citep{silberman2012indoor} datasets. The former involves 2 pixel-wise tasks: 7-class semantic segmentation (Task 1) and depth estimation (Task 2) while the latter involves 3 pixel-wise tasks: 13-class
semantic segmentation, depth estimation and surface normal estimation. Following the same experiment setup of \citep{xiao2024direction}, we build a SegNet \citep{badrinarayanan2017segnet} as the model. We compare the performance of MGDA {\bf with a warm start}, which is used to ensure a small CA distance per iteration and mitigate gradient conflicts, against 
{\bf popular MGDA-type methods} including MGDA \citep{desideri2012multiple}, PCGrad \citep{yu2020gradient}, GradDrop \citep{chen2020just}, CAGrad \citep{liu2021conflict}, MoCo \citep{fernando2022mitigating}, MoDo \citep{chen2024three}, Nash-MTL \citep{navon2022multi}, and FAMO \citep{liu2024famo}. {\em Since SDMGrad \citep{xiao2024direction} incorporates a preference-based regularization within the MGDA framework, we exclude it from the tables to ensure a fair comparison.}  We utilize the metric $\mathbf{\Delta m\%}$ to reflect the overall performance, which considers the average per-task performance drop versus the single-task (STL) baseline to assess methods. It can be observed in \Cref{tab:nyuv2} and \Cref{tab:cityscapes} that MGDA-warm start has a much more balanced performance. Meanwhile, the proposed MGDA-FA is much faster as shown in \Cref{tab:mgdaandmgda-fa} in the Appendix.

We also illustrate the relationship between the gradient norm and the local smoothness for each task of the Cityscapes dataset. To do this, we compute them according to the method provided in Section H.3 in \cite{zhang2019gradient}. We scatter the local smoothness constant against gradient norms in \Cref{fig:firsttaskgs} for the semantic segmentation task and depth estimation task in \Cref{fig:task2} (in the appendix), respectively. Both results demonstrate a positive correlation between them, which further substantiates the necessity of our analysis. More experimental details can be found in \Cref{app:exp}.

\begin{table}[t]
  \vspace{-0.5cm}
  \centering
  \begin{adjustbox}{width=\textwidth}
  \begin{tabular}{lccccccccccc}
    \toprule
    \multirow{3}*{Method} & \multicolumn{2}{c}{Segmentation} & \multicolumn{2}{c}{Depth} & \multicolumn{5}{c}{Surface Normal} & \multirow{3}*{MR $\downarrow$} & \multirow{3}*{$\Delta m\%\downarrow$} \\
    \cmidrule(lr){2-3}\cmidrule(lr){4-5}\cmidrule(lr){6-10}
    & \multirow{2}*{mIoU $\uparrow$} & \multirow{2}*{Pix Acc $\uparrow$} & \multirow{2}*{Abs Err $\downarrow$} & \multirow{2}*{Rel Err $\downarrow$} & \multicolumn{2}{c}{Angle Distance $\downarrow$} & \multicolumn{3}{c}{Within $t^\circ$ $\uparrow$} & \\
    \cmidrule(lr){6-7}\cmidrule(lr){8-10}
    & & & & & Mean & Median & 11.25 & 22.5 & 30 & \\
    \midrule
    STL & 38.30 & 63.76 & 0.6754 & 0.2780 & 25.01 & 19.21 & 30.14 & 57.20 & 69.15 & \\
    \midrule
    LS & 39.29 & 65.33 & 0.5493 & 0.2263 & 28.15 & 23.96 & 22.09 & 47.50 & 61.08 & 7.89 & 5.59  \\
    SI & 38.45 & 64.27 & 0.5354 & 0.2201 & 27.60 & 23.37 & 22.53 & 48.57 & 62.32 & 7.33 & 4.39 \\
    RLW \citep{lin2021reasonable} & 37.17 & 63.77 & 0.5759 & 0.2410 & 28.27 & 24.18 & 22.26 & 47.05 & 60.62 & 9.89 & 7.78 \\
    DWA \citep{liu2019end}& 39.11 & 65.31 & 0.5510 & 0.2285 & 27.61 & 23.18 & 24.17 & 50.18 & 62.39 & 7.11 & 3.57 \\
    UW \citep{kendall2018multi}& 36.87 & 63.17 & 0.5446 & 0.2260 & 27.04 & 22.61 & 23.54 & 49.05 & 63.65 &  7.11 & 4.05 \\
    MGDA \citep{desideri2012multiple} & 30.47 & 59.90 & 0.6070 & 0.2555 & \textbf{24.88} & \textbf{19.45} & 29.18 & \textbf{56.88} & \textbf{69.36} & 5.56 & 1.38 \\
    MoCo \citep{fernando2022mitigating}& 40.30 & 66.07 & 0.5575 & \textbf{0.2135} & 26.67 & 21.83 & 25.61 & 51.78 & 64.85 & 5.00 & 0.16 \\
    MoDo \citep{chen2024three}& 35.28 & 62.62 & 0.5821 & 0.2405 & 25.65 & 20.33 & 28.04 & 54.86 & 67.37 & 7.55 & 0.49 \\
    Nash-MTL \cite{navon2022multi}& 40.13 & 65.93 & 0.5261 & 0.2171 & 25.26 & 20.08 & 28.40 & 55.47 & 68.15 & 3.33 & -4.04 \\
    FAMO \citep{liu2024famo} & 38.88 & 64.90 & 0.5474 & 0.2194 & 25.06 & 19.57 & \textbf{29.21} & 56.61 & 68.98 & 3.22 & -4.10 \\
    % SDMGrad & 40.47 & 65.90 & \textbf{0.5225} & \textbf{0.2084} & 25.07 & 19.99 & 28.54 & 55.74 & 68.53 & 2.89 & \textbf{-4.84} \\
    \midrule
    MGDA-warm start & \textbf{40.57} & \textbf{67.17} & \textbf{0.5240} & 0.2281 & 25.21 & 19.74 & 28.74 & 55.79 & 68.21 & \textbf{2.78} & \textbf{-4.42} \\
    \bottomrule
  \end{tabular}
  \end{adjustbox}
  \vspace{-0.1cm}
  \caption{Multi-task supervised learning on NYU-v2 dataset for different MOO methods comparison.}\label{tab:nyuv2}
  \vspace{-0.6cm}
\end{table}

\begin{figure}[ht]
\centering
    \vspace{-6pt} % align at the top
    \centering
    \begin{adjustbox}{max width=0.7\textwidth}
        \begin{tabular}{llllll}
            \toprule
            \multirow{2}*{Method} & \multicolumn{2}{c}{Segmentation} & \multicolumn{2}{c}{Depth} &
            \multirow{2}*{$\Delta m\%\downarrow$} \\
            \cmidrule(lr){2-3}\cmidrule(lr){4-5}
            & mIoU $\uparrow$ & Pix Acc $\uparrow$ & Abs Err $\downarrow$ & Rel Err $\downarrow$ & \\
            \midrule
            STL & 74.01 & 93.16 & 0.0125 & 27.77\\
            \midrule
            MGDA~\citep{desideri2012multiple} & 68.84 & 91.54 & 0.0309 & 33.50 &  44.14  \\
            PCGrad~\citep{yu2020gradient} & 75.13 & 93.48 & 0.0154 & 42.07 &  18.29  \\
            GradDrop~\citep{chen2020just} & 75.27 & 93.53 & 0.0157 & 47.54 &  23.73  \\
            CAGrad~\citep{liu2021conflict} & 75.16 & 93.48 & 0.0141 & 37.60 & 11.64  \\
            MoCo~\citep{fernando2022mitigating} & \textbf{75.42} & 93.55 & 0.0149 & 34.19 & 9.90 \\
            MoDo~\citep{chen2024three} & 74.55 & 93.32 & 0.0159 & 41.51 & 18.89 \\
            Nash-MTL~\citep{navon2022multi} & 75.41 & \textbf{93.66} & \textbf{0.0129} & 35.02 & 6.82 \\
            % SDMGrad \citep{xiao2024direction} & 74.53 & 93.52 & 0.0137 & 34.01 & 7.79 \\
            FAMO~\citep{liu2024famo} & 74.54 & 93.29 & 0.0145 & 32.59 & 8.13 \\
            \midrule
            MGDA-warm start & 75.41 & 93.46 & 0.0133 & \textbf{31.07} &  {\textbf{3.93$\pm$1.19}}\\
            % MGDA-FA & 74.38 & 93.24 & 0.0160 & 41.78 & 19.44 \\
            \bottomrule
        \end{tabular}
    \end{adjustbox}
    \captionsetup{type=table}
    \vspace{-0.1cm}
    \caption{Multi-task learning on Cityscapes dataset.}
    \label{tab:cityscapes}

% \hspace{2pt}
\vspace{-0.6cm}
\end{figure}

\section{Conclusion}
\vspace{-0.2cm}

% In this paper, we investigate the multi-objective problem with a more challenging, relaxed and realistic $\ell$-smooth assumption. We propose the first efficient MGDA with warm start and its stochastic variant for this problem. We provide the convergence guarantee for both algorithms to find an $\epsilon$-accurate Pareto stationary point with $\epsilon$-level average/iteration-wise CA distance. Extensive experiments are conducted to validate our theoretical results. 
% \newpage
Building upon our observations of MGDA convergence under the $(L_0, L_1)$ smoothness condition, this paper provides a rigorous convergence analysis for both the fundamental MGDA and its stochastic variant under a more challenging, relaxed, and practical generalized 
$\ell$-smoothness assumption.  Furthermore, we introduce a warm start progress to provide a more precise control over the iteration-wise CA distance. Our analysis also shows that an efficient variant named MGDA-FA, which uses only $\mathcal O(1)$ time and space, achieves the same performance guarantee as MGDA. We anticipate that the convergence analysis developed in this work will provide valuable insights for analyzing other MOO algorithms such as CAGrad, PCGrad, FairGrad and FAMO under the generalized smoothness condition. The warm-start strategy may be of independent interest to other single-loop MOO algorithms to achieve a sufficiently small iteration-wise CA distance.  

\section*{Acknowledgments}
The work of Q. Zhang and S. Zou was partially supported by NSF under Grant CCF-2438429. P. Xiao and K. Ji were partially supported by NSF grants CCF-2311274 and ECCS-2326592.

% different assumptions. 
% Lastly, the proposed algorithms are expected to be applicable to various applications, including reinforcement learning and autonomous driving.

\bibliography{ICLR_2025/iclr2025_conference}

\begin{thebibliography}{40}
\providecommand{\natexlab}[1]{#1}
\providecommand{\url}[1]{\texttt{#1}}
\expandafter\ifx\csname urlstyle\endcsname\relax
  \providecommand{\doi}[1]{doi: #1}\else
  \providecommand{\doi}{doi: \begingroup \urlstyle{rm}\Url}\fi

\bibitem[Arjevani et~al.(2023)Arjevani, Carmon, Duchi, Foster, Srebro, and Woodworth]{arjevani2023lower}
Yossi Arjevani, Yair Carmon, John~C Duchi, Dylan~J Foster, Nathan Srebro, and Blake Woodworth.
\newblock Lower bounds for non-convex stochastic optimization.
\newblock \emph{Mathematical Programming}, 199\penalty0 (1):\penalty0 165--214, 2023.

\bibitem[Badrinarayanan et~al.(2017)Badrinarayanan, Kendall, and Cipolla]{badrinarayanan2017segnet}
Vijay Badrinarayanan, Alex Kendall, and Roberto Cipolla.
\newblock Segnet: A deep convolutional encoder-decoder architecture for image segmentation.
\newblock \emph{IEEE transactions on pattern analysis and machine intelligence}, 39\penalty0 (12):\penalty0 2481--2495, 2017.

\bibitem[Ban \& Ji(2024)Ban and Ji]{ban2024fair}
Hao Ban and Kaiyi Ji.
\newblock Fair resource allocation in multi-task learning.
\newblock \emph{arXiv preprint arXiv:2402.15638}, 2024.

\bibitem[Beck \& Teboulle(2009)Beck and Teboulle]{beck2009gradient}
Amir Beck and Marc Teboulle.
\newblock Gradient-based algorithms with applications to signal recovery.
\newblock \emph{Convex optimization in signal processing and communications}, pp.\  42--88, 2009.

\bibitem[Carmon et~al.(2020)Carmon, Duchi, Hinder, and Sidford]{carmon2020lower}
Yair Carmon, John~C Duchi, Oliver Hinder, and Aaron Sidford.
\newblock Lower bounds for finding stationary points i.
\newblock \emph{Mathematical Programming}, 184\penalty0 (1):\penalty0 71--120, 2020.

\bibitem[Chen et~al.(2024)Chen, Fernando, Ying, and Chen]{chen2024three}
Lisha Chen, Heshan Fernando, Yiming Ying, and Tianyi Chen.
\newblock Three-way trade-off in multi-objective learning: Optimization, generalization and conflict-avoidance.
\newblock \emph{Advances in Neural Information Processing Systems}, 36, 2024.

\bibitem[Chen et~al.(2018)Chen, Badrinarayanan, Lee, and Rabinovich]{chen2018gradnorm}
Zhao Chen, Vijay Badrinarayanan, Chen-Yu Lee, and Andrew Rabinovich.
\newblock Gradnorm: Gradient normalization for adaptive loss balancing in deep multitask networks.
\newblock In \emph{International conference on machine learning}, pp.\  794--803. PMLR, 2018.

\bibitem[Chen et~al.(2020)Chen, Ngiam, Huang, Luong, Kretzschmar, Chai, and Anguelov]{chen2020just}
Zhao Chen, Jiquan Ngiam, Yanping Huang, Thang Luong, Henrik Kretzschmar, Yuning Chai, and Dragomir Anguelov.
\newblock Just pick a sign: Optimizing deep multitask models with gradient sign dropout.
\newblock \emph{Advances in Neural Information Processing Systems}, 33:\penalty0 2039--2050, 2020.

\bibitem[Chen et~al.(2023)Chen, Zhou, Liang, and Lu]{chen2023generalized}
Ziyi Chen, Yi~Zhou, Yingbin Liang, and Zhaosong Lu.
\newblock Generalized-smooth nonconvex optimization is as efficient as smooth nonconvex optimization.
\newblock \emph{arXiv preprint arXiv:2303.02854}, 2023.

\bibitem[Cordts et~al.(2016)Cordts, Omran, Ramos, Rehfeld, Enzweiler, Benenson, Franke, Roth, and Schiele]{cordts2016cityscapes}
Marius Cordts, Mohamed Omran, Sebastian Ramos, Timo Rehfeld, Markus Enzweiler, Rodrigo Benenson, Uwe Franke, Stefan Roth, and Bernt Schiele.
\newblock The cityscapes dataset for semantic urban scene understanding.
\newblock In \emph{Proceedings of the IEEE conference on computer vision and pattern recognition}, pp.\  3213--3223, 2016.

\bibitem[Crawshaw et~al.(2022)Crawshaw, Liu, Orabona, Zhang, and Zhuang]{crawshaw2022robustness}
Michael Crawshaw, Mingrui Liu, Francesco Orabona, Wei Zhang, and Zhenxun Zhuang.
\newblock Robustness to unbounded smoothness of generalized signsgd.
\newblock \emph{Advances in Neural Information Processing Systems}, 35:\penalty0 9955--9968, 2022.

\bibitem[D{\'e}sid{\'e}ri(2012)]{desideri2012multiple}
Jean-Antoine D{\'e}sid{\'e}ri.
\newblock Multiple-gradient descent algorithm (mgda) for multiobjective optimization.
\newblock \emph{Comptes Rendus Mathematique}, 350\penalty0 (5-6):\penalty0 313--318, 2012.

\bibitem[Fernando et~al.(2022)Fernando, Shen, Liu, Chaudhury, Murugesan, and Chen]{fernando2022mitigating}
Heshan~Devaka Fernando, Han Shen, Miao Liu, Subhajit Chaudhury, Keerthiram Murugesan, and Tianyi Chen.
\newblock Mitigating gradient bias in multi-objective learning: A provably convergent approach.
\newblock In \emph{The Eleventh International Conference on Learning Representations}, 2022.

\bibitem[Garrigos \& Gower(2023)Garrigos and Gower]{garrigos2023handbook}
Guillaume Garrigos and Robert~M Gower.
\newblock Handbook of convergence theorems for (stochastic) gradient methods.
\newblock \emph{arXiv preprint arXiv:2301.11235}, 2023.

\bibitem[Ghadimi \& Lan(2013)Ghadimi and Lan]{ghadimi2013stochastic}
Saeed Ghadimi and Guanghui Lan.
\newblock Stochastic first-and zeroth-order methods for nonconvex stochastic programming.
\newblock \emph{SIAM Journal on Optimization}, 23\penalty0 (4):\penalty0 2341--2368, 2013.

\bibitem[Ghadimi et~al.(2016)Ghadimi, Lan, and Zhang]{ghadimi2016mini}
Saeed Ghadimi, Guanghui Lan, and Hongchao Zhang.
\newblock Mini-batch stochastic approximation methods for nonconvex stochastic composite optimization.
\newblock \emph{Mathematical Programming}, 155\penalty0 (1-2):\penalty0 267--305, 2016.

\bibitem[Guo et~al.(2018)Guo, Haque, Huang, Yeung, and Fei-Fei]{guo2018dynamic}
Michelle Guo, Albert Haque, De-An Huang, Serena Yeung, and Li~Fei-Fei.
\newblock Dynamic task prioritization for multitask learning.
\newblock In \emph{Proceedings of the European conference on computer vision (ECCV)}, pp.\  270--287, 2018.

\bibitem[Huang et~al.(2019)Huang, Wang, Cheng, Zhou, Geng, and Yang]{huang2019apolloscape}
Xinyu Huang, Peng Wang, Xinjing Cheng, Dingfu Zhou, Qichuan Geng, and Ruigang Yang.
\newblock The apolloscape open dataset for autonomous driving and its application.
\newblock \emph{IEEE transactions on pattern analysis and machine intelligence}, 42\penalty0 (10):\penalty0 2702--2719, 2019.

\bibitem[Jin et~al.(2021)Jin, Zhang, Wang, and Wang]{jin2021non}
Jikai Jin, Bohang Zhang, Haiyang Wang, and Liwei Wang.
\newblock Non-convex distributionally robust optimization: Non-asymptotic analysis.
\newblock \emph{Advances in Neural Information Processing Systems}, 34:\penalty0 2771--2782, 2021.

\bibitem[Kendall et~al.(2018)Kendall, Gal, and Cipolla]{kendall2018multi}
Alex Kendall, Yarin Gal, and Roberto Cipolla.
\newblock Multi-task learning using uncertainty to weigh losses for scene geometry and semantics.
\newblock In \emph{Proceedings of the IEEE conference on computer vision and pattern recognition}, pp.\  7482--7491, 2018.

\bibitem[Li et~al.(2024{\natexlab{a}})Li, Qian, Tian, Rakhlin, and Jadbabaie]{li2024convex}
Haochuan Li, Jian Qian, Yi~Tian, Alexander Rakhlin, and Ali Jadbabaie.
\newblock Convex and non-convex optimization under generalized smoothness.
\newblock \emph{Advances in Neural Information Processing Systems}, 36, 2024{\natexlab{a}}.

\bibitem[Li et~al.(2024{\natexlab{b}})Li, Rakhlin, and Jadbabaie]{li2024convergence2}
Haochuan Li, Alexander Rakhlin, and Ali Jadbabaie.
\newblock Convergence of adam under relaxed assumptions.
\newblock \emph{Advances in Neural Information Processing Systems}, 36, 2024{\natexlab{b}}.

\bibitem[Lin et~al.(2021)Lin, Ye, Zhang, and Tsang]{lin2021reasonable}
Baijiong Lin, Feiyang Ye, Yu~Zhang, and Ivor~W Tsang.
\newblock Reasonable effectiveness of random weighting: A litmus test for multi-task learning.
\newblock \emph{arXiv preprint arXiv:2111.10603}, 2021.

\bibitem[Liu et~al.(2021)Liu, Liu, Jin, Stone, and Liu]{liu2021conflict}
Bo~Liu, Xingchao Liu, Xiaojie Jin, Peter Stone, and Qiang Liu.
\newblock Conflict-averse gradient descent for multi-task learning.
\newblock \emph{Advances in Neural Information Processing Systems}, 34:\penalty0 18878--18890, 2021.

\bibitem[Liu et~al.(2024)Liu, Feng, Stone, and Liu]{liu2024famo}
Bo~Liu, Yihao Feng, Peter Stone, and Qiang Liu.
\newblock Famo: Fast adaptive multitask optimization.
\newblock \emph{Advances in Neural Information Processing Systems}, 36, 2024.

\bibitem[Liu et~al.(2019)Liu, Johns, and Davison]{liu2019end}
Shikun Liu, Edward Johns, and Andrew~J Davison.
\newblock End-to-end multi-task learning with attention.
\newblock In \emph{Proceedings of the IEEE/CVF conference on computer vision and pattern recognition}, pp.\  1871--1880, 2019.

\bibitem[Liu \& Vicente(2021)Liu and Vicente]{liu2021stochastic}
Suyun Liu and Luis~Nunes Vicente.
\newblock The stochastic multi-gradient algorithm for multi-objective optimization and its application to supervised machine learning.
\newblock \emph{Annals of Operations Research}, pp.\  1--30, 2021.

\bibitem[Ma et~al.(2018)Ma, Zhao, Yi, Chen, Hong, and Chi]{ma2018modeling}
Jiaqi Ma, Zhe Zhao, Xinyang Yi, Jilin Chen, Lichan Hong, and Ed~H Chi.
\newblock Modeling task relationships in multi-task learning with multi-gate mixture-of-experts.
\newblock In \emph{Proceedings of the 24th ACM SIGKDD international conference on knowledge discovery \& data mining}, pp.\  1930--1939, 2018.

\bibitem[Navon et~al.(2022)Navon, Shamsian, Achituve, Maron, Kawaguchi, Chechik, and Fetaya]{navon2022multi}
Aviv Navon, Aviv Shamsian, Idan Achituve, Haggai Maron, Kenji Kawaguchi, Gal Chechik, and Ethan Fetaya.
\newblock Multi-task learning as a bargaining game.
\newblock \emph{arXiv preprint arXiv:2202.01017}, 2022.

\bibitem[Nemirovskij \& Yudin(1983)Nemirovskij and Yudin]{nemirovskij1983problem}
Arkadij~Semenovi{\v{c}} Nemirovskij and David~Borisovich Yudin.
\newblock Problem complexity and method efficiency in optimization.
\newblock 1983.

\bibitem[Reisizadeh et~al.(2023)Reisizadeh, Li, Das, and Jadbabaie]{reisizadeh2023variance}
Amirhossein Reisizadeh, Haochuan Li, Subhro Das, and Ali Jadbabaie.
\newblock Variance-reduced clipping for non-convex optimization.
\newblock \emph{arXiv preprint arXiv:2303.00883}, 2023.

\bibitem[Sener \& Koltun(2018)Sener and Koltun]{sener2018multi}
Ozan Sener and Vladlen Koltun.
\newblock Multi-task learning as multi-objective optimization.
\newblock \emph{Advances in neural information processing systems}, 31, 2018.

\bibitem[Silberman et~al.(2012)Silberman, Hoiem, Kohli, and Fergus]{silberman2012indoor}
Nathan Silberman, Derek Hoiem, Pushmeet Kohli, and Rob Fergus.
\newblock Indoor segmentation and support inference from rgbd images.
\newblock In \emph{Computer Vision--ECCV 2012: 12th European Conference on Computer Vision, Florence, Italy, October 7-13, 2012, Proceedings, Part V 12}, pp.\  746--760. Springer, 2012.

\bibitem[Thomas et~al.(2021)Thomas, Pineau, Laroche, et~al.]{thomas2021multi}
Philip~S Thomas, Joelle Pineau, Romain Laroche, et~al.
\newblock Multi-objective spibb: Seldonian offline policy improvement with safety constraints in finite mdps.
\newblock \emph{Advances in Neural Information Processing Systems}, 34:\penalty0 2004--2017, 2021.

\bibitem[Xiao et~al.(2024)Xiao, Ban, and Ji]{xiao2024direction}
Peiyao Xiao, Hao Ban, and Kaiyi Ji.
\newblock Direction-oriented multi-objective learning: Simple and provable stochastic algorithms.
\newblock \emph{Advances in Neural Information Processing Systems}, 36, 2024.

\bibitem[Yang et~al.(2024)Yang, Liu, Liu, Dong, and Momma]{yang2024federated}
Haibo Yang, Zhuqing Liu, Jia Liu, Chaosheng Dong, and Michinari Momma.
\newblock Federated multi-objective learning.
\newblock \emph{Advances in Neural Information Processing Systems}, 36, 2024.

\bibitem[Yu et~al.(2020)Yu, Kumar, Gupta, Levine, Hausman, and Finn]{yu2020gradient}
Tianhe Yu, Saurabh Kumar, Abhishek Gupta, Sergey Levine, Karol Hausman, and Chelsea Finn.
\newblock Gradient surgery for multi-task learning.
\newblock \emph{Advances in Neural Information Processing Systems}, 33:\penalty0 5824--5836, 2020.

\bibitem[Zhang et~al.(2019)Zhang, He, Sra, and Jadbabaie]{zhang2019gradient}
Jingzhao Zhang, Tianxing He, Suvrit Sra, and Ali Jadbabaie.
\newblock Why gradient clipping accelerates training: A theoretical justification for adaptivity.
\newblock \emph{arXiv preprint arXiv:1905.11881}, 2019.

\bibitem[Zhang et~al.(2024)Zhang, Zhou, and Zou]{zhang2024convergence}
Qi~Zhang, Yi~Zhou, and Shaofeng Zou.
\newblock Convergence guarantees for rmsprop and adam in generalized-smooth non-convex optimization with affine noise variance.
\newblock \emph{arXiv preprint arXiv:2404.01436}, 2024.

\bibitem[Zhou et~al.(2022)Zhou, Zhang, Jiang, Zhong, Gu, and Zhu]{zhou2022convergence}
Shiji Zhou, Wenpeng Zhang, Jiyan Jiang, Wenliang Zhong, Jinjie Gu, and Wenwu Zhu.
\newblock On the convergence of stochastic multi-objective gradient manipulation and beyond.
\newblock \emph{Advances in Neural Information Processing Systems}, 35:\penalty0 38103--38115, 2022.

\end{thebibliography}
\bibliographystyle{iclr2025_conference}

%%%%%%%%%%%%%%%%%%%%%%%%%%%%%%%%%%%%%%%%%%%%%%%%%%%%%%%%%%%%

\newpage
\appendix
\section{Notation summary}
We have summarized the notations used in this paper in the following table.
\begin{table}[ht]
\caption{{Notations and their descriptions.}}
  \label{tab:notations}
  \small
  \centering
  \begin{tabular}{l|l l }
  \toprule
  Notations & Descriptions   \\
  % & Symbol  & Domain \\
  \midrule
  $x \in \mathbb{R}^{m}$ &Model parameter, or decision variable &   \\
  $s_{t,i,k}$ & i-th collection of samples at time t in the stochastic model for training or testing &   \\
  % $S\in \mathcal{Z}^n$ &Dataset such that $S = \{z_1,\dots, z_n\}$&   \\
  % \makecell[l]{$f_{z,m}(x)$, $f_{S,m}(x)$ \\
  % } &
  % \makecell[l]{A scalar-valued objective function evaluated on data point $z$,\\ with $f_{z,m}: \mathbb{R}^d \mapsto \mathbb{R}$, 
  % or on dataset $S$, $f_{S,m}$, 
  % with $f_{S,m} := \frac{1}{|S|} \sum_{z\in S} f_{z,m}(x) $} \\
  $f_{k}(x)$ &A scalar-valued population objective function \\
  $\nabla f_{k}(x)$ &Gradient of $f_{k}(x)$, with 
  $\nabla f_{k}(x): \mathbb{R}^m \mapsto \mathbb{R}^m$ \\
  % \makecell[l]{$F_{z}(x)$, $F_{S}(x)$ \\
  % } &
  % \makecell[l]{A vector-valued objective function evaluated on data point $z$,\\ with $F_{z}: \mathbb{R}^d \mapsto \mathbb{R}^M$,
  % or on dataset $S$, with
  % $F_{S} := \frac{1}{|S|} \sum_{z\in S} F_{z}(x) $ 
  % } \\
  $F(x)$ & A vector-valued population objective function \\
  $\nabla F(x)$ &Gradient of $F(x)$, with 
  $\nabla F(x): \mathbb{R}^d \mapsto \mathbb{R}^{d \times K}$ \\
  $\ell(a)$ &  A continuous non-decreasing function: $[0, +\infty) \rightarrow (0, +\infty)$ \\
   $\varphi(a)$ &  A monotonically increasing for any $a\ge 0$: $\varphi(a)=\frac{a^2}{2\ell(2a)}$ \\
  $w \in \mathcal{W} $ &Weighting parameter in a probability simplex over $[K]$,  &\\
  $w^*_\rho \in \mathcal{W} $ & optimal solution to~\eqref{eq:weightrho}, when $\rho=0$, it is  simplified as $w^*$  \\
  % $\mathbf{1} \in \mathbb{R}^{M} $ &All-one vector with dimension $M$  &\\
  \hline 
  $\alpha$ &Step size to update model parameter $x$\\
  $\beta$ & Step size to update weight in the main loop\\
  $\beta^\prime$ &Step size to update weight in the warm start\\
  $\rho$ &Regularization parameter in~\eqref{eq:weightrho} \\
  \bottomrule
  \end{tabular}
  \vspace{0.2cm}
\end{table}
\section{Experimental details}\label{app:exp}
\subsection{Relation between gradient norms and the local smoothness}
We show the relation between local smoothness and gradient norms of each task in this part. Both results demonstrate a positive correlation between them, which further substantiates the necessity of our analysis.
\begin{figure}[H]
\centering
\begin{minipage}[t]{0.49\textwidth}
    \centering
    \includegraphics[width=\textwidth]{scatter_plot_GS_task1.png}
    % \caption{}
    \label{fig:first-figure}
\end{minipage}
% \vspace{-0.5cm}
% \hfill
\begin{minipage}[t]{0.49\textwidth}
    \centering
    \includegraphics[width=\textwidth]{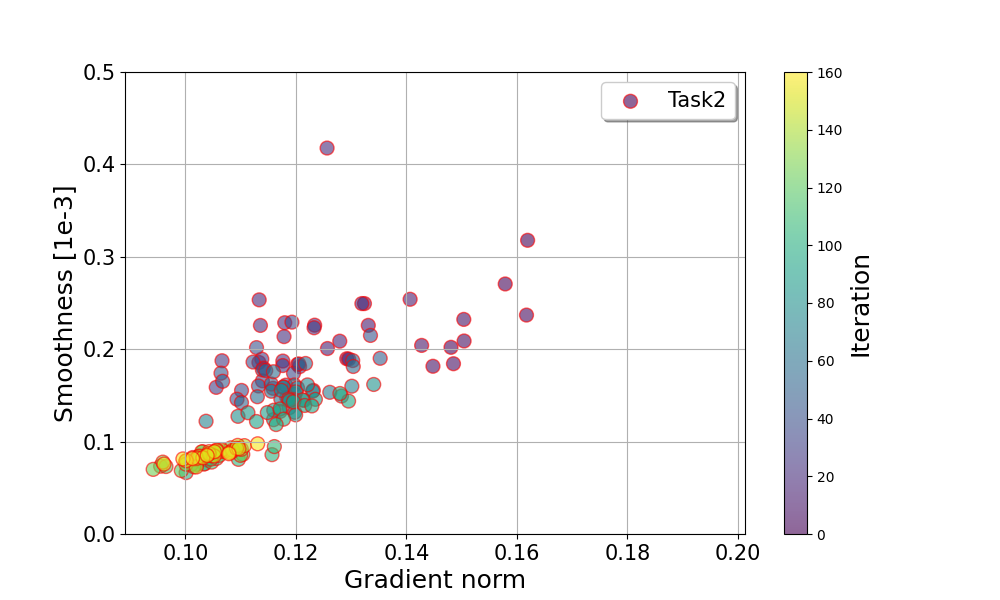}
    % \caption{Caption for the second figure}
\end{minipage}
\caption{Local smoothness constant vs. Gradient norm on training SegNet on CityScapes dataset of each task. Task 1 on the left and Task 2 on the right.} \label{fig:task2}
\end{figure}
\subsection{Results and running time comparison between MGDA-warm start and MGDA-FA}
We compare the results and average running time of the proposed algorithms, MGDA-warm start and MGDA-FA of the Cityscapes \citep{cordts2016cityscapes}. The time in \Cref{tab:mgdaandmgda-fa} is an average of the total running time over epochs (in minutes). The result solidifies the advantage of the fast approximation. {For its performance, the fast approximation may introduce substantial errors in practice, potentially leading to inaccuracies in the weight update process. Meanwhile, the convergence analysis of the stochastic variant is not revealed. However, to enhance its effectiveness, we could apply a logarithmic technique in FAMO\citep{liu2024famo}}.
\begin{figure}[ht]
\centering
    \vspace{0pt} % align at the top
    \centering
    \begin{adjustbox}{max width=\textwidth}
        \begin{tabular}{lcccccc}
            \toprule
            \multirow{2}*{Method} & \multicolumn{2}{c}{Segmentation} & \multicolumn{2}{c}{Depth} &
            \multirow{2}*{$\Delta m\%\downarrow$} &
            \multirow{2}*{Average running time} \\
            \cmidrule(lr){2-3}\cmidrule(lr){4-5}
            & mIoU $\uparrow$ & Pix Acc $\uparrow$ & Abs Err $\downarrow$ & Rel Err $\downarrow$ & \\
            \midrule
            STL & 74.01 & 93.16 & 0.0125 & 27.77\\
            \midrule
            MGDA-warm start & 75.41 & 93.46 & 0.0133 & 31.07 &  3.93 & 2.93\\
            MGDA-FA & 74.38 & 93.24 & 0.0160 & 41.78 & 19.44 & 1.93\\
            \bottomrule
        \end{tabular}
    \end{adjustbox}
    \captionsetup{type=table}
    \caption{Multi-task learning on Cityscapes dataset.}
    \label{tab:mgdaandmgda-fa}

% \hspace{2pt}
% \vspace{-0.3cm}
\end{figure}
% \begin{table}[ht]
%   \centering
%   \begin{tabular}{cc}
%     \toprule
%     Method & {Average running time} \\
%     % \cmidrule(lr){2-3}
%     \midrule
%     MGDA-warm start & 2.93 \\
%     MGDA-FA & 1.93 \\
%     \bottomrule
%   \end{tabular}
%   \vspace{0.3cm}
%   \caption{Average running time comparison between MGDA-warm start and MGDA-FA.}
%   \label{tab:time}
%   \vspace{-0.5cm}
% \end{table}
\subsection{Results of other MOO methods with warm start}
{The warm start strategy can be applied to other multi-objective optimization algorithms. For MGDA-based methods, both MoCo \citep{fernando2022mitigating} and MoDo \citep{chen2024three} can incorporate the warm start as an add-on. However, the MoDo-warm start will exactly recover to our algorithm. Therefore, we evaluate the performance of the MoCo-warm start on the Cityscapes \citep{cordts2016cityscapes} and NYU-v2 \citep{silberman2012indoor} datasets. As shown in the following tables, the warm start strategy improves performance.}
\begin{table}[ht]
  \centering
  % \small
    \begin{adjustbox}{max width=0.7\textwidth}
  \begin{tabular}{llllll}
    \toprule
    \multirow{2}*{Method} & \multicolumn{2}{c}{Segmentation} & \multicolumn{2}{c}{Depth} & 
    \multirow{2}*{$\Delta m\%\downarrow$} \\
    \cmidrule(lr){2-3}\cmidrule(lr){4-5}
    % & \multicolumn{2}{c}{(Higher Better)} & \multicolumn{2}{c}{(Lower Better)} &  \\
    & mIoU $\uparrow$ & Pix Acc $\uparrow$ & Abs Err $\downarrow$ & Rel Err $\downarrow$ & \\
    \midrule
    MoCo & 75.42 & 93.55 & 0.0149 & 34.19 & 9.90\\
    MoCo-warm start & 75.48 & 93.54 & 0.0148 & 31.43 & 7.32 \\
    \bottomrule
  \end{tabular}
  \end{adjustbox}
  % \vspace{0.1cm}
  \caption{Multi-task supervised learning with MoCo-warm start on Cityscapes dataset.}\label{tab:cityscapes-addon}
  % \vspace{-0.4cm}
\end{table}

\begin{table}[ht]
  \centering
  \begin{adjustbox}{max width=\textwidth}
  \begin{tabular}{lllllllllll}
    \toprule
    \multirow{3}*{Method} & \multicolumn{2}{c}{Segmentation} & \multicolumn{2}{c}{Depth} & \multicolumn{5}{c}{Surface Normal}  & \multirow{3}*{$\Delta m\%\downarrow$} \\
    \cmidrule(lr){2-3}\cmidrule(lr){4-5}\cmidrule(lr){6-10}
    & \multirow{2}*{mIoU $\uparrow$} & \multirow{2}*{Pix Acc $\uparrow$} & \multirow{2}*{Abs Err $\downarrow$} & \multirow{2}*{Rel Err $\downarrow$} & \multicolumn{2}{c}{Angle Distance $\downarrow$} & \multicolumn{3}{c}{Within $t^\circ$ $\uparrow$} & \\
    \cmidrule(lr){6-7}\cmidrule(lr){8-10}
    & & & & & Mean & Median & 11.25 & 22.5 & 30 & \\
    \midrule
    MoCo & 40.30 & 66.07 & 0.5575 & 0.2135 & 26.67 & 21.83 & 25.61 & 51.78 & 64.85 &  0.16 \\
    \midrule
    MoCo-warm start & 38.40 & 64.40 & 0.5377 & 0.2315 & 26.04 & 20.57 & 27.11 & 54.02 & 66.63 &  -0.88  \\
    \bottomrule
  \end{tabular}
  \end{adjustbox}
  % \vspace{0.1cm}
  \caption{Multi-task supervised learning with MoCo-warm start on NYU-v2 dataset.}\label{tab:nyuv2-addon}
  % \vspace{-0.5cm}
\end{table}

\subsection{Implementation details}
\textbf{Multi-task learning on Cityscapes dataset.} Following the experiment setup in \cite{xiao2024direction}, we train
our method for 200 epochs, using SGD optimizers for both model parameters and weights, and the batch size for Cityscapes is 8. We
compute the averaged test performance over the last 10 epochs as the final performance measure. We fix the $\beta=0.5$ and do a grid search on hyperparameters including $N\in[10, 20, 40, 50], \alpha\in[0.0001, 0.0002, 0.0005, 0.001]$, and $\rho\in[0.01, 0.05, 0.1, 0.2, 0.5, 0.6,0.7, 0.8,0.9, 1]$ and choose the best result from them. It turns out our best performance is based on the choice that $N=40, \alpha=0.0005,\beta=0.5$, and $\rho=0.5$. The choice of hyperparameters for MGDA-FA turns out to be the same as that for MGDA-warm start. All experiments are run on NVIDIA RTX A6000.

\textbf{Multi-task learning on NYU-v2 dataset.} Following the experiment setup in \cite{xiao2024direction}, we train
our method for 200 epochs, using SGD optimizers for both model parameters and weights, and the batch size for NYU-v2 is 2. We
compute the averaged test performance over the last 10 epochs as the final performance measure. We fix the $\beta=0.5$ and do a grid search on hyperparameters including $N\in[10, 20, 40, 50], \alpha\in[0.0001, 0.0002, 0.0005, 0.001]$, and $\rho\in[0.01, 0.05, 0.1, 0.2, 0.5, 0.6,0.7, 0.8,0.9, 1]$ and choose the best result from them. It turns out our best performance is based on the choice that $N=40, \alpha=0.0001,\beta=0.5$, and $\rho=0.5$. All experiments are run on NVIDIA RTX A6000.

$\Delta m\%$ reflects the average per-task performance drop versus the single-task (STL) baseline $b$ to assess method $m$. We calculate it by the following equation
$$\Delta m\%=\frac{1}{K}\sum_{k=1}^K (-1)^{l_k} (M_{m,k}-M_{b,k})\slash M_{b,k}\times 100,$$ where $K$ is the number of metrics, $M_{b,k}$ is the value of metric $M_k$ obtained by baseline $b$, and $M_{m,k}$ obtained by the compared method $m$. $l_k=1$ if the evaluation metric $M_k$ on task $k$ prefers a higher value and $0$ otherwise.

\textbf{Generalized smoothness illustration.} {To illustrate the relation between gradient norms and local smoothness, we follow the method in Section H.3 in \cite{zhang2019gradient}and run SGD on each task separately without the warm start process. The code is available in \href{https://github.com/JingzhaoZhang/why-clipping-accelerates}{https://github.com/JingzhaoZhang/why-clipping-accelerates.}} Since there is no weight update process, we only need to choose $\alpha=0.0005$ for both tasks. 
% \section{Notations}
% In this part, we first summarize all the notations that we used in this paper in order to help readers understand. First, in multi-objective optimization, we have $K\geq2$ different objectives and each of them has the objective function $f_i(\theta)$. Let $\nabla f_i(x)$ denote the gradient of the objective $i$. $w=(w_1, ..., w_K)^T\in\mathbb{R}^K\;\; and\;\; \mathcal{W}$ denotes the probability simplex. Other useful notations are listed below:
% \xiao{list notations we will use in the following proof}
\section{Algorithm}\label{sec:algorithm4}
% We show our MGDA with Fast Approximation (MGDA-FA): 
\begin{algorithm}[htbp]
   \caption{ MGDA with Fast Approximation (MGDA-FA)}
   \label{alg:gs-famo}
\begin{algorithmic}[1]\label{alg:4}
    \STATE {\bfseries Initialize:}  model parameters $x_0$, weights $w_0$ and a constant $\rho$
    % \STATE{$x_1=x_0-\alpha_0\nabla F(x_0)w_0$}
    \STATE{$w_0$=\textbf{Warm-start}($w_0$, $x_0$, $\rho$)}
    \FOR{$t=0, 1,...,T-1$}
    \STATE{$x_{t+1}=x_t-\alpha \nabla F(x_t)w_t$}
    \STATE{Update $w_t$ according to \cref{eq:updaterule}}
    
\ENDFOR
\end{algorithmic}
\end{algorithm}

\section{Detailed Proofs for Average CA Distance}
\subsection{Formal version and proof of \Cref{thm:gddeterministic}}\label{proof:gddeterministic-avg}
%Let $c_1>0, c_2>0$, $c_3\ge 0$ and $F>0$ be some constants such that 
Let $F>0$ be some constants such that
%\begin{flalign}\label{eq:selectf}
%    \Delta+c_1+{ c_2}+ c_3\le F.
%\end{flalign}
\begin{flalign}\label{eq:selectf}
    \Delta+{3}\le F.
\end{flalign}
Define $M=\sup\{z\ge0|\varphi(z)\le F\}$. We then have the following convergence rate for  Algorithm \ref{alg:1} without warm start:
\begin{theorem}
Suppose Assumptions \ref{ass:diffandlowerbound} and \ref{ass:lsmooth} are satisfied, and we choose constant step sizes that $\beta\le \frac{1}{4KM^2}, \alpha\le \min\Big\{{\beta},\frac{1}{2\ell(M+1)},\frac{1}{M\ell(M+1)}\Big\}, T\ge\max\Big\{(\frac{10\Delta}{\alpha \epsilon^2},\frac{10}{\epsilon^2\beta}\Big\})= \Theta(\epsilon^{-2})$, and $\rho\le\min \left(\frac{\epsilon^2}{20}, \sqrt{\frac{\epsilon^2}{10\beta}},  {\frac{1}{2T\alpha}, \sqrt{\frac{1}{T\alpha \beta}}}\right)= \mathcal O(\epsilon^2)$. We have that
\begin{flalign}
    \frac{1}{T}\sum_{t=0}^{T-1}\|\nabla F(x_t)w_t\|^2\le \epsilon^2.
\end{flalign}
\end{theorem}
%\begin{theorem}
%Suppose Assumptions \ref{ass:diffandlowerbound} and \ref{ass:lsmooth} are satisfied, and we choose constant step sizes that $\beta\le \frac{1}{4KM^2}, \alpha\le \min\left(c_1\beta,\frac{1}{2\ell(M+1)},\frac{1}{M\ell(M+1)}\right), T\ge\max\left(\frac{10\Delta}{\alpha \epsilon^2},\frac{10}{\epsilon^2\beta}\right)= \Theta(\epsilon^{-2})$, and $\rho\le\min \left(\frac{\epsilon^2}{20}, \sqrt{\frac{\epsilon^2}{10\beta}},  \frac{c_2}{2T\alpha}, \sqrt{\frac{c_3}{T\alpha \beta}}\right)= \mathcal O(\epsilon^2)$. We have that
%\begin{flalign}
%    \frac{1}{T}\sum_{t=0}^{T-1}\|\nabla F(x_t)w_t\|^2\le \epsilon^2.
%\end{flalign}
%\end{theorem}
\begin{proof}
Compared with the standard $L$-smoothness, the generalized smoothness is more challenging to address due to the unbounded Lipschitz constant. Lemma \ref{lemma:varphi} demonstrates that a bounded function value implies a bounded gradient norm, which further implies a bounded Lipschitz constant. In the following, we solve the unbounded Lipschitz constant problem by showing that the function value is bounded with the parameters selected in \Cref{thm:gddeterministic}. We prove that for any $i\in K$ and $t\le T$ we have that $f_i(x_t)-f_i^*\le F$ by induction.

\textbf{Base Case:} since $ M $ is non-negative, according to \eqref{eq:selectf} we have that $f_i(x_0)-f_i^*\le \Delta\le F$ holds for any $i\in[K]$.

\textbf{Induction step:} assume that for any $i\in[K]$ and $t\le k<T$, we have that $f_i(x_t)-f_i^*\le F$ holds. We then prove that $f_i(x_{k+1})-f_i^*\le F$ holds for any $i\in[K]$.

For $f_i(x_t)-f_i^*\le F$, based on the monotonicity shown in Lemma \ref{lemma:varphi}, we have that $\|\nabla f_i(x_t)\|\le M$. From assumption \ref{ass:lsmooth}, we have that $f_i(x)$ is $\left(\frac{1}{\ell(\|\nabla f_i(x))\|+1)},\ell(\|\nabla f_i(x)\|+1)\right)$-smooth by setting $a=1$. For any $t\le k$, we have that 
\begin{flalign}
    \|x_{t+1}-x_t\|=\alpha \|\nabla F(x_t)w_t\|\le \alpha M\le \frac{1}{\ell(M+1)}\le \frac{1}{\ell(\|\nabla f_i(x_t)\|+1)},\nonumber
\end{flalign}
where the second inequality is due to $\alpha\le \frac{1}{M\ell(M+1)}$ and the last inequality is due to $\|\nabla f_i(x_t)\|\le M$. Based on Assumption \ref{ass:lsmooth}, Definition \ref{define:rlsmooth} and Lemma 3.3 in \citep{li2024convex}, we have the following descent lemma:
\begin{flalign*}
    f_i(x_{t+1})&\le f_i(x_t)-\alpha\langle\nabla f_i(x_t),\nabla F(x_t)w_t\rangle+\frac{\ell(\|\nabla f_i(x_t)\|+1)}{2}\alpha^2\|\nabla F(x_t)w_t\|^2\nonumber\\
    &\le f_i(x_t)-\alpha\langle\nabla f_i(x_t),\nabla F(x_t)w_t\rangle+\frac{\ell(M+1)}{2}\alpha^2\|\nabla F(x_t)w_t\|^2.
\end{flalign*}
As a result, for any $w \in \mathcal W$, we have that
\begin{flalign}\label{eq:fwdescent}
    F(x_{t+1})w\le F(x_{t})w-\alpha\langle\nabla F(x_t)w,\nabla F(x_t)w_t\rangle+\frac{\ell(M+1)}{2}\alpha^2\|\nabla F(x_t)w_t\|^2.
\end{flalign}
Based on the update process of $w$, we have that 
\begin{flalign*}
   w_{t+1}
   &=\Pi_\mathcal{W}\Big(w_t-\beta\Big(\nabla F(x_t)^\top\nabla F(x_t)w_t+\rho w_t\Big)\Big).  
\end{flalign*}
It then follows that
\begin{flalign*}
    &\|w_{t+1}-w\|^2\nonumber\\
&=\Big\|\Pi_\mathcal{W}\Big(w_t-\beta\Big(\nabla F(x_t)^\top\nabla F(x_t)w_t+\rho w_t\Big)\Big)-w\Big\|^2\nonumber\\
    &\le \Big\|\Big(w_t-\beta\Big(\nabla F(x_t)^\top\nabla F(x_t)w_t+\rho w_t\Big)\Big)-w\Big\|^2\nonumber\\
    &=\|w_t-w\|^2-2\beta\left\langle w_t-w, (\nabla F(x_t)^\top\nabla F(x_t)+\rho I)w_t\right\rangle\nonumber\\
    &+\beta^2\left\|(\nabla F(x_t)^\top\nabla F(x_t)+\rho I)w_t\right\|^2,
\end{flalign*}
where the inequality is due to the non-expansiveness of projection.
By rearranging the above inequality, we have that
\begin{flalign}\label{eq:wdecsent}
    &\langle w_t-w, \nabla F(x_t)^\top\nabla F(x_t)w_t\rangle\nonumber\\
    &\le \frac{1}{2\beta}\left(\|w_{t}-w\|^2-\|w_{t+1}-w\|^2\right)+2\rho+\beta KM^2\|\nabla F(x_t)w_t\|^2+\beta \rho^2.
\end{flalign}
Plug \eqref{eq:wdecsent} into \eqref{eq:fwdescent}, and we can show that 
\begin{flalign}\label{eq:fwdescent2}
    &F(x_{t+1})w- F(x_{t})w\nonumber\\
    &\le -\alpha\|\nabla F(x_t)w_t\|^2+\frac{\ell(M+1)}{2}\alpha^2\|\nabla F(x_t)w_t\|^2\nonumber\\
    &+\frac{\alpha}{2\beta}\left(\|w_{t}-w\|^2-\|w_{t+1}-w\|^2\right)+\alpha\beta KM^2\|\nabla F(x_t)w_t\|^2+\alpha\beta \rho^2+2\alpha\rho.
\end{flalign}
Taking sums of \eqref{eq:fwdescent2} from $t=0$ to $k$, for any $w\in\mathcal W$ we have that
\begin{flalign}\label{eq:gddetermi}
    &F(x_{k+1})w- F(x_{0})w\nonumber\\
    &\le -\sum_{t=0}^k\alpha\|\nabla F(x_t)w_t\|^2+\sum_{t=0}^k\left(\frac{\ell(M+1)}{2}\alpha^2+\alpha\beta KM^2\right)\|\nabla F(x_t)w_t\|^2\nonumber\\
    &+\frac{\alpha}{2\beta}\|w_{0}-w\|^2+T\alpha\beta\rho^2+2T\alpha\rho\nonumber\\
    &\le \frac{\alpha}{2\beta}\|w_{0}-w\|^2+T\alpha\beta\rho^2+2T\alpha\rho,
\end{flalign}
where the first inequality is due to $k<T$ and the last inequality is due to that $\alpha\le \frac{1}{2\ell(M+1)}$ and $\beta KM^2 \le{\frac{1}{4}}$.
Thus for any $i\in [K]$ it can be shown that
\begin{flalign*}
    f_i(x_{t+1})-f_i^*\le f_i(x_0)-f_i^*+\frac{\alpha}{\beta}+T\alpha\beta\rho^2+2T\alpha\rho\le F,
\end{flalign*}
since we have that {$\frac{\alpha}{\beta}\le 1, T\alpha \beta \rho^2\le 1$ and $2T\alpha \rho \le 1$}. Now we finish the induction step and can show that $f_i(x_{k+1})-f_i^*\le F$ and \eqref{eq:gddetermi} hold for all $k<T$ and $i\in [K]$.\\

Specifically, for $\alpha<\frac{1}{2\ell(M+1)},\beta\le\frac{1}{4KM^2}$, according to \eqref{eq:gddetermi}, for $k=T-1$
we have that
\begin{flalign*}
    \frac{1}{T}\sum_{t=0}^{T-1}\|\nabla F(x_t)w_t\|^2\le \frac{2(F(x_0)w-F^*w)}{\alpha T}+\frac{2}{\beta T}+2\beta\rho^2+4\rho=\mathcal{O}(\epsilon^2),
\end{flalign*}
which completes our proof.
\end{proof}
\begin{lemma}\label{lemma:varphi}
    (Lemma 3.5 in \cite{li2024convex}) If a function $f$ is $\ell$-smooth, we have that
    \begin{flalign}\label{eq:lemma1}
        \varphi(\|\nabla f(x)\|)=\frac{\|\nabla f(x)\|^2}{2\ell(2\|\nabla f(x)\|)}\le f(x)-f^*.
    \end{flalign}
\end{lemma}
\subsection{Proof of \Cref{coro:gdaverageca}}\label{proof:gdaverageca}
\begin{proof}
{Recall that 
% \begin{align*}
% \|\nabla F(x_t)w_t-\nabla F(x_t)w_t^*\|^2\leq\|\nabla F(x_t)w_t\|^2-\|\nabla F(x_t)w_t^*\|^2\leq\|\nabla F(x_t)w_t\|^2,
% \end{align*}
\begin{align*}
&\|\nabla F(x_t)w_t-\nabla F(x_t)w_t^*\|^2\\
=&\|\nabla F(x_t)w_t\|^2+\|\nabla F(x_t)w_t^*\|^2-2\langle\nabla F(x_t)w_t,\nabla F(x
_t)w_t^*\rangle\\
\leq&\|\nabla F(x_t)w_t\|^2+\|\nabla F(x_t)w_t^*\|^2-2\|\nabla F(x_t)w_t^*\|^2\\
=&\|\nabla F(x_t)w_t\|^2-\|\nabla F(x_t)w_t^*\|^2\\
\leq&\|\nabla F(x_t)w_t\|^2,
\end{align*}
where the first inequality holds because of the optimality condition of $w_t^*=\arg\min_{w\in\mathcal{W}} \frac{1}{2}\|\nabla F(x_t)w\|^2$ that for any $w\in\mathcal{W}$, $$\langle w, \nabla F(x_t)^\top\nabla F(x_t)w_t^*\rangle\geq\langle w_t^*, \nabla F(x_t)^\top\nabla F(x_t)w_t^*\rangle=\|\nabla F(x_t)w_t^*\|^2.$$} Then we have
\begin{align*}
\frac{1}{T}\sum_{t=0}^{T-1}\|\nabla F(x_t)w_t-\nabla F(x_t)w_t^*\|^2\leq\frac{1}{T}\sum_{t=0}^{T-1}\|\nabla F(x_t)w_t\|^2=\mathcal{O}(\epsilon^2),
\end{align*}
where we follow the same setting in \Cref{thm:gddeterministic}. The proof is complete.
\end{proof}

\subsection{Formal Version and Its Proof of \Cref{thm:gdstochastic}}\label{proof:stochasticaverage}
Let $c_1>0, c_2>0, c_3>0, c_4> 0, c_5> 0, c_6>0$ be some constants.
Let $F>0$ and $0<\delta\le \frac{1}{2}$ be some constants such that 
\begin{flalign*}
    F\ge \frac{8(\Delta+c_1+c_2+c_3+c_4+c_5+c_6)}{\delta},
\end{flalign*}
where 
$\Delta=\max_{i\in[d]} \{f_i(x_0)-f^*\}$. Let $L_0>0, L_1>0, b_1>0, b_2>0, b_3>0$ be some constants,
{for $B_1=\min\left\{ \frac{1}{4L_0^2\ell(M+1)^2}, \frac{b_1^2}{(3M\sqrt{K}L_0+M\sqrt{K}L_1)^2}, \frac{b_2}{\ell(M+1)L_0^2 }\right\}$.\\
$\rho\le \min\{\frac{\delta\epsilon^2}{48}, \sqrt{\frac{\delta\epsilon^2}{48}},\frac{L_0^2\delta^2\epsilon^2}{1152K\sigma^2(\Delta+c_1) },\frac{L_1^2\delta^2\epsilon^2}{384K\sigma^2(\Delta+c_1) } \}=\mathcal O(\delta^2 \epsilon^2)$,\\
$\beta\le \min\{  \frac{\delta \epsilon^2}{192 K(M^2+\sigma^2)^2}, \frac{b_3\rho}{8KM^2L_0^2+4KL_1^2 },1\}=\mathcal O(\min\{\frac{\delta\epsilon^2}{M^4},\frac{\rho}{M^2}\})$,\\
$\alpha \le \min\{\frac{\ell(M+1)}{2\max(1,M)},c_1\beta,\rho B_1, \frac{c_2}{\sqrt{2T}M(3\sigma+\sigma^2),},\frac{\sqrt{c_3}}{\sqrt{\ell(M+1)\sigma^2T}},\frac{\delta\epsilon^2}{48\sigma^2\ell(M+1)},\\ \frac{1}{\rho T}\min\{c_5,\frac{\delta L_0^2}{24K\sigma^2},\frac{\delta L_1^2}{8K\sigma^4} \},\frac{c_4}{4K\beta T(M^2+\sigma^2)^2},\frac{c_6}{\rho \beta T}\}=\mathcal O(\min\{\beta, \frac{\rho}{\ell(M+1)^2},  \frac{\rho}{M^2}, \frac{\delta }{\rho T},\frac{1}{\beta TM^2}\})$\\
and  $T\ge \max\{\frac{48\Delta+48c_1}{\delta\alpha\epsilon^2},\frac{4608M^2(3\sigma+\sigma^2)^2}{\delta^2\epsilon^4} \}=\mathcal O(\max\{\frac{1}{\delta \alpha \epsilon^2},\frac{M^2}{\delta^2 \epsilon^4}\})$
}
%\noindent For 
%$\alpha \le \min\{\frac{\ell(M+1)}{2\max(1,M)},c_1\beta,\rho\min\left\{ \frac{1}{4L_0^2\ell(M+1)^2}, \frac{b_1^2}{(3M\sqrt{K}L_0+M\sqrt{K}L_1)^2}, \frac{b_2}{\ell(M+1)L_0^2 }\right\}, \frac{c_2}{\sqrt{2T}M(3\sigma+\sigma^2),}, \\\frac{\sqrt{c_3}}{\sqrt{\ell(M+1)\sigma^2T}},\frac{\delta\epsilon^2}{48\sigma^2\ell(M+1)}\}$,\\
%$\beta\le \min\{\frac{c_4}{4K\alpha T(M^2+\sigma^2)^2},  \frac{\delta \epsilon^2}{192 K(M^2+\sigma^2)^2}, \frac{b_3\rho}{8KM^2L_0^2+4KL_1^2 }\},\\
%\rho\le \min\{\frac{c_5}{\alpha T},{\frac{\sqrt{c_6}}{\sqrt{\alpha\beta T}}}, \frac{\delta\epsilon^2}{48}, \sqrt{\frac{\delta\epsilon^2}{48\beta}}, \frac{\delta L_0^2}{24K\sigma^2\alpha T}, \frac{\delta L_1^2}{8K^2\sigma^4\alpha T}\}$,\\ and  $T\ge \max\{\frac{48\Delta+48c_1}{\delta\alpha\epsilon^2},\frac{4608M^2(3\sigma+\sigma^2)^2}{\delta^2\epsilon^4}\}$ such that
such that
\begin{flalign}\label{eq:bcondition}
&b_1+b_2+b_3+c_1+\alpha\rho(1+\beta\rho)+4\alpha\beta  KM^4\le \frac{F}{2}.
\end{flalign}
Define the following random variables
\begin{flalign*}
    &\tau_1=\min\{t | \exists i \in [K], f_i(x_{t+1})-f_i^*>F\} \wedge T,\nonumber\\
    &\tau_2=\min\{t | \exists i \in [K], j\in [3] ,\|\varepsilon_{t,j,i}\|>\frac{L_0}{\sqrt{\alpha\rho} }\}\wedge T,\nonumber\\
    &\tau_3=\min\{t | \exists i,j \in [K],\|\varepsilon_{t,2,i}\|\|\varepsilon_{t,3,j}\|>\frac{L_1}{\sqrt{\alpha\rho} }\}\wedge T,\nonumber\\
    &\tau=\min\{\tau_1, \tau_2, \tau_3\}.
\end{flalign*}  
We then have the following theorem:
\begin{theorem}
 If Assumptions \ref{ass:diffandlowerbound}, \ref{ass:lsmooth} and \ref{ass:boundedvariance} hold, with the parameters selected above, we have that
\begin{flalign*}
    \frac{1}{T}\sum_{t=0}^{T-1}\|\nabla F(x_t)w_t\|^2=\mathcal{O}( \epsilon^2),
\end{flalign*}
with the probability at least $1-\delta$.
\end{theorem}
\begin{proof}

\textbf{Small probability of the event $\{\tau<T\}$.} 

We first show that the probability of the event $\{\tau<T\}$ is small: $\mathbb P(\tau<T)\le \delta$. Note that 
\begin{flalign*}
    \{\tau<T\}= \{\tau_2<T\}\cup\{\tau_3<T\}\cup  \{\tau_1<T, \tau_2=T,\tau_3=T\}.
\end{flalign*}
For any $i \in [K], j\in [3]$, we have that
\begin{flalign*}
    \mathbb P(\|\varepsilon_{t,j,i}\|>\frac{L_0}{\sqrt{\alpha\rho} })= \mathbb P(\|\varepsilon_{t,j,i}\|^2>\frac{L_0^2}{{\alpha\rho} })\le \frac{\sigma^2{\alpha\rho}}{L_0^2},
\end{flalign*}
where the last inequality is due to Chebyshev’s inequality. Based on the union bound, we have that 
\begin{flalign}\label{eq:prob1}
    \mathbb P(\{\tau_2<T\})\le \sum_{t=0}^{T-1}\sum_{j=1}^K\sum_{i=1}^3 \mathbb P(\|\varepsilon_{t,j,i}\|>\frac{L_0}{\sqrt{\alpha\rho} })\le \frac{3K\sigma^2{\alpha\rho}T}{L_0^2}\le \frac{\delta}{8}
\end{flalign}
since $\rho\le \frac{\delta L_0^2}{24K\sigma^2\alpha T}$.
Similarly, we have that 
\begin{flalign*}
    \mathbb P(\|\varepsilon_{t,2,i}\|\|\varepsilon_{t,3,i}\|>\frac{L_1}{\sqrt{\alpha\rho} })= \mathbb P(\|\varepsilon_{t,2,i}\|^2\|\varepsilon_{t,3,i}\|^2>\frac{L_1^2}{{\alpha\rho} })\le \frac{\sigma^4{\alpha\rho}}{L_1^2}.
\end{flalign*}
It follows that
\begin{flalign}\label{eq:prob2}
    \mathbb P(\{\tau_3<T\})\le \sum_{t=0}^{T-1}\sum_{i=1}^K\sum_{j=1}^K \mathbb P(\|\varepsilon_{t,2,i}\|\|\varepsilon_{t,3,j}\|>\frac{L_1}{\sqrt{\alpha\rho} })\le \frac{K^2\sigma^4{\alpha\rho}T}{L_1^2}\le \frac{\delta}{8}.
\end{flalign}
We then bound the probability of the event $\{\tau_1<T, \tau_2=T,\tau_3=T\}$. Since $\tau=\tau_1<T,$ we have that for some $i\in [K]$, $f_i(x_{\tau+1})-f_i^*>F$.

According to \eqref{eq:sfinaldescent} shown in Lemma \ref{lemma:sgdlemmaresult}, for any $i\in[K]$ and $t=\tau$ we have that
\begin{flalign*}
    f_i(x_{\tau+1})-f_i(x_{\tau})&\le\alpha \|\nabla F(x_\tau)w\|\|\varepsilon_{t,1}w_t\|+{\ell(M+1)}\alpha^2\|\varepsilon_{t,1}w_t\|^2+\frac{\alpha}{\beta}+\alpha\rho+\alpha\beta\rho^2\nonumber\\
    &+\alpha M\|\varepsilon_{t,2}\|+\alpha M\|\varepsilon_{t,3}\| +\alpha \|\varepsilon_{t,2}^\top\varepsilon_{t,3}w_t\|
    \nonumber\\
       &+4\alpha\beta KM^4+4\alpha\beta  M^2\|\varepsilon_{t,2}\|^2+4\alpha\beta M^2\|\varepsilon_{t,3}\|^2+4\alpha\beta \|\varepsilon_{t,2}^\top \varepsilon_{t,3}w_t\|^2\nonumber\\
    &\le \alpha M\frac{L_0}{\sqrt{\alpha\rho}}+{\ell(M+1)}\alpha\frac{L_0^2}{{\rho}}+\frac{\alpha}{\beta}+\alpha\rho+\alpha\beta\rho^2\nonumber\\
    &+\alpha M\frac{\sqrt{K}L_0}{\sqrt{\alpha\rho}}+\alpha M\frac{\sqrt{K}L_0}{\sqrt{\alpha\rho}}+\alpha M\frac{\sqrt{K}L_1}{\sqrt{\alpha\rho}}\nonumber\\
    &+4\alpha\beta KM^4+4\alpha\beta  M^2\frac{KL_0^2}{{\alpha\rho}}+4\alpha\beta M^2\frac{KL_0^2}{{\alpha\rho}}+4\alpha\beta \frac{KL_1^2}{{\alpha\rho}}\nonumber\\
    &\le b_1+b_2+b_3+c_1+\alpha\rho(1+\beta\rho)+4\alpha\beta  KM^4\nonumber\\
    &\le \frac{F}{2},
\end{flalign*}
where the first inequality is due to that $\tau_2=\tau_3=T$, and the second one is due to  $\frac{\beta}{\rho}\le \frac{b_3}{8KM^2L_0^2+4KL_1^2 }$ and $\frac{\alpha}{\rho}\le\min\left\{\frac{b_1^2}{(3M\sqrt{K}L_0+M\sqrt{K}L_1)^2}, \frac{b_2}{\ell(M+1)L_0^2 }  \right\}$.

However, for some $i\in [K]$, $f_i(x_{\tau+1})-f_i^*>F$. Thus for this task, we have that
\begin{flalign*}
    f_i(x_{\tau})-f_i^*> \frac{F}{2}.
\end{flalign*}
According to Lemma \ref{lemma:sgdlemmaresult}, we have that
\begin{flalign*}
    \mathbb E[f_i(x_{\tau})-f_i^*]\le \frac{\delta}{8}F.
\end{flalign*}
Based on Markov inequality, it follows that
\begin{flalign}\label{eq:prob3}
    \mathbb P\left(f_i(x_{\tau})-f_i^*\le \frac{F}{2}\right)\le \frac{\mathbb E[f_i(x_{\tau})-f_i^*]}{F/2}\le \frac{\delta}{4},
\end{flalign}
which indicates that $\mathbb  P(\tau_1<T, \tau_2=T, \tau_3=T)\le \frac{\delta}{4}$. It follows that $\mathbb P(\tau<T)\le \frac{\delta}{2}$. 

\textbf{Convergence of $\frac{1}{T}\mathbb E\left[\sum_{t=0}^{T-1}\|\nabla F(x_t)w_t\|^2\Big|\tau=T\right]$.} 
Based on \eqref{eq:sfinaldescent3} in Lemma \ref{lemma:sgdlemmaresult}, we have that
\begin{flalign*}
    &\frac{1}{T}\mathbb E\left[\sum_{t=0}^{T-1}\|\nabla F(x_t)w_t\|^2\Big|\tau=T\right] \nonumber\\
    &\le \frac{1}{T}\frac{1}{\mathbb P(\tau=T)}\mathbb E\left[\sum_{t=0}^{\tau-1}\|\nabla F(x_t)w_t\|^2\right] \nonumber\\
    &\le \frac{4F(x_0)w-4F^*w+\frac{4\alpha}{\beta}}{\alpha T}+\frac{4\sqrt{2}M(3\sigma+\sigma^2)}{\sqrt{T}}+4\alpha\sigma^2\ell(M+1)+4\rho\nonumber\\
    &+4\beta\rho^2+16\beta KM^4+32\beta KM^2\sigma^2+16\beta K\sigma^4\nonumber\\
    &\le \frac{\delta}{2} \epsilon^2,
\end{flalign*}
where the second inequality is due to $\delta<\frac{1}{2}$ and the last inequality is due to our selection of parameters. 
As a result, we have that
\begin{flalign}\label{eq:prob4}
     &\mathbb P\left(\frac{1}{T}\sum_{t=0}^{T-1}\|\nabla F(x_t)w_t\|^2> \epsilon^2\Big|\tau=T\right)\le \frac{\mathbb E\left[\frac{1}{T}\sum_{t=0}^{T-1}\|\nabla F(x_t)w_t\|^2\ge \epsilon^2\Big|\tau=T\right]}{\epsilon^2}\le \frac{\delta}{2},
\end{flalign}
where the first probability is due to Markov inequality.
Thus we have that
\begin{flalign*}
    &\mathbb P\left(\frac{1}{T}\sum_{t=0}^{T-1}\|\nabla F(x_t)w_t\|^2\le \epsilon^2\right)\nonumber\\
    &\ge 1- \mathbb P\left(\tau<T\right)-\mathbb P\left(\frac{1}{T}\sum_{t=0}^{T-1}\|\nabla F(x_t)w_t\|^2> \epsilon^2\Big|\tau=T\right)\mathbb P\left(\tau=T\right)\nonumber\\
    &\ge 1-\delta,
\end{flalign*}
where the last inequality is due to \eqref{eq:prob1}, \eqref{eq:prob2}, \eqref{eq:prob3}, and \eqref{eq:prob4}.
This completes the proof.
\end{proof}
\subsection{Proof of Lemma \ref{lemma:sgdlemmaresult}}\label{proof:lemmasgd}
\begin{proof}
    For all $i\in [K], t\le \tau$, we have $f_i(x_{t})-f_i^*\le F$ which further implies that $\|\nabla f_i(x_{t})\|\le M$. Moreover, we have that for any $t\le \tau$ and $i\in [K]$, 
\begin{flalign*}
    \|x_{t+1}-x_t\|\le \alpha\|\nabla G_1(x_t)w_t\|\le \alpha(\|\nabla F(x_t)w_t\|+\|\varepsilon_{t,1}w_t\|)\le \alpha\left(M+\frac{L_0}{\sqrt{\alpha\rho}}\right)\le \frac{1}{\ell(M+1)}.
\end{flalign*}
Since $f_i(x)$ is $\left(\frac{1}{\ell(\|\nabla f_i(x))\|+1)},\ell(\|\nabla f_i(x)\|+1)\right)$-smooth, it follows that
\begin{flalign*}
    f_i(x_{t+1})-f_i(x_{t})\le -\alpha\langle\nabla f_i(x_t),\nabla G_1(x_t)w_t\rangle+\frac{\ell(\|\nabla f_i(x_t)\|+1)}{2}\alpha^2\|\nabla G_1(x_t)w_t\|^2.\nonumber
\end{flalign*}
As a result, for any $w \in \mathcal W$, we have that
\begin{flalign}\label{eq:sfwdescent}
    F(x_{t+1})w&\le F(x_{t})w-\alpha\langle\nabla F(x_t)w,\nabla G_1(x_t)w_t\rangle+\frac{\ell(M+1)}{2}\alpha^2\|\nabla G_1(x_t)w_t\|^2\nonumber\\
    &\le F(x_{t})w-\alpha\langle\nabla F(x_t)w,\nabla F(x_t)w_t\rangle+ \alpha\langle\nabla F(x_t)w,\varepsilon_{t,1}w_t\rangle  \nonumber\\
    &+ {\ell(M+1)}\alpha^2\|\nabla F(x_t)w_t\|^2+ {\ell(M+1)}\alpha^2\|\varepsilon_{t,1}w_t\|^2
\end{flalign}
Based on the update process of $w$, we have that 
\begin{flalign*}
    &\|w_{t+1}-w\|^2\nonumber\\
    &=\|\Pi_\mathcal{W}\big(w_t-\beta[{\nabla G_2(x_t)^\top\nabla G_3(x_t)w_t}+\rho w_t]\big)-w\|^2\nonumber\\
    &\le \|\big(w_t-\beta[{\nabla G_2(x_t)^\top\nabla G_3(x_t)w_t}+\rho w_t]\big)-w\|^2\nonumber\\
    &=\|w_t-w\|^2-2\beta\langle w_t-w, (\nabla G_2(x_t)^\top\nabla G_3(x_t)+\rho)w_t\rangle\nonumber\\
    &+\beta^2\|(\nabla G_2(x_t)^\top\nabla G_3(x_t)+\rho)w_t\|^2,
\end{flalign*}
where the inequality follows from the non-expansiveness of projection. 
It follows that
\begin{flalign}\label{eq:swdecent}
    &2\beta\langle w_t-w, \nabla F(x_t)^\top\nabla F(x_t)w_t\rangle\nonumber\\
    &\le \left(\|w_{t}-w\|^2-\|w_{t+1}-w\|^2\right)+2\beta\rho +2\beta^2 \rho^2\nonumber\\
    &+2\beta \left\langle w_t-w, \varepsilon_{t,2}^\top \nabla F(x_t)w_t+  \nabla F(x_t)^\top\varepsilon_{t,3}w_t-\varepsilon_{t,2}^\top\varepsilon_{t,3}w_t\right\rangle\nonumber\\
   &+2\beta^2 \|\nabla F(x_t)^\top\nabla F(x_t)w_t-\varepsilon_{t,2}^\top\nabla F(x_t)w_t-\nabla F(x_t)^\top\varepsilon_{t,3}w_t+\varepsilon_{t,2}^\top\varepsilon_{t,3}w_t\|^2\nonumber\\
     &\le \left(\|w_{t}-w\|^2-\|w_{t+1}-w\|^2\right)+2\beta\rho +2\beta^2 \rho^2\nonumber\\
     &+2\beta \left\langle w_t-w, \varepsilon_{t,2}^\top \nabla F(x_t)w_t+  \nabla F(x_t)^\top\varepsilon_{t,3}w_t-\varepsilon_{t,2}^\top\varepsilon_{t,3}w_t\right\rangle\nonumber\\
       &+8\beta^2 KM^4+8\beta^2 M^2\|\varepsilon_{t,2}\|^2+8\beta^2 M^2\|\varepsilon_{t,3}\|^2+8\beta^2 \|\varepsilon_{t,2}^\top \varepsilon_{t,3}w_t\|^2.
\end{flalign} 
Combine \eqref{eq:sfwdescent} and \eqref{eq:swdecent}, and we can get that
\begin{flalign}\label{eq:sfinaldescent}
      F(x_{t+1})w-F(x_{t})w
    &\le -\alpha\|\nabla F(x_t)w_t\|^2+ \alpha\langle\nabla F(x_t)w,\varepsilon_{t,1}w_t\rangle  \nonumber\\
     &+ {\ell(M+1)}\alpha^2\|\nabla F(x_t)w_t\|^2+ {\ell(M+1)}\alpha^2\|\varepsilon_{t,1}w_t\|^2\nonumber\\
          &+ \frac{\alpha}{2\beta}\left(\|w_{t}-w\|^2-\|w_{t+1}-w\|^2\right)+\alpha\rho +\alpha\beta \rho^2\nonumber\\
    &+\alpha \left\langle w_t-w, \varepsilon_{t,2}^\top \nabla F(x_t)w_t+\nabla F(x_t)w_t^\top\varepsilon_{t,3} -\varepsilon_{t,2}^\top \varepsilon_{t,3}w_t \right\rangle\nonumber\\
    &+4\alpha\beta KM^4+4\alpha\beta  M^2\|\varepsilon_{t,2}\|^2+4\alpha\beta  M^2\|\varepsilon_{t,3}\|^2+4\alpha\beta \|\varepsilon_{t,2}^\top \varepsilon_{t,3}w_t\|^2.
\end{flalign}

Taking expectation and sum up \eqref{eq:sfinaldescent} from $t=0$ to $\tau-1$, we have that
\begin{flalign}\label{eq:sfinaldescent2}
      \mathbb E[F(x_{\tau})w]-F(x_{0})w
    &\le -\frac{\alpha}{2}\mathbb E\left[\sum_{t=0}^{\tau-1}\|\nabla F(x_t)w_t\|^2\right]+ \alpha\mathbb E\left[\sum_{t=0}^{\tau-1}\langle\nabla F(x_t)w,\varepsilon_{t,1}w_t\rangle\right]\nonumber\\  
    &+\alpha \mathbb E\left[\sum_{t=0}^{\tau-1}\left\langle w_t-w, \varepsilon_{t,2}^\top \nabla F(x_t)w_t+\nabla F(x_t)w_t^\top\varepsilon_{t,3} -\varepsilon_{t,2}^\top \varepsilon_{t,3}w_t \right\rangle\right]\nonumber\\
    &+ {\ell(M+1)}\alpha^2\mathbb E\left[\sum_{t=0}^{\tau-1}\|\varepsilon_{t,1}w_t\|^2\right]+ \frac{\alpha}{2\beta}\|w_{0}-w\|^2+\alpha\rho T +\alpha\beta \rho^2T\nonumber\\
    &+4\alpha\beta KM^4T+4\alpha\beta  M^2\mathbb E\left[\sum_{t=0}^{\tau-1}\|\varepsilon_{t,2}\|^2\right]\nonumber\\
    &+4\alpha\beta  M^2\mathbb E\left[\sum_{t=0}^{\tau-1}\|\varepsilon_{t,3}\|^2\right]+4\alpha\beta \mathbb E\left[\sum_{t=0}^{\tau-1}\|\varepsilon_{t,2}^\top \varepsilon_{t,3}w_t\|^2\right]\nonumber\\
        &\le -\frac{\alpha}{2}\mathbb E\left[\sum_{t=0}^{\tau-1}\|\nabla F(x_t)w_t\|^2\right]+ \alpha\mathbb E\left[\sum_{t=0}^{\tau-1}\langle\nabla F(x_t)w,\varepsilon_{t,1}w_t\rangle\right]\nonumber\\  
        &+\alpha \mathbb E\left[\sum_{t=0}^{\tau-1}\left\langle w_t-w, \varepsilon_{t,2}^\top \nabla F(x_t)w_t+\nabla F(x_t)w_t^\top\varepsilon_{t,3} -\varepsilon_{t,2}^\top \varepsilon_{t,3}w_t \right\rangle\right]\nonumber\\
        &+ {\ell(M+1)}\alpha^2T\sigma^2+ \frac{\alpha}{\beta}+\alpha\rho T +\alpha\beta \rho^2T\nonumber\\
    &+4\alpha\beta KM^4T+4\alpha\beta K M^2T\sigma^2\nonumber\\
    &+4\alpha\beta K M^2T\sigma^2+4\alpha\beta TK\sigma^4,
\end{flalign}
where the last inequality is due to that  $\tau\le T$ and for any $i\in[K], j\in[3]$, $\mathbb E\left[\sum_{t=0}^{\tau-1}\|\varepsilon_{t,j,i}\|^2\right]\le\mathbb E\left[\sum_{t=0}^{T-1}\|\varepsilon_{t,j,i}\|^2\right]\le TK \sigma^2$. By the optional stopping
theorem, we have that
\begin{flalign*}
    \mathbb E\left[\sum_{t=0}^{\tau}\langle\nabla F(x_t)w,\varepsilon_{t,1}w_t\rangle\right]=0,
\end{flalign*}
which further implies that
\begin{flalign}\label{eq:expectstop}
    \mathbb E\left[\sum_{t=0}^{\tau-1}\langle\nabla F(x_t)w,\varepsilon_{t,1}w_t\rangle\right]&=-\mathbb E[\langle\nabla F(x_\tau)w,\varepsilon_{\tau,1}w_\tau\rangle]\nonumber\\
    &\le\mathbb E[M\|\varepsilon_{\tau,1}w_\tau\|]
    \le M \sqrt{\mathbb E[\|\varepsilon_{\tau,1}w_\tau\|^2]}\nonumber\\
    &\le M \sqrt{\mathbb E\left[\sum_{t=0}^T\|\varepsilon_{t,1}w_t\|^2\right]}\le M\sigma\sqrt{T+1}\nonumber\\
    &\le \sqrt{2}M\sigma\sqrt{T}.
\end{flalign}
Similarly, we have  $ \mathbb E\left[\sum_{t=0}^{\tau-1}\langle\nabla F(x_t)w,\varepsilon_{t,2}w_t\rangle\right]\le \sqrt{2}M\sigma\sqrt{T}$, $ \mathbb E\left[\sum_{t=0}^{\tau-1}\langle\nabla F(x_t)w,\varepsilon_{t,3}w_t\rangle\right]\le \sqrt{2}M\sigma\sqrt{T}$
and $ \mathbb E\left[\sum_{t=0}^{\tau-1}\langle\nabla F(x_t)w,\varepsilon_{t,2}^\top \varepsilon_{t,3}w_t\rangle\right]\le \sqrt{2K}M\sigma^2\sqrt{T}$.

Based on \eqref{eq:sfinaldescent2} and \eqref{eq:expectstop}, we have that
\begin{flalign}\label{eq:sfinaldescent3}
\mathbb E[F(x_{\tau})w]-F^*w
        &\le F(x_{0})w-F^*w-\mathbb E\left[\sum_{t=0}^{\tau-1}\frac{\alpha}{2}\|\nabla F(x_t)w_t\|^2\right]+ \alpha\sqrt{2T}M(3\sigma+\sigma^2)\nonumber\\  
        &+ {\ell(M+1)}\alpha^2T\sigma^2+ \frac{\alpha}{\beta}+\alpha\rho T +\alpha\beta \rho^2T\nonumber\\
    &+4\alpha\beta KM^4T+4\alpha\beta K M^2T\sigma^2\nonumber\\
    &+4\alpha\beta K M^2T\sigma^2+4\alpha\beta TK\sigma^4\nonumber\\
    &\le \frac{\delta F}{8}-\mathbb E\left[\sum_{t=0}^{\tau-1}\frac{\alpha}{2}\|\nabla F(x_t)w_t\|^2\right],
\end{flalign}
which completes the proof.
\end{proof}

\section{Detailed Proofs for Iteration-wise CA Distance}
We first provide some useful lemmas, which will be used in our main theorems.
\begin{lemma}[Continuity of $w_{t,\rho}^*$]\label{lemma:rho}
Suppose Assumptions \ref{ass:diffandlowerbound} and \ref{ass:lsmooth} are satisfied. If for any $i\in[K], \|\nabla f_i(x_t)\|\le M$ and $\|x_t-x_{t+1}\|\le \frac{1}{\ell(M+1)}$, we have,
\begin{align*}
\|w_\rho^*(x_t)-w_\rho^*(x_{t+1})\|\leq L_w\|x_t-x_{t+1}\|,
\end{align*}
{where $L_w=2\rho^{-1}KM\ell(M+1)$.}
\begin{proof}
We first define that $w_{Q,\rho}(x_t)\in\mathcal{W}$ is the $Q$-th iterate of a function $J(w)=\frac{1}{2}\|\nabla F(x_t)w\|^2+\frac{\rho}{2}\|w\|^2$ using projected gradient descent (PGD) with a constant step size $\beta$. The update rule is $w_{Q+1,\rho}(x_t)=\Pi_\mathcal{W}\Big(\big((1-\beta\rho)I-\beta\nabla F(x_t)^\top\nabla F(x_t)\big)w_{Q,\rho}(x_t)\Big)$. By the non-expansiveness of projection, we have
\begin{align*}
\|w_{Q+1,\rho}&(x_t)-w_{Q+1,\rho}(x_{t+1})\|\nonumber\\
\leq&\|\big((1-\beta\rho)I-\beta\nabla F(x_t)^\top\nabla F(x_t)\big)w_{Q,\rho}(x_t)\nonumber\\
&-\big((1-\beta\rho)I-\beta\nabla F^\top(x_{t+1})\nabla F(x_{t+1})\big)w_{Q,\rho}(x_{t+1})\|\nonumber\\
\leq&\|(1-\beta\rho)I-\beta\nabla F(x_t)^\top\nabla F(x_t)\|\|w_{Q,\rho}(x_t)-w_{Q,\rho}(x_{t+1})\|\nonumber\\
&+\beta\|\big(\nabla F(x_t)^\top\nabla F(x_t)-\nabla F^\top(x_{t+1})\nabla F(x_{t+1})\big)w_{Q,\rho}(x_{t+1})\|\nonumber\\
\leq&(1-\beta\rho)\|w_{Q,\rho}(x_t)-w_{Q,\rho}(x_{t+1})\|\nonumber\\
&+\beta\|\big(\nabla F(x_t)^\top\nabla F(x_t)-\nabla F^\top(x_{t+1})\nabla F(x_{t+1})\big)w_{Q,\rho}(x_{t+1})\|.
\end{align*}
Since we set $w_{0,\rho}(x_t)=w_{0,\rho}(x_{t+1})$ and $\|w_{Q,\rho}(x_{t+1})\|\leq 1$, telescoping the above inequality over $Q$ gives,
\begin{align}\label{eq:oneoverrho}
\|w_{Q,\rho}&(x_t)-w_{Q,\rho}(x_{t+1})\|\leq\rho^{-1}\|\nabla F(x_t)^\top\nabla F(x_t)-\nabla F^\top(x_{t+1})\nabla F(x_{t+1})\|.
\end{align}
Then according to the Cauchy-Schwartz inequality, it follows that
\begin{align}\label{eq:stargap}
\|w_\rho^*(x_t)-w_\rho^*(x_{t+1})\|\leq&\lim_{Q\rightarrow\infty}\big(\|w_\rho^*(x_t)-w_{Q,\rho}(x_t)\|+\|w_\rho^*(x_{t+1})-w_{Q,\rho}(x_{t+1})\|\nonumber\\
&+\|w_{Q,\rho}(x_t)-w_{Q,\rho}(x_{t+1})\|\big)\nonumber\\
\overset{(i)}{\leq}&\lim_{Q\rightarrow\infty}\big(\|w_\rho^*(x_t)-w_{Q,\rho}(x_t)\|+\|w_\rho^*(x_{t+1})-w_{Q,\rho}(x_{t+1})\|\big)\nonumber\\
&+\rho^{-1}\|\nabla F(x_t)^\top\nabla F(x_t)-\nabla F^\top(x_{t+1})\nabla F(x_{t+1})\|\nonumber\\
\overset{(ii)}{\leq}&\lim_{Q\rightarrow\infty}2\sqrt{\frac{4}{\rho\beta Q}}+\rho^{-1}\|\nabla F(x_t)^\top\nabla F(x_t)-\nabla F^\top(x_{t+1})\nabla F(x_{t+1})\|,
\end{align}
where $(i)$ follows from \cref{eq:oneoverrho} and $(ii)$ follows from the convergence of PGD (Theorem 1.1, \citep{beck2009gradient}) on $\rho$-strongly convex objectives that
\begin{align*}
\|w_\rho^*(x_t)-w_{Q,\rho}(x_t)\|^2\leq\frac{2}{\rho}\big(J(w_\rho^*(x_t))-J(w_{Q,\rho}(x_t))\big)\leq\frac{2}{\rho}\frac{\|w_{0,\rho}(x_t)-w_\rho^*(x_t)\|^2}{2\beta Q}\leq\frac{4}{\rho\beta Q}.
\end{align*}
Then \cref{eq:stargap} can be bounded by 
\begin{align*}
\|w_\rho^*(x_t)-w_\rho^*(x_{t+1})\|\leq&\rho^{-1}\|\nabla F(x_t)^\top\nabla F(x_t)-\nabla F^\top(x_{t+1})\nabla F(x_{t+1})\|\nonumber\\
\leq&\rho^{-1}\|\nabla F(x_t)+\nabla F(x_{t+1})\|\|\nabla F(x_t)-\nabla F(x_{t+1})\|\nonumber\\
\leq&2\rho^{-1}KM\ell(M+1)\|x_t-x_{t+1}\|,
\end{align*}
where the last inequality follows from $\|\nabla f_i(x_t)\|\leq M$ and $f_i(x)$ is $\Big(\frac{1}{\ell(\|\nabla f_i(x)\|+1)},\ell(\|\nabla f_i(x)\|+1)\Big)$-smooth by setting $a=1$. The proof is complete.
\end{proof}
\end{lemma}
\begin{lemma}\label{lemma:sqrtrho}
Given $w_t^*=\arg\min_{w\in\mathcal{W}}\frac{1}{2}\|\nabla F(x_t)w\|^2$ and $w_{t,\rho}^*=\arg\min_{w\in\mathcal{W}}\frac{1}{2}\|\nabla F(x_t)w\|^2+\frac{\rho}{2}\|w\|^2$, we have
\begin{align*}
\|\nabla F(x_t)w_t^*-\nabla F(x_t)w_{t,\rho}^*\|\leq\sqrt{\rho}.
\end{align*}
\end{lemma}
\begin{proof}
Recall that $w_{t,\rho}^*=\arg\min_{w\in\mathcal{W}}\frac{1}{2}\|\nabla F(x_t)w\|^2+\frac{\rho}{2}\|w\|^2$, then we have
\begin{align*}
\frac{1}{2}\|\nabla F(x_t)w_t^*\|^2+\frac{\rho}{2}\|w_t^*\|^2-\frac{1}{2}\|\nabla F(x_t)w_{t,\rho}^*\|^2-\frac{\rho}{2}\|w_{t,\rho}^*\|^2\geq 0.
\end{align*}
By rearranging the above inequality, we have
\begin{align*}
\|\nabla F(x_t)w_{t,\rho}^*\|^2-\|\nabla F(x_t)w_t^*\|^2\leq\rho(\|w_{t,\rho}^*\|^2-\|w_t^*\|^2)\leq\rho.
\end{align*}
Then recall that $w_t^*=\arg\min_{w\in\mathcal{W}}\frac{1}{2}\|\nabla F(x_t)w\|^2$, we have
\begin{align*}
\|\nabla F(x_t)w_t^*-\nabla F(x_t)w_{t,\rho}^*\|^2=&\|\nabla F(x_t)w_t^*\|^2+\|\nabla F(x_t)w_{t,\rho}^*\|^2-2\langle\nabla F(x_t)w_t^*,\nabla F(x_t)w_{t,\rho}^*\rangle\nonumber\\
\leq&\|\nabla F(x_t)w_{t,\rho}^*\|^2-\|\nabla F(x_t)w_{t}^*\|^2\nonumber\\
\leq&\rho,
\end{align*}
where the first inequlity follows from the optimality that $$-2\langle w_{t,\rho}^*,\nabla F(x_t)^\top\nabla F(x_t)w_t^*\rangle\leq-2\|\nabla F(x_t)w_{t}^*\|^2.$$ The proof is complete.
\end{proof}
\begin{lemma}\label{lemma:remainder}
Suppose Assumptions \ref{ass:diffandlowerbound} and \ref{ass:lsmooth} are satisfied.  If for any $i\in[K], \|\nabla f_i(x_t)\|\le M$ and $\|x_t-x_{t+1}\|\le \frac{1}{\ell(M+1)}$, we have
\begin{align*}
\|R(x_t)\|\leq\frac{\alpha^2K\ell(M+1)M^2}{2}.
    \end{align*}
\begin{proof}
According to the Talyor Theorem, we have the following result for any objective function $f_i(x_t), i\in[K]$.
\begin{align*}
f_i(x_{t+1})=f_i(x_t)+\nabla f_i^\top(x_t)(x_{t+1}-x_{t})+R_i(x_t),\nonumber\\
\end{align*}
where $R_i(x_t)$ is the remainder term. Then according to the descent lemma of each objective function $f_i(x)$, we have
\begin{align*}
f_i(x_{t+1})\leq&f_i(x_{t})+\nabla f_i^\top(x_t)(x_{t+1}-x_t)+\alpha^2\frac{\ell(\|\nabla f_i(x_t)\|+1)}{2}\|\nabla F(x_t)w_t\|^2\nonumber\\
\leq&f_i(x_{t})+\nabla f_i^\top(x_t)(x_{t+1}-x_t)+\alpha^2\frac{\ell(M+1)}{2}\|\nabla F(x_t)w_t\|^2.
\end{align*}
Then we can obtain
\begin{align*}
    R_i(x_t)\leq\alpha^2\frac{\ell(M+1)}{2}\|\nabla F(x_t)w_t\|^2.
\end{align*}
Thus, according to the Cauchy-Schwartz inequality, we have
\begin{align*}
\|R(x_t)\|\leq \alpha^2\frac{\ell(M+1)}{2}\|\nabla F(x_t)w_t\|^2\leq\frac{\alpha^2K\ell(M+1)M^2}{2}. 
\end{align*}
The proof is complete.
\end{proof}
\end{lemma}

% \begin{lemma}
% Suppose \xiao{?????}
% \begin{align}
% \|w_N-w_0^*\|^2=\mathcal{O}((1-2\beta\rho)^N+\frac{\alpha}{\rho}
% +\frac{\beta}{\rho}).
% \end{align}
% \begin{proof}
% By the non-expansiveness of projection, we have
% \begin{align}
% \|w_{n+1}-w_0^*\|^2\leq&\|w_n-\beta\Big[\frac{F(x_0)-F(x_{1})}{\alpha}+\rho w_n\Big]-w_0^*\|^2\nonumber\\
% =&\|w_n-\beta\Big[\nabla F^\top(x_0)\nabla
% F(x_0)w_n+\rho w_n+\frac{R(x_0)}{\alpha}\Big]-w_0^*\|^2\nonumber\\
% =&\|w_n-w_0^*\|^2-2\beta\langle w_n-w_0^*,\nabla F^\top(x_0)\nabla F(x_0)w_n+\rho w_n\rangle-2\beta\langle w_n-w_0^*,\frac{R(x_0)}{\alpha}\rangle\nonumber\\
% &+\beta^2\Big\|\nabla F^\top(x_0)\nabla
% F(x_0)w_n+\rho w_n+\frac{R(x_0)}{\alpha}\Big\|^2\nonumber\\
% \overset{(i)}{\leq}&(1-2\beta\rho)\|w_n-w_0^*\|^2+\frac{2\beta}{\alpha}\|R(x_t)\|+\beta^2(\sqrt{K}M^2+\rho+\frac{\|R(x_t)\|}{\alpha})^2\nonumber\\
% \overset{(ii)}{\leq}&(1-2\beta\rho)\|w_n-w_0^*\|^2+\frac{\alpha\beta K\ell(M+1)M}{2}+\beta^2(\sqrt{K}M^2+\rho+\frac{\alpha K\ell(M+1)M}{2})^2,
% \end{align}
% where $(i)$ follows from the strong convexity and Cauchy-Schwartz inequality, and $(ii)$ follows from \Cref{lemma:remainder}. Then telescope the above inequality over $n=0,1,...,N-1$,
% \begin{align}
% \|w_N-w_0^*\|^2\leq&(1-2\beta\rho)^N\|w_0-w_0^*\|^2+\frac{\alpha K\ell(M+1)M}{2\rho}+\frac{\beta}{\rho}(\sqrt{K}M^2+\rho+\frac{\alpha K\ell(M+1)M}{2})^2\nonumber\\
% =&\mathcal{O}((1-2\beta\rho)^N+\frac{\alpha}{\rho}
% +\frac{\beta}{\rho}).
% \end{align}
% The proof is complete.
% \end{proof}
% \end{lemma}
\subsection{Proof of \Cref{thm:consistsmallcadistance}}\label{proof:gddeterministic}
\begin{theorem}
Suppose Assumptions \ref{ass:diffandlowerbound} and \ref{ass:lsmooth} are satisfied.  We choose $\beta^\prime\leq\frac{1}{M^2}, N=\mathcal{O}(\epsilon^{-2}), {C_1\ge\sqrt{K}M^2+\rho}, \beta\leq\min\left(\frac{\epsilon^2\rho}{C_1^2},\epsilon^2,\frac{1}{4KM^2}\right), \alpha\le \min\left(\beta,\frac{1}{\ell(M+1)},\frac{\beta\rho\epsilon}{2L_w\sqrt{M}},\frac{\rho\epsilon^2}{2L_wMC_1}\right), T\ge\max\left(\frac{10\Delta}{\alpha \epsilon^2},\frac{10}{\epsilon^2\beta}\right)= \Theta(\epsilon^{-11})$, and $\rho\le\min \left(\frac{\epsilon^2}{20}, \frac{1}{2T\alpha},\sqrt{\frac{\epsilon^2}{10\beta}},\sqrt{\frac{1}{T\alpha\beta}}\right)= \mathcal O(\epsilon^2)$. The CA distance in every iteration takes the order of $\mathcal{O}(\epsilon)$.
\end{theorem}
\begin{proof}
Since our parameters satisfy all requirements in \Cref{thm:gddeterministic}, we have that $\|\nabla f_i(x_t)\|\leq M$. According to the definition of CA distance, we have
\begin{align}\label{eq:ca}
\|\nabla F&(x_t)w_t-\nabla F(x_t)w_t^*\|\nonumber\\
=&\|\nabla F(x_t)w_t-\nabla F(x_t)w_{t,\rho}^*+\nabla F(x_t)w_{t,\rho}^*-\nabla F(x_t)w_t^*\|\nonumber\\
\overset{(i)}{\leq}&\|\nabla F(x_t)w_t-\nabla F(x_t)w_{t,\rho}^*\|+\|\nabla F(x_t)w_{t,\rho}^*-\nabla F(x_t)w_t^*\|\nonumber\\
\overset{(ii)}{\leq}&\sqrt{K}M\|w_t-w_{t,\rho}^*\|+\sqrt{\rho},
\end{align}
where $(i)$ follows from Cauchy-Schwartz inequality and $(ii)$ follows from $\|\nabla f_i(x_t)\|\leq M$ for any $i$ and \Cref{lemma:sqrtrho}. Then for the first term in the above inequality on the right-hand side (RHS), we have
\begin{align}\label{eq:t+1term}
\|w_{t+1}-w_{t+1,\rho}^*\|^2=\|w_{t+1}-w_{t,\rho}^*\|^2+\|w_{t+1,\rho}^*-w_{t,\rho}^*\|^2-2\langle w_{t+1}-w_{t,\rho}^*,w_{t+1,\rho}^*-w_{t,\rho}^*\rangle.
\end{align}
For the first term on the RHS in the above inequality, we have
\begin{align}\label{eq:t+1term-1}
\|w_{t+1}-w_{t,\rho}^*\|^2\overset{(i)}{\leq}&\|w_t-\beta[\nabla F(x_t)^\top\nabla F(x_t)w_t+\rho w_t]-w_{t,\rho}^*\|^2\nonumber\\
=&\|w_t-w_{t,\rho}^*\|^2-2\beta\langle\nabla F(x_t)^\top\nabla F(x_t)w_t+\rho w_t,w_t-w_{t,\rho}^*\rangle\nonumber\\
&+\beta^2\|\nabla F(x_t)^\top\nabla F(x_t)w_t+\rho w_t\|^2\nonumber\\
\overset{(ii)}{\leq}&(1-2\beta\rho)\|w_t-w_{t,\rho}^*\|^2+\beta^2(\rho+\sqrt{K}M^2)^2,
\end{align}
where $(i)$ follows from the non-expansiveness of projection and $(ii)$ follows from properties of strong convexity and Cauchy-Schwartz inequality. Then for the second term on the RHS in \cref{eq:t+1term}, we have
\begin{align}\label{eq:t+1term-2}
\|w_{t+1,\rho}^*-w_{t,\rho}^*\|^2\leq L_w^2\|x_t-x_{t+1}\|^2=L_w^2\alpha^2\|\nabla F
(x_t)w_t\|^2\leq \alpha^2L_w^2M^2.
\end{align}
Then for the last term on the RHS in \cref{eq:t+1term}, we have
\begin{align}\label{eq:t+1term-3}
-2&\langle w_{t+1}-w_{t,\rho}^*,w_{t+1,\rho}^*-w_{t,\rho}^*\rangle\nonumber\\
\leq&2\|w_{t+1}-w_{t,\rho}^*\|\|w_{t+1,\rho}^*-w_{t,\rho}^*\|\nonumber\\
\leq&2(\|w_{t+1}-w_t\|+\|w_t-w_{t,\rho}^*\|)\|w_{t+1,\rho}^*-w_{t,\rho}^*\|\nonumber\\
\overset{(i)}{\leq}&2\alpha\beta L_w\|\nabla F(x
_t)^\top\nabla F(x_t)w_t+\rho w_t\|\|\nabla F(x_t)w_t\|+\beta\rho\|w_t-w_{t,\rho}^*\|^2+\frac{4}{\beta\rho}\|w_{t+1,\rho}^*-w_{t,\rho}^*\|^2\nonumber\\
\leq&2\alpha\beta L_wM(\sqrt{K}M^2+\rho)+\beta\rho\|w_t-w_{t,\rho}^*\|^2+\frac{4\alpha^2L_w^2M^2}{\beta\rho},
\end{align}
where $(i)$ follows from the update rule in \Cref{alg:1}, \Cref{lemma:rho}, and Young's inequality. Then substituting \cref{eq:t+1term-1}, \cref{eq:t+1term-2} and \cref{eq:t+1term-3} into \cref{eq:t+1term}, we have
\begin{align*}
\|w_{t+1}-w_{t+1,\rho}^*\|^2\leq&(1-\beta\rho)\|w_t-w_{t,\rho}^*\|^2+\beta^2(\rho+\sqrt{K}M^2)^2+\alpha^2L_w^2M^2\nonumber\\
&+2\alpha\beta L_wM(\sqrt{K}M^2+\rho)+\frac{4\alpha^2L_w^2M^2}{\beta\rho}\nonumber\\
\leq&(1-\beta\rho)\|w_t-w_{t,\rho}^*\|^2+\beta^2C_1^2+\alpha^2L_w^2M^2+2\alpha\beta L_wMC_1+\frac{4\alpha^2L_w^2M^2}{\beta\rho},
\end{align*}
where the last inequality follows from \Cref{lemma:remainder} and $C_1\ge\sqrt{K}M^2+\rho$. Then we do telescoping over $t=0,1,..., T-1$
\begin{align*}
\|w_T-w_{T,\rho}^*\|^2\leq&(1-\beta\rho)^T\|w_0-w_{0,\rho}^*\|^2+\frac{\beta}{\rho}C_1^2+\frac{\alpha^2}{\beta\rho}L_w^2M^2+\frac{2\alpha L_wM}{\rho}C_1+\frac{4\alpha^2L_w^2M^2}{\beta^2\rho^2}.
\end{align*}
Then recalling that $L_w=\mathcal{O}(\frac{1}{\rho})$ and substituting the above inequality into \cref{eq:ca}, we have
\begin{align*}
\|\nabla F(x_t)w_t&-\nabla F(x_t)w_t^*\|\nonumber\\
\leq&\sqrt{K}M\Big[(1-\beta\rho)^t\|w_0-w_{0,\rho}^*\|^2+\frac{\beta}{\rho}C_1^2+\frac{\alpha^2}{\beta\rho}L_w^2M+\nonumber\\
&\frac{2\alpha L_wM}{\rho}C_1+\frac{4\alpha^2L_w^2M^2}{\beta^2\rho^2}\Big]^{\frac{1}{2}}+\sqrt{\rho}\nonumber\\
=&\mathcal{O}\Big((1-\beta\rho)^{\frac{t}{2}}\|w_0-w_{0,\rho}^*\|+\sqrt{\frac{\beta}{\rho}}+\frac{\alpha}{\beta\rho^2}+\sqrt{\rho}\Big).
\end{align*}
Since we run projected gradient descent for the strongly convex function $J(w_n)=\frac{1}{2}\|\nabla F(x_0)w_n\|^2+\frac{\rho}{2}\|w_n\|^2$ in the N-loop in \Cref{alg:1}, according to Theorem 10.5 \citep{garrigos2023handbook}, we have by choosing $\beta'\in(0,\frac{1}{M^2}]$ $$\|w_0-w_{0,\rho}^*\|^2=\|w_N-w_{0,\rho}^*\|^2\leq2\Big(1-\frac{\rho}{M^2}\Big)^N.$$
Thus, $\|w_0-w_{0,\rho}^*\|=\mathcal{O}(\epsilon)$ as $N= \mathcal O(\rho^{-1})$.
CA distance takes the order of $\epsilon$ in every iteration by choosing $\rho= \mathcal O(\epsilon^2), \beta= \mathcal O(\epsilon^4)$, $\alpha= \mathcal O(\epsilon^9)$, and $N= \mathcal O(\epsilon^{-2})$. The proof is complete.
\end{proof}

\subsection{Formal Version and Its Proof of \Cref{thm:gsmgrad-famoinformal}}\label{proof:gsmgrad-famoformal}
Let $c_1>0, c_2^\prime>0$, $c_3^\prime, c_4^\prime\ge 0$, and $F>0$ be some constants such that 
\begin{flalign*}
    \Delta+c_1+{ c_2^\prime}+ c_3^\prime+c_4^\prime\le F
\end{flalign*}
{and $C_1'\ge\sqrt{K}M^2+\rho+\frac{\alpha K\ell(M+1)M^2}{2}$.}
We then have the following convergence rate for  Algorithm \ref{alg:gs-famo}.
\begin{theorem}
Suppose Assumptions \ref{ass:diffandlowerbound} and \ref{ass:lsmooth} are satisfied, and we choose constant step sizes that $\beta\leq\frac{\epsilon^2}{C_1'^2}, \alpha\leq\min\left(c_1\beta,\sqrt{\frac{2c_3^\prime}{K\ell(M+1)M^2T}}, \frac{2c_4^\prime}{\beta C_1'^2T}, \frac{\epsilon^2}{K\ell(M+1)M^2}\right), \rho\leq\min\left(\frac{\epsilon^2}{2},\frac{c_2^\prime}{\alpha T}\right)$, and $T\ge\max\left(\frac{10\Delta}{\alpha \epsilon^2},\frac{10}{\epsilon^2\beta}\right)$. We have
\begin{align*}
\frac{1}{T}\sum_{t=0}^{T-1}\|\nabla F(x_t)w_t\|^2=\mathcal{O}(\epsilon^2).
\end{align*}
\end{theorem}
\begin{proof}[Proof.] Following similar steps in \Cref{proof:gddeterministic-avg}, we also prove that for any $i\in K$ and $t\leq T$, we have that $f_i(x_t)-f_i^*\leq F$ by induction.

\textbf{Base case:} since all constants $c_1, c_2^\prime, c_3^\prime, c_4^\prime$ are non-negative, we have that $f_i(x_0)-f_i^*\leq\Delta\leq F$ holds for any $i\in[K]$.

\textbf{Induction step:} assume that for any $i\in[K]$ and $t\leq k<T$, $f_i(x_t)-f_i^*\leq F$ holds. We then prove $f_i(x_{k+1})-f_i^*\leq F$ holds for any $i\in[K]$. Following similar steps in \Cref{proof:gddeterministic-avg}, we have
\begin{align}\label{eq:descentlemma-1}
F(x_{t+1})w\leq F(x_t)w-\alpha\langle \nabla F(x_t)w,\nabla F(x_t)w_t\rangle+\frac{\alpha^2\ell(M+1)}{2}\|\nabla F(x_t)w_t\|^2.
\end{align}
Based on the update rule of $w$ and non-expansiveness of projection, we have
\begin{align*}
\|w_{t+1}-w\|^2\leq&\Big\|w_t-\beta\Big(\frac{F(x_t)-F(x_{t+1})}{\alpha}+\rho w_t\Big)-w\Big\|^2\nonumber\\
=&\Big\|w_t-\beta\Big(\nabla F(x_t)^\top\nabla F(x_t)w_t+\rho w_t+\frac{R(x_t)}{\alpha}\Big)-w\Big\|^2\nonumber\\
\overset{(i)}{\leq}&\|w_t-w\|^2-2\beta\langle w_t-w,(\nabla F(x_t)^\top\nabla F(x_t)+\rho I)w_t\rangle+2\frac{\beta}{\alpha}\|R(x_t)\|+\beta^2(C_1^\prime)^2\nonumber\\
\overset{(ii)}{\leq}&\|w_t-w\|^2-2\beta\langle w_t-w,(\nabla F(x_t)^\top\nabla F(x_t)+\rho I)w_t\rangle\nonumber\\
&+\alpha\beta K\ell(M+1)M^2+\beta^2(C_1^\prime)^2\nonumber,
\end{align*}
where $(i)$ follows from Cauchy-Schwartz inequality and $C_1'\ge\sqrt{K}M^2+\rho+\frac{\alpha K\ell(M+1)M^2}{2}$, and $(ii)$ follows from \Cref{lemma:remainder}. 
Then we have
\begin{align*}
\langle w_t-w,\nabla F(x_t)^\top\nabla F(x_t)w_t\rangle\leq&\frac{1}{2\beta}(\|w_t-w\|^2-\|w_{t+1}-w\|^2)+\rho\nonumber\\
&+\frac{\alpha K\ell(M+1)M^2}{2}+\frac{\beta (C_1^\prime)^2}{2}.
\end{align*}
Then substituting the above inequality into \cref{eq:descentlemma-1}, we can obtain
\begin{align*}
F(x_{t+1})w-F(x_t)\leq& -\alpha\|\nabla F(x_t)w_t\|^2+\frac{\alpha^2\ell(M+1)}{2}\|\nabla F(x_t)w_t\|^2\nonumber\\
&+\frac{\alpha}{2\beta}(\|w_t-w\|^2-\|w_{t+1}-w\|^2)+\alpha\rho+\frac{\alpha^2K\ell(M+1)M^2}{2}+\frac{\alpha\beta (C_1^\prime)^2}{2}.
\end{align*}
Then taking sums of the above inequality from $t=0$ to $k$, for any $w\in\mathcal{W}$, we have
\begin{align}\label{eq:convergenceeq}
F(x_{k+1})w-F(x_0)w\leq&-\sum_{t=0}^k\alpha\|\nabla F(x_t)w_t\|^2+\sum_{t=0}^k\frac{\alpha^2\ell(M+1)}{2}\|\nabla F(x_t)w_t\|^2\nonumber\\
&+\frac{\alpha}{2\beta}\|w_0-w\|^2+\alpha\rho T+\frac{\alpha^2 K\ell(M+1)M^2T}{2}+\frac{\alpha\beta (C_1^\prime)^2T}{2}\nonumber\\
\leq&\frac{\alpha}{\beta}+\alpha\rho T+\frac{\alpha^2 K\ell(M+1)M^2T}{2}+\frac{\alpha\beta (C_1^\prime)^2T}{2},
\end{align}
where the last inequality follows from $\alpha\leq\frac{1}{\ell(M+1)}$. Thus, for any $i\in[K]$, it can be shown that
\begin{align*}
f_i(x_{k+1})-f_i^*\leq f_i(x_0)-f_i^*+\frac{\alpha}{\beta}+\alpha\rho T+\frac{\alpha^2 K\ell(M+1)M^2T}{2}+\frac{\alpha\beta (C_1^\prime)^2T}{2}\leq F,
\end{align*}
since we have that $\frac{\alpha}{\beta}\leq c_1$, $\alpha\rho T\leq c_2^\prime$, $\frac{\alpha^2K\ell(M+1)M^2T}{2}\leq c_3\prime$, $\frac{\alpha\beta (C_1^\prime)^2 T}{2}\leq c_4^\prime$. Now we finish the induction step and can show that $f_i(x_k)-f_i^*\leq F$ and \cref{eq:convergenceeq} hold for all $k<T$ and $i\in[K]$. Specifically, for $\alpha\leq\frac{1}{\ell(M+1)}$, we have
\begin{align*}
\frac{1}{T}\sum_{t=0}^{T-1}\|\nabla F(x_t)w_t\|^2\leq\frac{2F(x_0)w-2F^*w}{\alpha T}+\frac{2}{\beta T}+2\rho+\alpha K\ell(M+1)M^2+\beta (C_1^\prime)^2.
\end{align*}
Then following the choice of step sizes, we can obtain
\begin{align*}
\frac{1}{T}\sum_{t=0}^{T-1}\|\nabla F(x_t)w_t\|^2=\mathcal{O}(\epsilon^2).
\end{align*}
The proof is complete.
\end{proof}
\subsection{Formal Version and Its Proof of \Cref{thm:consistsmallcadistance-famo}}\label{proof:gddeterministic-famo}
\begin{theorem}
Suppose Assumptions \ref{ass:diffandlowerbound} and \ref{ass:lsmooth} are satisfied.  We choose $\beta^\prime\leq\frac{1}{M^2}, N=\Omega(\epsilon^{-2}), {C_1^\prime\ge\sqrt{K}M^2+\rho+\frac{\alpha K\ell(M+1)M^2}{2}}, \beta\leq\min\left(\frac{\epsilon^2\rho}{(C_1^\prime)^2},\epsilon^2\right), \alpha\le \min\left(c_1\beta,\frac{2c_3^\prime}{\beta c_1^\prime T},\frac{1}{\ell(M+1)},\frac{\beta\rho\epsilon}{2L_w M},\frac{\rho\epsilon^2}{2L_wMC_1^\prime}\right), T\ge\max\left(\frac{10\Delta}{\alpha \epsilon^2},\frac{10}{\epsilon^2\beta}\right)= \Theta(\epsilon^{-11})$, and $\rho\le\min \left(\frac{\epsilon^2}{20}, \frac{c_2^\prime}{2T\alpha}\right)= \mathcal O(\epsilon^2)$. The CA distance in every iteration takes the order of $\mathcal{O}(\epsilon)$.
\end{theorem}
\begin{proof}
According to the definition of CA distance, we have
\begin{align}\label{eq:ca-famo}
\|\nabla F&(x_t)w_t-\nabla F(x_t)w_t^*\|\nonumber\\
=&\|\nabla F(x_t)w_t-\nabla F(x_t)w_{t,\rho}^*+\nabla F(x_t)w_{t,\rho}^*-\nabla F(x_t)w_t^*\|\nonumber\\
\overset{(i)}{\leq}&\|\nabla F(x_t)w_t-\nabla F(x_t)w_{t,\rho}^*\|+\|\nabla F(x_t)w_{t,\rho}^*-\nabla F(x_t)w_t^*\|\nonumber\\
\overset{(ii)}{\leq}&\sqrt{K}M\|w_t-w_{t,\rho}^*\|+\sqrt{\rho},
\end{align}
where $(i)$ follows from Cauchy-Schwartz inequality and $(ii)$ follows from $\|\nabla f_i(x_t)\|\leq M$ for any $i$ and \Cref{lemma:rho}. Then for the first term in the above inequality on the right-hand side (RHS), we have
\begin{align}\label{eq:t+1term-famo}
\|w_{t+1}-w_{t+1,\rho}^*\|^2=\|w_{t+1}-w_{t,\rho}^*\|^2+\|w_{t+1,\rho}^*-w_{t,\rho}^*\|^2-2\langle w_{t+1}-w_{t,\rho}^*,w_{t+1,\rho}^*-w_{t,\rho}^*\rangle.
\end{align}
For the first term on the RHS in the above inequality, we have
\begin{align}\label{eq:t+1term-1-famo}
\|w_{t+1}-w_{t,\rho}^*\|^2\overset{(i)}{\leq}&\Big\|w_t-\beta\Big(\frac{F(x_t)-F(x_{t+1})}{\alpha}+\rho w_t\Big)-w_{t,\rho}^*\Big\|^2\nonumber\\
=&\Big\|w_t-\beta\Big(\nabla F(x_t)^\top\nabla F(x_t)w_t+\rho w_t+\frac{R(x_t)}{\alpha}\Big)-w_{t,\rho}^*\Big\|^2\nonumber\\
=&\|w_t-w_{t,\rho}^*\|^2-2\beta\langle\nabla F(x_t)^\top\nabla F(x_t)w_t+\rho w_t,w_t-w_{t,\rho}^*\rangle\nonumber\\
&-2\frac{\beta}{\alpha}\langle R(x_t),w_t-w_{t,\rho}^*\rangle+\beta^2\left\|\nabla F(x_t)^\top\nabla F(x_t)w_t+\frac{R(x_t)}{\alpha}+\rho w_t\right\|^2\nonumber\\
\overset{(ii)}{\leq}&(1-2\beta\rho)\|w_t-w_{t,\rho}^*\|^2+2\frac{\beta}{\alpha}\|R(x_t)\|+\beta^2\left(\rho+\sqrt{K}M^2+\frac{\|R(x_t)\|}{\alpha}\right)^2,
\end{align}
where $(i)$ follows from the non-expansiveness of projection and $(ii)$ follows from properties of strong convexity and Cauchy-Schwartz inequality. Then for the second term on the RHS in \cref{eq:t+1term-famo}, we have
\begin{align}\label{eq:t+1term-2-famo}
\|w_{t+1,\rho}^*-w_{t,\rho}^*\|^2\leq L_w^2\|x_t-x_{t+1}\|^2=L_w^2\alpha^2\|\nabla F
(x_t)w_t\|^2\leq \alpha^2L_w^2M^2.
\end{align}
Then for the last term on the RHS in \cref{eq:t+1term-famo}, we have
\begin{align}\label{eq:t+1term-3-famo}
-2\langle w_{t+1}&-w_{t,\rho}^*,w_{t+1,\rho}^*-w_{t,\rho}^*\rangle\nonumber\\
\leq&2\|w_{t+1}-w_{t,\rho}^*\|\|w_{t+1,\rho}^*-w_{t,\rho}^*\|\nonumber\\
\leq&2(\|w_{t+1}-w_t\|+\|w_t-w_{t,\rho}^*\|)\|w_{t+1,\rho}^*-w_{t,\rho}^*\|\nonumber\\
\overset{(i)}{\leq}&2\alpha\beta L_w\Big\|\frac{F(x_t)-F(x_{t+1})}{\alpha}+\rho w_t\Big\|\|\nabla F(x_t)w_t\|+\beta\rho\|w_t-w_{t,\rho}^*\|^2+\frac{4}{\beta\rho}\|w_{t+1,\rho}^*-w_{t,\rho}^*\|^2\nonumber\\
\leq&2\alpha\beta L_wM\left(\sqrt{K}M^2+\frac{\|R(x_t)\|}{\alpha}+\rho\right)+\beta\rho\|w_t-w_{t,\rho}^*\|^2+\frac{4\alpha^2L_w^2M^2}{\beta\rho}.
\end{align}
Then substituting \cref{eq:t+1term-1-famo}, \cref{eq:t+1term-2-famo} and \cref{eq:t+1term-3-famo} into \cref{eq:t+1term-famo}, we have
\begin{align*}
\|w_{t+1}-w_{t+1,\rho}^*\|^2\leq&(1-\beta\rho)\|w_t-w_{t,\rho}^*\|^2+2\frac{\beta}{\alpha}\|R(x_t)\|+\beta^2(\rho+\sqrt{K}M^2+\|R(x_t)\|)^2\nonumber\\
&+\alpha^2L_w^2M+2\alpha\beta L_wM\left(\sqrt{K}M^2+\frac{\|R(x_t)\|}{\alpha}+\rho\right)+\frac{4\alpha^2L_w^2M^2}{\beta\rho}\nonumber\\
\leq&(1-\beta\rho)\|w_t-w_{t,\rho}^*\|^2+\alpha\beta K\ell(M+1)M^2+\beta^2(C_1^\prime)^2\nonumber\\
&+\alpha^2L_w^2M+2\alpha\beta L_wMC_1^\prime+\frac{4\alpha^2L_w^2M^2}{\beta\rho},
\end{align*}
where the last inequality follows from \Cref{lemma:remainder} and $C_1^\prime\ge\sqrt{K}M^2+\rho+\frac{\alpha K\ell(M+1)M^2}{2}$. Then we do telescoping over $t=0,1,..., T-1$
\begin{align*}
\|w_T-w_{T,\rho}^*\|^2\leq&(1-\beta\rho)^T\|w_0-w_{0,\rho}^*\|^2+\frac{\alpha}{\rho} K\ell(M+1)M^2+\frac{\beta}{\rho}(C_1^\prime)^2\nonumber\\
&+\frac{\alpha^2}{\beta\rho}L_w^2M+\frac{2\alpha L_wM}{\rho}C_1^\prime+\frac{4\alpha^2L_w^2M^2}{\beta^2\rho^2}.
\end{align*}
Then substituting the above inequality into \cref{eq:ca-famo}, we have
\begin{align*}
\|\nabla F(x_t)w_t&-\nabla F(x_t)w_t^*\|\nonumber\\
\leq&\sqrt{K}M\Big[(1-\beta\rho)^t\|w_0-w_{0,\rho}^*\|^2+\frac{\alpha}{\rho} K\ell(M+1)M^2+\frac{\beta}{\rho}(C_1^\prime)^2\nonumber\\
&+\frac{\alpha^2}{\beta\rho}L_w^2M+\frac{2\alpha L_wM}{\rho}C_1^\prime+\frac{4\alpha^2L_w^2M}{\beta^2\rho^2}\Big]^{\frac{1}{2}}+\sqrt{\rho}\nonumber\\
=&\mathcal{O}\Big((1-\beta\rho)^{\frac{t}{2}}\|w_0-w_{0,\rho}^*\|+\sqrt{\frac{\alpha}{\rho^2}}+\sqrt{\frac{\beta}{\rho}}+\frac{\alpha}{\beta\rho^2}+\sqrt{\rho}\Big).
\end{align*}
Since we run projected gradient descent for the strongly convex function $J(w_n)=\frac{1}{2}\|\nabla F(x_0)w_n\|^2+\frac{\rho}{2}\|w_n\|^2$ in the N-loop in \Cref{alg:gs-famo}, according to Theorem 10.5 \citep{garrigos2023handbook}, we have $$\|w_0-w_{0,\rho}^*\|^2=\|w_N-w_{0,\rho}^*\|^2\leq2\Big(1-\frac{\rho}{M^2+\rho}\Big)^N.$$
Thus, $\|w_0-w_{0,\rho}^*\|=\mathcal{O}(\epsilon)$ as $N= \Omega(\rho^{-1})$.
CA distance takes the order of $\epsilon$ in every iteration by choosing $\rho= \mathcal O(\epsilon^2), \beta= \mathcal O(\epsilon^4)$, $\alpha= \mathcal O(\epsilon^9)$, and $N=\Omega(\epsilon^{-2})$. The proof is complete.
\end{proof}

\subsection{Formal Version of Its Proof of \Cref{thm:sgdeachca}}\label{proof:sgdeachca}
Let $\alpha, \beta, \rho, T$ satisfy all requirements for \Cref{thm:gdstochastic} with $\delta<\frac{1}{2}$. Moreover, for $\rho = \mathcal O(\epsilon^2), N=\Omega(\epsilon^{-2}), \beta\le \frac{\delta\rho\epsilon^2}{60(1+KM^4)}= \mathcal O(\epsilon^4), n_s\ge \max\{{K\sigma^2},\frac{36K\sigma^2M^2(6+20\beta\rho)}{\delta\rho^2\epsilon^2}\}= \Omega(\epsilon^{-6})$ and $\alpha\le \sqrt{\frac{\delta\beta^2\rho^2\epsilon^2}{12L_w^2(2M^2+4K\sigma^2)(\beta\rho+1)}}= \mathcal O(\epsilon^9)$ and $T= \Theta(\epsilon^{-11})$, we have the following theorem:

\begin{theorem}
If Assumptions \ref{ass:diffandlowerbound}, \ref{ass:lsmooth} and \ref{ass:boundedvariance} hold, with the values of the parameters mentioned  above,  we have that for each $t\le T$,
\begin{flalign*}
    \|\nabla F(x_t)w_t-\nabla F(x_t)w_t^*\|= \mathcal O(\epsilon),
\end{flalign*}
with the probability at least $1-\delta$.
\end{theorem}

\begin{proof}
When $\tau=T$ and $t<\tau $, according to the definition of CA distance, we have
\begin{align}\label{eq:ca2}
\|\nabla F&(x_t)w_t-\nabla F(x_t)w_t^*\|\nonumber\\
=&\|\nabla F(x_t)w_t-\nabla F(x_t)w_{t,\rho}^*+\nabla F(x_t)w_{t,\rho}^*-\nabla F(x_t)w_t^*\|\nonumber\\
\overset{(i)}{\leq}&\|\nabla F(x_t)w_t-\nabla F(x_t)w_{t,\rho}^*\|+\|\nabla F(x_t)w_{t,\rho}^*-\nabla F(x_t)w_t^*\|\nonumber\\
\overset{(ii)}{\leq}&\sqrt{K}M\|w_t-w_{t,\rho}^*\|+\sqrt{\rho},
\end{align}
where $(i)$ follows from Cauchy-Schwartz inequality and $(ii)$ follows from $\|\nabla f_i(x_t)\|\leq M$ for any $i\in[K]$ and \Cref{lemma:sqrtrho}. 
We then show that for any $t\le \tau$, we have that $\mathbb E[\|w_t-w_{t,\rho}^*\|^2|\tau=T]\le \frac{\delta}{2}\epsilon^2$ by induction. 

\textbf{Base case:} Since we run projected gradient descent for the strongly convex function $J(w_n)=\frac{1}{2}\|\nabla F(x_0)w_n\|^2+\frac{\rho}{2}\|w_n\|^2$ in the N-loop in \Cref{alg:1},  according to Theorem 10.5 \citep{garrigos2023handbook}, we have by choosing $\beta'\in(0,\frac{1}{M^2}]$ $$\|w_0-w_{0,\rho}^*\|^2=\|w_N-w_{0,\rho}^*\|^2\leq2\Big(1-\frac{\rho}{M^2}\Big)^N.$$
Thus, $\|w_0-w_{0,\rho}^*\|^2=\mathcal{O}(\frac{\delta}{2}\epsilon^2)$ as $N= \Omega(\rho^{-1})$.

\textbf{Induction:} Assume we have that $ \mathbb E[\|w_t-w_{t,\rho}^*\|^2|\tau=T]\le \frac{\delta}{2}\epsilon^2$, we will show that  $ \mathbb E[\|w_{t+1}-w_{t+1,\rho}^*\|^2|\tau=T]\le \frac{\delta}{2}\epsilon^2$ holds for any $t<\tau $ in the following proof.
We first divide $\|w_{t+1}-w_{t+1,\rho}^*\|^2$ into three parts:
\begin{align}\label{eq:t+1term2}
\|w_{t+1}-w_{t+1,\rho}^*\|^2=\|w_{t+1}-w_{t,\rho}^*\|^2+\|w_{t+1,\rho}^*-w_{t,\rho}^*\|^2-2\langle w_{t+1}-w_{t,\rho}^*,w_{t+1,\rho}^*-w_{t,\rho}^*\rangle.
\end{align}
For the first term on the RHS in the above inequality, we have that
\begin{align}\label{eq:t+1term-12}
&\|w_{t+1}-w_{t,\rho}^*\|^2\nonumber\\
\overset{(i)}{\leq}&\Big\|w_t-\beta\Big(\nabla G_2(x_t)^\top\nabla G_3(x_t)w_t+\rho w_t\Big)-w_{t,\rho}^*\Big\|^2\nonumber\\
=&\|w_t-w_{t,\rho}^*\|^2-2\beta\langle\nabla G_2(x_t)^\top\nabla G_2(x_t)w_t+\rho w_t,w_t-w_{t,\rho}^*\rangle\nonumber\\
&+\beta^2\|\nabla G_2(x_t)^\top\nabla G_3(x_t)w_t+\rho w_t\|^2\nonumber\\
\overset{(ii)}{\leq}&(1-2\beta\rho)\|w_t-w_{t,\rho}^*\|^2\nonumber\\
&+2\beta \left\langle w_t-w_{t,\rho}^*, \varepsilon_{t,2}^\top \nabla F(x_t)w_t
+  \nabla F(x_t)^\top\varepsilon_{t,3}w_t-\varepsilon_{t,2}^\top\varepsilon_{t,3}w_t\right\rangle\nonumber\\
&+\beta^2\|\rho w_t
+\nabla F(x_t)^\top \nabla F(x_t)w_t-\varepsilon_{t,2}^\top \nabla F(x_t)w_t-  \nabla F(x_t)^\top\varepsilon_{t,3}w_t+\varepsilon_{t,2}^\top\varepsilon_{t,3}w_t\|^2,
\end{align}
where  $(i)$ follows from the non-expansiveness of projection and $(ii)$ follows from properties of strong convexity and Cauchy-Schwartz inequality. Taking the conditional expectation of \eqref{eq:t+1term-12}, we have that for any $a_1>0$,
\begin{flalign}\label{eq:t+1term87}
    &\mathbb E[\|w_{t+1}-w_{t,\rho}^*\|^2|\tau=T]\nonumber\\
    \le& \frac{\delta}{2}(1-2\beta\rho)\epsilon^2+2\beta\mathbb E[\|w_t-w_{t,\rho}^*\|\|\varepsilon_{t,2}^\top \nabla F(x_t)w_t+  \nabla F(x_t)^\top\varepsilon_{t,3}w_t-\varepsilon_{t,2}^\top\varepsilon_{t,3}w_t\||\tau=T]\nonumber\\
    &+\mathbb E[\beta^2\|\rho w_t+\nabla F(x_t)^\top \nabla F(x_t)w_t-\varepsilon_{t,2}^\top \nabla F(x_t)w_t-  \nabla F(x_t)^\top\varepsilon_{t,3}w_t+\varepsilon_{t,2}^\top\varepsilon_{t,3}w_t\|^2|\tau=T]\nonumber\\
    \le& \beta (\mathbb E[a_1\|w_t-w_{t,\rho}^*\|^2+\|\varepsilon_{t,2}^\top \nabla F(x_t)w_t+  \nabla F(x_t)^\top\varepsilon_{t,3}w_t-\varepsilon_{t,2}^\top\varepsilon_{t,3}w_t\|^2/a_1|\tau=T]])\nonumber\\
    &+ \frac{\delta}{2}(1-2\beta\rho)\epsilon^2+5\beta^2\rho^2+5\beta^2KM^4+5\beta^2\mathbb E[M^2\|\epsilon_{t,2}\|^2|\tau=T]\nonumber\\
    &+5\beta^2\mathbb E[M^2\|\epsilon_{t,3}\|^2|\tau=T]+5\beta^2\mathbb E[\|\epsilon_{t,2}\|^2\|\epsilon_{t,3}\|^2|\tau=T],
\end{flalign}
where the last inequality is due to that for $t\le \tau=T$, and for any $i\in [K]$, we have that $\|\nabla f_i(x_t)\|\le M.$
Then for the second term on the RHS in \cref{eq:t+1term2}, we have
\begin{align}\label{eq:t+1term-22}
\mathbb E[\|w_{t+1,\rho}^*-w_{t,\rho}^*\|^2|\tau=T]&\leq \mathbb E[L_w^2\|x_t-x_{t+1}\|^2|\tau=T]\nonumber\\
&=\mathbb E[L_w^2\alpha^2\|\nabla F
(x_t, s_{t,1})w_t\|^2|\tau=T]\nonumber\\
&\leq \mathbb E[\alpha^2L_w^2(M+\|\epsilon_{t,1}w_t\|)^2|\tau=T],
\end{align}
where the first inequality is due to \Cref{lemma:rho}, where $L_w=\mathcal O(\rho^{-1})$.
Then for the last term on the RHS in \cref{eq:t+1term2}, for any $a_2>0, a_3>0$, we have that
\begin{align}\label{eq:t+1term-32}
&\mathbb E[-2\langle w_{t+1}-w_{t,\rho}^*,w_{t+1,\rho}^*-w_{t,\rho}^*\rangle|\tau=T]\nonumber\\
\leq&\mathbb E[2\|w_{t+1}-w_{t,\rho}^*\|\|w_{t+1,\rho}^*-w_{t,\rho}^*\||\tau=T]\nonumber\\
\leq&\mathbb E[2(\|w_{t+1}-w_t\|+\|w_t-w_{t,\rho}^*\|)\|w_{t+1,\rho}^*-w_{t,\rho}^*\||\tau=T]\nonumber\\
\le& \mathbb E\left[a_2\|w_{t+1}-w_t\|^2+\frac{1}{a_2}\|w_{t+1,\rho}^*-w_{t,\rho}^*\|^2+a_3\|w_{t}-w_{t,\rho}^*\|^2+\frac{1}{a_3}\|w_{t+1,\rho}^*-w_{t,\rho}^*\|^2|\tau=T\right]\nonumber\\
\overset{(i)}{\leq}&\mathbb E\left[a_2\beta^2\|\nabla G_2(x_t)^\top\nabla G_3(x_t)w_t+\rho w_t\|^2+a_3\frac{\delta}{2}\epsilon^2+\left(\frac{1}{a_2}+\frac{1}{a_3}\right)\alpha^2L_w^2(M+\|\epsilon_{t,1}w_t\|)^2|\tau=T\right]\nonumber\\
\leq&a_2(5\beta^2\rho^2+5\beta^2KM^4+5\beta^2\mathbb E[M^2\|\epsilon_{t,2}\|^2|\tau=T]\nonumber\\
    &+5\beta^2\mathbb E[M^2\|\epsilon_{t,3}\|^2|\tau=T]+5\beta^2\mathbb E[\|\epsilon_{t,2}\|^2|\|\epsilon_{t,3}\|^2|\tau=T])\nonumber\\
&+\mathbb E\Big[\left(\frac{1}{a_2}+\frac{1}{a_3}\right)\alpha^2L_w^2(M+\|\epsilon_{t,1}w_t\|)^2|\tau=T\Big]+a_3\frac{\delta}{2}\epsilon^2,
\end{align}
where $(i)$ follows from  the non-expansiveness of projection and \eqref{eq:t+1term-22}, and the last inequality is from  \eqref{eq:t+1term87}.
Then substituting \cref{eq:t+1term87}, \cref{eq:t+1term-22} and \cref{eq:t+1term-32} into \cref{eq:t+1term2}, we have
\begin{align}\label{eq:t+1final}
&\mathbb E[\|w_{t+1}-w_{t+1,\rho}^*\|^2|\tau=T]\nonumber\\
\leq&(1-2\beta\rho+\beta a_1+a_3)\frac{\delta}{2}\epsilon^2\nonumber\\
&+\beta^2(5\rho^2+5KM^4)(1+a_2)\nonumber\\
&+M^2\left(\frac{3\beta}{a_1}+5\beta^2+5\beta^2a_2\right)\mathbb E[\|\varepsilon_{t,2}\|^2|\tau=T]\nonumber\\
&+M^2\left(\frac{3\beta}{a_1}+5\beta^2+5\beta^2a_2\right)\mathbb E[\|\varepsilon_{t,3}\|^2|\tau=T]\nonumber\\
&+\left(\frac{3\beta}{a_1}+5\beta^2+5\beta^2a_2\right)\mathbb E[\|\varepsilon_{t,2}\|^2\|\varepsilon_{t,3}\|^2|\tau=T]\nonumber\\
&+\mathbb E\Big[\left(1+\frac{1}{a_2}+\frac{1}{a_3}\right)\alpha^2L_w^2(M+\|\epsilon_{t,1}w_t\|)^2|\tau=T\Big]\nonumber\\
\leq&(1-2\beta\rho+\beta a_1+a_3)\frac{\delta}{2}\epsilon^2\nonumber\\
&+\beta^2(5\rho^2+5KM^4)(1+a_2)\nonumber\\
&+M^2\left(\frac{3\beta}{a_1}+5\beta^2+5\beta^2a_2\right)\left({\frac{4K\sigma^2}{n_s}}+\frac{{2}K^2\sigma^4}{n_s^2}\right)\nonumber\\
&+\left(1+\frac{1}{a_2}+\frac{1}{a_3}\right)\alpha^2L_w^2\left(2M^2+\frac{4K\sigma^2}{n_s}\right),
\end{align}
where the last inequality is due to that for any $i\in[3]$,
\begin{flalign*}
    \mathbb E[\|\epsilon_{t,i}\||\tau=T]\le  \sqrt{\mathbb E[\|\epsilon_{t,i}\|^2|\tau=T]}\le \sqrt{\mathbb E[\|\epsilon_{t,i}\|^2]/\mathbb P(\tau=T)}\le \sqrt{\frac{2K}{n_s}}\sigma
\end{flalign*}
and 
\begin{flalign*}
    \mathbb E[\|\epsilon_{t,2}\|\|\epsilon_{t,3}\||\tau=T]
    &\le  \sqrt{\mathbb E[\|\epsilon_{t,2}\|^2\|\epsilon_{t,3}\|^2|\tau=T]}\nonumber\\
    &\le \sqrt{\mathbb E[\|\epsilon_{t,2}\|^2\|\epsilon_{t,3}\|^2]/\mathbb P(\tau=T)}\nonumber\\
    &\le \sqrt{\mathbb E[\|\epsilon_{t,2}\|^2]\mathbb E[\|\epsilon_{t,3}\|^2]/\mathbb P(\tau=T)}\nonumber\\
    &\le \frac{\sqrt{2}K\sigma^2}{n_s}.
\end{flalign*}
According to \eqref{eq:t+1final}, with $a_1=0.5\rho, a_2=1, a_3=0.5\beta \rho, \beta\le \frac{\delta\rho\epsilon^2}{60(1+KM^4)}, n_s\ge \max\{{K\sigma^2},\frac{36K\sigma^2M^2(6+20\beta\rho)}{\delta\rho^2\epsilon^2}\}$ and $\alpha\le \sqrt{\frac{\delta\beta^2\rho^2\epsilon^2}{12L_w^2(2M^2+4K\sigma^2)(\beta\rho+1)}}$,
we have that
\begin{flalign*}
    \mathbb E[\|w_{t+1}-w_{t+1,\rho}^*\|^2|\tau=T]\le \frac{\delta}{2}\epsilon^2.
\end{flalign*}
We then complete our induction and prove that for any $t<\tau$, we have that $\mathbb E[\|w_{t+1}-w_{t+1,\rho}^*\|^2|\tau=T]\le \frac{\delta}{2}\epsilon^2$.

As a result, we have that
\begin{flalign*}
     &\mathbb P\left(\|w_{t+1}-w_{t+1,\rho}^*\|^2> \epsilon^2\Big|\tau=T\right)\le \frac{\mathbb E\left[\|w_{t+1}-w_{t+1,\rho}^*\|^2\Big|\tau=T\right]}{\epsilon^2}\le \frac{\delta}{2},
\end{flalign*}
where the first probability is due to Markov inequality.
Thus we have that
\begin{flalign}
    &\mathbb P\left(\|w_{t+1}-w_{t+1,\rho}^*\|^2\le \epsilon^2\right)\nonumber\\
    &\ge 1- \mathbb P\left(\tau<T\right)-\mathbb P\left(\|w_{t+1}-w_{t+1,\rho}^*\|^2\Big|\tau=T\right)\mathbb P\left(\tau=T\right)\nonumber\\
    &\ge 1-\delta,
\end{flalign}
where the last inequality is because our parameters satisfy all the requirements in \Cref{thm:gdstochastic}, thus $\mathbb P(\tau<T)\le \frac{\delta}{2}$.
Then based on \eqref{eq:ca2}, by setting $\rho= \mathcal O(\epsilon^2)$, we have that $\|\nabla F(x_t)w_t-\nabla F(x_t)w_t^*\|= \mathcal O(\epsilon)$ with probability at least $1-\delta$ for each iteration $t$, which completes the proof. 
\end{proof}

\end{document}